\documentclass[11pt]{article}  
\usepackage{amsmath}
\usepackage{graphicx,psfrag,epsf}
\usepackage{enumerate}
\usepackage[left=1.2in, right=1.2in, top=1in, bottom=1in]{geometry}
\usepackage[numbers, square]{natbib}
\usepackage{amsfonts}
\usepackage{amssymb}
\usepackage[usenames]{xcolor}
\usepackage{algorithm}
\usepackage{bbm}
\usepackage{subfig}
\usepackage[normalem]{ulem}
\useunder{\uline}{\ul}{}
\usepackage{amsthm}
\usepackage[CJKbookmarks=true,
            bookmarksnumbered=true,  
            bookmarksopen=true,
            colorlinks=true,
            citecolor=red, 
            linkcolor=blue,
            anchorcolor=red,
            urlcolor=blue]{hyperref}
\newcommand{\nb}[1]{\textcolor{orange}{\texttt{[#1]}}}
\newcommand{\gsc}[1]{\textcolor{blue}{\texttt{[#1]}}}

\newcommand{\rev}[1]{\textcolor{black}{#1}} 
\newcommand{\red}{\color{red}}

\usepackage{multirow} 

\usepackage[noend]{algpseudocode} 

\makeatletter
\def\algbackskip{\hskip-\ALG@thistlm}
\makeatother 
\newcommand{\tilU}{\widetilde{U}}

\newcommand{\barU}{\overline{U}}
\newcommand{\supp}{\mathrm{supp}}
\newcommand{\kl}{\mathrm{KL}}
\newcommand\numberthis{\addtocounter{equation}{1}\tag{\theequation}}

\newcommand{\dc}[1]{\{#1\}} 
\newcommand{\LambdaRest}{\overline{\Lambda}} 
\newcommand{\norminit}{\widetilde{A}_0} 
\newcommand{\eigen}{s} 

\newcommand{\ha}{\widehat{A}}
\newcommand{\hb}{\widehat{B}}
\newcommand{\bA}{\widehat{A}}
\newcommand{\init}{\widehat{A}_0}

\newcommand{\hl}{\widehat{\Lambda}_r}

\newcommand{\hs}{s'}
\newcommand{\festimate}{\widehat{A}_t}
\newcommand{\scale}{V}
\newcommand{\hU}{\widehat{V}}
\newcommand{\gd}{\overline{V}}
\newcommand{\dist}{\mathrm{dist}}
\newcommand{\has}{\widehat{\LambdaRest}}
\newcommand{\rescov}{\widehat{\Sigma}_{S_tS_t}}
\newcommand{\resdiag}{\widehat{\Sigma}_{0,S_tS_t}}

\newcommand{\so}{\widehat{\Sigma}_0}
\newcommand{\sot}{\widehat{\Sigma}_{0,t}}

\newcommand{\AI}{\widehat{A}(\cI)}

\newcommand{\cI}{{\mathcal{I}}}

\DeclareMathOperator{\Tr}{Tr}
\newcommand{\0}{{\mathbf{0}}}

\newcommand{\ve}{{\mathrm{vec}}}

\newcommand{\bx}{{\bar{X}}}

\newcommand{\E}{{\mathbf{E}}}

\newcommand{\op}{{\mathrm{op}}}
\newcommand{\rank}{{\mathrm{rank}}}

\newcommand{\col}{\mathrm{Col}}
\newcommand{\tZ}{{\widetilde{Z}}}
\newcommand{\tA}{{\widetilde{A}}}
\newcommand{\tB}{{\widetilde{B}}}

\newcommand{\tF}{{\widetilde{F}}}
\newcommand{\tJ}{{\widetilde{J}}}

\newcommand{\tlambda}{{\widetilde{\Lambda}_r}}
\newcommand{\sfirst}{{\widehat{\Sigma}}}
\newcommand{\szerofirst}{{\widehat{\Sigma}_0}}
\newcommand{\ssecond}{{\widehat{\Sigma}}}
\newcommand{\szerosecond}{{\widehat{\Sigma}_0}}

\newcommand{\szerothird}{{\widehat{\Sigma}_0}}

\newtheorem{theorem}{Theorem}[section]
\newtheorem{lemma}[theorem]{Lemma}
\newtheorem{proposition}[theorem]{Proposition}

\newtheorem{corollary}[theorem]{Corollary}
\newtheorem{definition}[theorem]{Definition}
\newtheorem{remark}[theorem]{Remark}
\newtheorem{assumption}[theorem]{Assumption}
\newtheorem{example}[theorem]{Example}
\usepackage{url} 

\title{Sparse GCA and Thresholded Gradient Descent}
\author{Sheng Gao\footnote{Email: shenggao@wharton.upenn.edu.} 
~and Zongming Ma\footnote{Email: zongming@wharton.upenn.edu.}\\
~\\
\textit{University of Pennsylvania}}
\date{\today}

\begin{document}

\maketitle

\begin{abstract}
Generalized correlation analysis (GCA) is concerned with uncovering linear relationships across multiple datasets. 
It generalizes canonical correlation analysis that is designed for two datasets.
We study sparse GCA when there are potentially multiple leading generalized correlation tuples in data that are of interest and the loading matrix has a small number of nonzero rows.
It includes sparse CCA and sparse PCA of correlation matrices as special cases.
We first formulate sparse GCA as a generalized eigenvalue problem at both population and sample levels via a careful choice of normalization constraints.
Based on a Lagrangian form of the sample optimization problem, we propose a thresholded gradient descent algorithm for estimating GCA loading vectors and matrices in high dimensions.
We derive tight estimation error bounds for estimators generated by the algorithm with proper initialization.
We also demonstrate the prowess of the algorithm on a number of synthetic datasets.
\end{abstract}
\tableofcontents

\section{Introduction}

With the advent of big data acquisition technology, 
it has become increasingly important to integrate information across 
multiple datasets collected on a common set of subjects. 
Canonical correlation analysis (CCA), first proposed by \citet{hotelling1992relations}, is a widely used statistical tool to integrate information from two datasets: It seeks linear combinations of variables within each dataset such that their correlation is maximized.

However, recent advances in 
fields such as multi-omics and multimodal brain imaging have presented us with new challenges, since scientists are often able to collect more than two datasets on the same set of subjects nowadays.
To tackle these challenges, we turn to
a useful generalization of CCA called \textit{generalized correlation analysis} (GCA) \citep{kettenring1971canonical} which 
aims to explore linear relationships across multiple data sources. 
\citet{kettenring1971canonical} proposed five different variants for generalized correlation analysis of multiple datasets, 
where
different methods correspond to maximization of different objective functions of covariances and correlations, 
subject to certain normalization constraints.
\citet{tenenhaus2011regularized} extended 
these approaches
to regularized versions by adding $l_2$ penalties to the objective functions and proposed partial least squares algorithms for solving them. 
In addition to unsupervised settings, it has been shown that GCA can be incorporated into supervised learning settings to improve efficiency and accuracy \citep{shen2014generalized}.

Suppose that the $k$ datasets are i.i.d.~realizations of 
$k$ random vectors ${ X_{\dc{1}}}\in\mathbb{R}^{p_1}, 
{X_{\dc{2}}}\in\mathbb{R}^{p_2}, ..., {X_{\dc{k}}}\in\mathbb{R}^{p_k}$. 
{Here and after, we use subscript $\{j\}$ to denote the $j$th set of features and/or the $j$th dataset.}
At the population level, 
we propose to seek vectors ${ a_{\dc{1}}},{a_{\dc{2}}},...,.{a_{\dc{k}}}$ 
(called \textit{the first $k$-tuple of generalized loading vectors}) that solve
\begin{equation}
      \label{eq:pop-gca}
\begin{aligned}
& \underset{{l_{\dc{1}}, \dots, l_{\dc{k}}}}{\text{maximize}}
& & \sum_{i=1}^k \sum_{j=1}^k\text{cov}({ l_{\dc{i}}^\top} { X_{\dc{i}}},{ l_{\dc{j}}^\top} {X_{\dc{j}}}) \\
& \text{subject to}
& & \sum_{i=1}^k \text{var}({ l_{\dc{i}}^\top} { X_{\dc{i}}})=1.
\end{aligned}
\end{equation}
An important difference between the foregoing formulation and those in \citet{kettenring1971canonical} is the normalization constraint. 
We propose to normalize the sum of the variances of the linear combinations while \citet{kettenring1971canonical} 
requires
the individual variance of each linear combination to be one.
When $k = 2$, the two normalizations essentially lead to the same solution.
One can show that the optimal solution to \eqref{eq:pop-gca} necessarily has $\text{var}({ a_{\dc{i}}^\top  X_{\dc{i}}}) = \frac{1}{2}$ for $i=1$ and $2$, and hence the optimal ${ a_{\dc{i}}}$'s are proportional to those from normalizing individual variances. 
However, when $k \geq 3$, the solutions are usually different and we argue that 
formulation 
\eqref{eq:pop-gca} is more appropriate for two reasons.
First, the normalization in \eqref{eq:pop-gca} conforms with the ``null principle'' \citep{donnell1994analysis} in that if all linear combinations of the $X_i$'s are uncorrelated then the objective function reduces to the lefthand side of the normalization constraint.
In addition, it is more resilient to spurious solutions as illustrated by the following example. 

\begin{example}
\label{example:normalization}
Suppose $k = 3$ and $p_1= p_2 = p_3 = m$. 
Let $X_{\dc{1}} = Y v_{\dc{1}} + Z_{\dc{1}} $, 
$X_{\dc{2}} = Y v_{\dc{2}} + Z_{\dc{2}}$  and $X_{\dc{3}} = Z_{\dc{3}}$ where $v_{\dc{1}}$ and $v_{\dc{2}}$ are two deterministic unit vectors in 
$\mathbb{R}^m$, 
$Y\sim N(0,1)$, 
$Z_{\dc{i}} \stackrel{iid}{\sim} N_m(0, I_m)$, 
and they are mutually independent. 
Simple linear algebra shows that the optimal solution to \eqref{eq:pop-gca} is 
$a_{\dc{1}} =  \frac{1}{2}v_{\dc{1}}$, 
$a_{\dc{2}} =  \frac{1}{2}v_{\dc{2}}$ and 
$a_{\dc{3}} = 0$. 
In contrast, with the individual normalization constraint 
$\mathrm{var}(l_{\dc{i}}^\top  X_{\dc{i}}) = 1$, the optimal solution would change to $a_{\dc{1}} = \frac{1}{\sqrt{2}}v_{\dc{1}}$, 
$a_{\dc{2}} = \frac{1}{\sqrt{2}}v_{\dc{2}}$ and 
$a_{\dc{3}}$ is any vector in $\mathbb{B}^{m-1}$ 
where $\mathbb{B}^{m-1}$ is the unit ball in $\mathbb{R}^m$.
Comparing the solutions to two different normalizations, formulation \eqref{eq:pop-gca} is advantageous in that it provides a meaningful and unique optimal $a_{\dc{3}}$.
\end{example}

The solution to \eqref{eq:pop-gca} gives the first $k$-tuple of leading GCA loading vectors. 
More generally, we want to extract successive $k$-tuples of leading loading vectors subject to certain additional
constraints.
In view of \eqref{eq:pop-gca}, suppose we have found $({ a_{\dc{1}}^{(m)}}, \dots, {a_{\dc{k}}^{(m)}})$, 
 for $m=1,\dots, r-1$, as the first $(r-1)$ $k$-tuples of leading loading vectors. 
We define the $r$th $k$-tuple $({a_{\dc{1}}^{(r)}}, \dots, {a_{\dc{k}}^{(r)}})$ as the solution to \eqref{eq:pop-gca} with the following additional constraints:
\begin{equation*}
      \sum_{i=1}^k \text{cov}(({l_{\dc{i}}^{(m)}})^\top  X_{ \dc{i}}, {l_{\dc{i}}^\top}  X_{\dc{i}}) = 0, \qquad
      m = 1,\dots, r-1.
\end{equation*}
To be more concise, let the columns of $L_{\dc{i}}\in \mathbb{R}^{p_i\times r}$ 
be the successive leading loading vectors for $X_{\dc{i}}$ for $i=1,\dots, k$.
The foregoing proposal for finding the first $r$ 
leading $k$-tuples then reduces to the following optimization problem
\begin{equation}
      \label{eq:pop-gca-mat1}
\begin{aligned}
& \underset{L_{\dc{1}},\dots, L_{\dc{k}}}{\text{maximize}}
& & \sum_{i,j}\Tr({L_{\dc{i}}^\top \Sigma_{\dc{ij}} L_{\dc{j}}}) \\
& \text{subject to}
& & \sum_{i=1}^k {L_{\dc{i}}^\top \Sigma_{\dc{ii}}L_{\dc{i}}}=I_r.
\end{aligned}
\end{equation}
Here $\Sigma_{\dc{ii}}$ 
denotes the covariance matrix of $X_{\dc{i}}$ and $\Sigma_{{\dc{ij}}}=\text{cov}(X_{\dc{i}},X_{\dc{j}})$ for $i\neq j$. 
To write \eqref{eq:pop-gca-mat1} more concisely, we define 
$L = [{ L_{\dc{1}}^\top} ,\dots, { L_{\dc{k}}^\top} ]^\top  \in \mathbb{R}^{p\times r}$ 
where $p = \sum_{i=1}^k p_i$, 
$\Sigma_0 = \mathrm{diag}(\Sigma_{\dc{11}},\dots, \Sigma_{\dc{kk}})\in \mathbb{R}^{p\times p}$, and $\Sigma \in \mathbb{R}^{p\times p}$ where its $(i,j)$th block is $\Sigma_{ \dc{ij}}$ and $i$th diagonal block is $\Sigma_{\dc{ii}}$.
Then we can rewrite \eqref{eq:pop-gca-mat1} as 
\begin{equation}
      \label{eq:pop-gca-mat2}
\begin{aligned}
& \underset{L}{\text{maximize}}
& & \Tr(L^\top \Sigma L) \\
& \text{subject to}
& & L^\top \Sigma_0L=I_r. 
\end{aligned}
\end{equation}
In what follows, $A = [{ A_{\dc{1}}^\top},\dots, { A_{\dc{k}}^\top}]^\top$ denotes the solution to \eqref{eq:pop-gca-mat2} where each $A_{\dc{i}}\in \mathrm{R}^{p_i\times r}$ and the $m$th column of $A_{\dc{i}}$ is $a^{(m)}_{\dc{i}}$ for $i=1,\dots,k$ and $m=1,\dots,r$.

Informed readers might have already realized that \eqref{eq:pop-gca-mat2} is closely connected to a generalized eigenvalue problem 
$|\Sigma - \lambda \Sigma_0| = 0$,
where $|M|$ stands for the determinant of a square matrix $M$.
In addition, when $\Sigma_0$ is diagonal (i.e., when $k=p$), the problem is equivalent to principal component analysis (PCA) of the correlation matrix:
For any non-singular $\Sigma_0$, the columns of $\Sigma_0^{1/2}A$ are the leading eigenvectors of the correlation matrix $\Sigma_0^{-1/2}\Sigma \Sigma_0^{-1/2}$.
When $k = 2$, the solution to \eqref{eq:pop-gca-mat2} is equivalent to CCA up to scaling.
Thus, program \eqref{eq:pop-gca-mat2} provides a 
unified
formulation for extracting leading linear covariation within one or across multiple datasets.

In practice,
one does not have direct knowledge of the joint sample covariance matrix $\Sigma$ or even its block diagonal part $\Sigma_0$.
The natural choice is to replace them with their sample counterparts.
In potential modern applications of GCA, the dataset dimensions (i.e., the $p_i$'s) can be much larger than the sample size $n$. 
Hence, one suffers from the curse of dimensionality if no further structural assumption is made \citep{donoho2000high,johnstone2001distribution,bickel2008covariance,bickel2008regularized}.
Due to its interpretability and practicality, a structural assumption 
that has been widely adopted in both theory and practice
is \emph{sparsity}: Most energy of the solution to \eqref{eq:pop-gca}--\eqref{eq:pop-gca-mat2} concentrates on a small number of entries \citep{hastie2015statistical,wainwright2019high}.
Let $A$ denote the solution to \eqref{eq:pop-gca-mat2}.
In this manuscript, we adopt the assumption that at most $s$ rows of $A$  contain nonzero entries. 
In other words, the target of estimation is also the solution to the following \emph{sparse generalized correlation analysis} (Sparse GCA, or SGCA) problem:
\begin{equation} 
      \label{eq:pop-sgca}
\begin{aligned}
& \underset{L}{\text{maximize}}
& & \Tr(L^\top \Sigma L) \\
& \text{subject to}
& & L^\top \Sigma_0L=I_r\quad \|L\|_{2,0}\leq s.
\end{aligned}
\end{equation}
Here and after, for any matrix $L$, $\|L\|_{2,0}$ counts the number of nonzero rows in $L$.
In view of the discussion following \eqref{eq:pop-gca-mat2}, when $k=p$ and $k=2$, \eqref{eq:pop-sgca} reduces to sparse PCA of correlation matrix and sparse CCA, respectively.

\subsection{Main contributions}

The main contributions of the present manuscript are the following.

First, we clarify the target of estimation in generalized correlation analysis (i.e., the solution $A$ to \eqref{eq:pop-gca-mat2}) by considering a natural latent variable model in which an $r$-dimensional latent variable drives the covariation of $k$ random vectors.
Under mild conditions,
we show that the linear subspace spanned by the leading $r$ generalized eigenvectors in our GCA formulation (i.e., the columns of $A$) coincides with the subspace spanned by a concrete functional of parameters in the latent variable model.
In addition, we characterize the behavior of generalized eigenvalues under the latent variable model.

Next, for sample sparse GCA, we propose a thresholded gradient descent algorithm for solving a Lagrangian version of the sample counterpart of \eqref{eq:pop-sgca}. 
The algorithm is intuitive and easy to implement. 
In view of the discussion following \eqref{eq:pop-gca-mat2}, such an algorithm provides a unified approach to a number of different sparse unsupervised learning problems, including sparse PCA of correlation matrices and sparse CCA.

Furthermore, we provide a theoretical analysis of the thresholded gradient descent algorithm. 
We show that with high probability, when initialized properly, estimation errors of the intermediate results after each iteration converge at a geometric rate until they arrive in a tight neighborhood of the population solution to \eqref{eq:pop-sgca}. 
Statistically, we establish tight estimation error bounds for the output of our algorithm as an estimator.
Numerically, this implies geometric convergence of our algorithm to such an estimator.
Finally, these theoretical findings are corroborated by numerical studies on simulated datasets.

\subsection{Related works} 

In view of the discussion following \eqref{eq:pop-gca-mat2}, the present paper is closely connected to the sparse PCA and sparse CCA literature.
To date, there is a large literature devoted to various aspects of the sparse PCA problem, including algorithms \citep{Zou06,JohnstoneLu09,Amini09,Witten09,ma2013sparse,yuan2013truncated,vu2013fantope,wang2014tighten}, information-theoretic limits \citep{Birnbaum12,cai2013sparse,Vu12sp} and computational theoretic limits \citep{Berthet13,gao2017sparse}. 
However, this literature has mostly focused on sparse PCA of covariance matrices. 
Thus, the special case of $k = p$ in our setting complements the existing literature by providing both theory and method for sparse PCA of correlation matrices.

In the case of $k = 2$, the theory and method in this paper specialize to the sparse CCA setting \citep{chen2013sparse,gao2015minimax,gao2017sparse}. 
In this case, 
we provide
a new iterative algorithm for sparse CCA that achieves optimal estimation rates derived in \citet{gao2015minimax}.
Therefore, for this special case, the present manuscript provides a competitive alternative to existing sparse CCA methods.

When $r=1$, \eqref{eq:pop-sgca} reduces to the population version of the sparse generalized eigenvalue problem considered in \citet{tan2018sparse}.
\citet{tan2018sparse} proposed a truncated Rayleigh flow method for estimating the population solution in this special case, and established its rates of convergence.
It is not clear how their method can be generalized to estimate successive generalized eigenvectors or eigenspaces.
In contrast, our estimator is based on a different algorithm (thresholded gradient descent of a Lagrangian objective function).
It not only achieves fast converging estimation rates for estimating the first generalized eigenvector but also works for estimating leading generalized eigenspaces of fixed dimensions that are greater than one.

\subsection{Organization of the paper}

The rest of this paper is organized as follows. 
In section 2 we examine the generalized eigenvalue problem underpinning sparse GCA under a latent variable model.
Section 3 proposes a thresholded gradient descent algorithm and its initialization via generalized Fantope projection. 
Section 4 establishes the convergence rate of our algorithm under reasonable initialization. 
Numerical results are presented in Section 5. 
Technical proofs are deferred to appendices.

\subsection{Notation}
\label{sec:notation}

For any set $S$, let $|S|$ denote its size and $S^c$ denote its complement. 
For any event $E$, $\mathbf{1}_E$ is its indicator function. 
For a vector $u$, $\|u\|={(\sum_i u_i^2)^{1/2}}$, $\|u\|_0=\sum_i \mathbf{1}_{u_i\neq 0}$, $\|u\|_1=\sum_i |u_i|$, $\|u\|_\infty = \max_i |u_i|.$ For any matrix $A=(a_{ij})$, the $i$th row of $A$ is denoted by $A_{i*}$ and the $j$th column by $A_{*j}$. 
We let $\text{Col}(A)$ denote the span of columns of $A$. 
For a positive integer $m$, $[m]$ denotes the index set ${1, 2,\dots ,m}$. 
For two subsets $I$ and $J$ of indices, we write $A_{IJ}$ for the $|I|\times|J|$ submatrices formed by $a_{ij}$ with $(i,j)\in I \times J$. 
When $I$ or $J$ is the whole set, we abbreviate it with an $*$, and so if $A \in \mathbb{R}^{p\times k}$, then $A_{I*} = A_{I[k]}$ and $A_{*J} =A_{[p]J}$. 
For any square matrix $A = (a_{ij})$, denote its trace by $\Tr(A) = \sum_i a_{ii}$. 
Moreover, 
we denote the collection of all
 $p\times r$ matrices with orthonormal columns by $\mathcal{O}(p,r)$, and abbreviate $\mathcal{O}(r,r)$ by
${\mathcal{O}(r)}$.
The set of $p\times p$ symmetric matrices is denoted by $\mathbb{S}^p$.
Furthermore, $\sigma_i(A)$ stands for the $i$th largest singular value of $A$ and
$\sigma_{\max}(A) = \sigma_1(A)$, $\sigma_{\min}(A) = \sigma_{\min\{p,k\}}(A)$. 
The Frobenius norm and the operator norm of $A$ are $\|A\|_\mathrm{F} =({\sum_{i,j}a_{ij}^2})^{1/2}$ and $\|A\|_\op = \sigma_1(A)$, respectively. 
The infinity norm of $A$ is defined as $\|A\|_\infty = \max_{i, j} |a_{ij}|$. 
The $l_1$ norm and the nuclear norm of a matrix $A$ are $\|A\|_1=\sum_{i,j}|a_{ij}|$ and $\|A\|_*=\sum_i \sigma_i(A)$, respectively. 
Similarly, $\eigen_i(A)$ stands for the $i$th largest eigenvalue of $A\in \mathbb{R}^{p\times p} $, while
$\eigen_{\max}(A) = \eigen_1(A)$ and $\eigen_{\min}(A) = \eigen_{p}(A)$.
The support of $A$ is
defined as $\mathrm{supp}(A) = \{i\in[n] : \|A_{i*}\| > 0\}$. 
For two symmetric matrices $A$ and $B$, we write $A\preceq B$ if $B-A$ is positive semidefinite. 
For any positive semi-definite matrix $A$, $A^{1/2}$ denotes its principal square root that is positive semi-definite and satisfies $A^{1/2}A^{1/2} = A$. 
The trace inner product of two matrices $A, B \in \mathbb{R}^{p\times k}$ is $\langle A, B\rangle  = \Tr(A^\top B)$. 
For any real numbers $a$ and $b$, let $a\wedge b = \min(a,b)$ and $a\vee b = \max(a,b)$.
For any two sequences of positive numbers $\{a_n\}$ and $\{b_n\}$, we write $a_n = O(b_n)$ if $\limsup_{n\to\infty} a_n/b_n$ is finite.
Given a random element $X$, $\mathcal{L}(X)$ denotes its probability distribution. The symbol $C$ and its variants $C_1, C'$, etc.~are generic positive constants
and may vary from occurrence to occurrence, unless otherwise specified. The symbols $\mathbb{P}$ and $\mathbb{E}$ stand for generic probability and expectation when the distribution is clear from the context.

\section{A Latent Variable Model}
\label{sec:latent}

In this part, we aim to identify the solutions to \eqref{eq:pop-gca-mat2} as functionals of parameters in the joint distribution of 
$({ X_{\dc{1}}^\top} ,\dots, { X_{\dc{k}}^\top} )^\top $.
To this end, we introduce an intuitive latent variable model.

The fundamental assumption underlying generalized correlation analysis is that there exists a shared low-dimensional latent variable which orchestrates the (linear) covariation of observed features across all datasets.
Let $z$ be an $r$-dimensional latent variable. 
The following latent variable model describes an idealized data generating process:
\begin{equation}
      \label{eq:latent-var-model}
\begin{aligned}
      X_{\dc{i}} & = U_{ \dc{i}} z + e_{\dc{i}},\quad i=1,\dots,k,\\
      z   & \sim N_r(0, I_r),
      \quad
      e_{\dc{i}}  \stackrel{ind}{\sim} N_{p_i}(0, \Psi_{\dc{ii}}).
\end{aligned}
\end{equation}
Here the deterministic matrix $U_{\dc{i}}\in \mathbb{R}^{p_i\times r}$ for $i=1,\dots, k$, $\Psi_{\dc{ii}}$ 
is positive definite, and the latent variable $z$ and the idiosyncratic noises $\{e_{\dc{i}}\}_{i=1}^k$ are mutually independent.
Under the foregoing latent variable model, 
the joint covariance matrix of 
$({X_{\dc{1}}^\top} ,\dots, {X_{\dc{k}}^\top} )^\top \in \mathbb{R}^p$ is given by $\Sigma$ with the $i$th diagonal block
\begin{equation}
      \label{eq:Sigma-i}
      \Sigma_{\dc{ii}}=U_{\dc{i}}U_{\dc{i}}^\top +\Psi_{\dc{ii}},\quad \mbox{for $i=1,\dots,k$,}
\end{equation}
and the $(i,j)$th block
\begin{equation}
      \label{eq:Sigma-i-2}
    \Sigma_{\dc{ij}}={ U_{\dc{i}}}{U_{\dc{j}}^\top} ,\quad \mbox{for $1\leq i\neq j \leq k$}.
\end{equation}
We denote by $\Psi$ the block diagonal matrix with blocks ${\Psi_{\dc{ii}}}$, $i=1,\dots, k$, 
on the diagonal. 
We also let $U=[{ U_{\dc{1}}^\top}, { U_{\dc{2}}^\top} ,...,{ U_{\dc{k}}^\top}  ]^\top  \in \mathbb{R}^{p\times r}$.
In the rest of this section, 
we assume that the observed datasets are generated by model \eqref{eq:latent-var-model}.
In addition, we let 
\begin{equation}
      \label{eq:lambda-def}
\lambda_1\geq \lambda_2 \geq \cdots \geq \lambda_p\geq 0
\end{equation}
denote the population generalized eigenvalues of $\Sigma$ with respect to $\Sigma_0 = \mathrm{diag}(\Sigma_{\dc{11}},\dots,\Sigma_{\dc{kk}})$.

The following key lemma identifies the connection between parameters in model \eqref{eq:latent-var-model} and the solution to \eqref{eq:pop-gca-mat2}.


\begin{lemma}\label{lemma:latent_var}
Suppose that $\Sigma$ and $\Sigma_0$ are specified by model \eqref{eq:latent-var-model}--\eqref{eq:Sigma-i-2}. 
Let $A = [{ A_{\dc{1}}^\top},\dots, { A_{\dc{k}}^\top}]^\top$ be the solution to \eqref{eq:pop-gca-mat2}.
If the $r$th generalized eigenvalue of $\Sigma$ w.r.t.~$\Sigma_0$ is larger than $1$, i.e., $\lambda_r > 1$, 
then $\col(A_{\dc{i}}) \subset \col({\Sigma_{\dc{ii}}^{-1}}U_{\dc{i}})$ for all $i\in [k]$. 
For any $i$, if further $\rank(A_{\dc{i}})=r$ , then
$\col (A_{\dc{i}})=\col ({\Sigma_{\dc{ii}}^{-1}}U_{\dc{i}})$.
\end{lemma}






The next lemma describes the behavior of population generalized eigenvalues
under latent variable model \eqref{eq:latent-var-model}. 
It also identifies a sufficient condition for $\lambda_r > 1$.
To this end, we start with
an assumption motivated by 
\cite{fan2019estimating}.

\begin{assumption}
      \label{ass:latent}
In latent variable model \eqref{eq:latent-var-model}--\eqref{eq:Sigma-i-2}, 
we assume the following:
\begin{itemize}
\item The matrix $U$ has full column rank, that is, $\rank(U)=r$;
\item $\Sigma_0$ and $U$ satisfy $\sigma_r(\Sigma_0^{-1/2}U) \geq 1$.
\end{itemize}
\end{assumption}

\begin{lemma}
      \label{lem:eigengap}
In model \eqref{eq:latent-var-model}--\eqref{eq:Sigma-i-2},
under Assumption \ref{ass:latent},
we have \begin{equation}
r=\max\{j: \lambda_j>1, j\in[p]\}.
\end{equation}
In other words, there are exactly $r$ generalized eigenvalues greater than 1. 
Moreover, define 
$Y=[{ U_{\dc{1}}^\dagger}, { U_{\dc{2}}^\dagger}, \dots, { U_{\dc{k}}^\dagger}]^\top \in\mathbb{R}^{p\times r}$
where ${ U_{\dc{i}}^\dagger}$ is the 
Moore-Penrose inverse of ${ U_{\dc{i}}}$. 
Then the multiplicity of $1$ as a generalized eigenvalue is
\begin{equation}
\#\{j:\lambda_j=1\} =p-\sum_{i=1}^k \rank(U_{\dc{i}})+r- \rank(U-UY^\top U).
\end{equation}
\end{lemma}

The foregoing lemma guarantees an eigengap between the $r$th and the $(r+1)$th generalized eigenvalues, and so the leading rank $r$ generalized eigenspace is well-defined at the  population level.

By the foregoing lemma, under model \eqref{eq:latent-var-model} we can decompose $\Sigma$ as
\begin{equation}
      \label{eq:Sigma-decomp}
\Sigma=\Sigma_0 K { \Lambda} K^\top \Sigma_0=\Sigma_0 A\Lambda_r A^\top \Sigma_0+\Sigma_0 B {\LambdaRest} B^\top \Sigma_0.
\end{equation}
Here ${ \Lambda}, \Lambda_r$, and ${\LambdaRest}$ are all diagonal where $\Lambda_r=\text{diag}(\lambda_1,\lambda_2,...,\lambda_r)$ collects the first $r$ generalized eigenvalues, 
${\LambdaRest}$ collects the remaining $p-r$ generalized eigenvalues, 
and ${\Lambda} = \mathrm{diag}(\Lambda_r, {\LambdaRest})$. 
In addition, $B$ collects the eigenvectors associated with the bottom $p-r$ generalized eigenvalues.
 On the other hand, the decomposition \eqref{eq:Sigma-decomp} hold as long as the $r$th and $(r+1)$th generalized eigenvalues are distinct and hence is more general than model \eqref{eq:latent-var-model}.
\begin{remark}
      \label{rmk:cca}
In the case when $k=2$, The covariance matrices between $X$ and $Y$ can be reparameterized as $\Sigma_{xy}=\Sigma_xV{\Theta_r} W^\top\Sigma_y$ with $V^\top \Sigma_xV=W^\top \Sigma_yW= { I_r}$ \cite{gao2015minimax,gao2017sparse}. 
{Here $\Theta_r=\mathrm{diag}(\theta_1,\theta_2,...,\theta_r)$ collects the leading $r$ canonical coefficients for CCA and $\Theta_r +I_r= \Lambda_r$. Hence we have $\lambda_r = \theta_r+1$.}
The solution $A=[{ A_{\dc{1}}^\top}, {A_{\dc{2}}^\top}]^\top$ to \eqref{eq:pop-gca-mat2} satisfies that $A_{\dc{1}}=\frac{1}{\sqrt{2}}V, A_{\dc{2}}=\frac{1}{\sqrt{2}}W$. 
Such a relationship does not hold in general for $k\geq 3$.
\end{remark}

\begin{remark}
      \label{rmk:pca}
When $k=p$, $\Sigma_0$ becomes a diagonal matrix where the diagonal entries are variances of the variables.
The problem \eqref{eq:pop-gca-mat2} is then equivalent to finding the leading eigenspace of correlation matrix of the data, defined as $R =\Sigma_0^{-1/2}\Sigma\Sigma_0^{-1/2}$. 
Let $E_R\in\mathbb{R}^{p\times r}$ be the matrix containing the eigenvectors that span the leading $r$ dimensional eigenspace. 
If $A$ is the solution to \eqref{eq:pop-gca-mat2}, then $\Sigma_0^{1/2}A$ coincides with $E_R$ (up to an $r\times r$ rotation matrix when there is any generalized eigenvalue with multiplicity larger than one). 
Thus, we essentially estimate the leading eigenspace of the correlation matrix.
\end{remark}

\section{Gradient Descent with Hard Thresholding}
\label{sec:alg}

In this section, we present a thresholded gradient descent algorithm for simultaneously finding multiple leading sparse generalized eigenvectors. 

\subsection{Motivation}

The sample counterpart of \eqref{eq:pop-sgca} can be recast as the following minimization problem:
\begin{equation*} 
\begin{aligned}
& \underset{L}{\text{minimize}}
& & -\langle \widehat{\Sigma}, LL^\top \rangle  \\
& \text{subject to}
& & L^\top \widehat\Sigma_0L=I_r,\quad \|L\|_{2,0}\leq s.
\end{aligned}
\end{equation*}
Here $\widehat{\Sigma}$ is the joint sample covariance matrix of all $k$ sets of features.
We propose to solve its Lagrangian version:
\begin{equation}
      \label{eq:obj-lagrangian}
    \underset{{L\in\mathbb{R}^{p\times r}}}{\text{minimize}} \quad f(L)\quad\quad\quad\text{subject to}\quad\|L\|_{2,0}\leq s,
\end{equation}
where the objective function $f(L)$ is 
\begin{equation}
      \label{eq:obj-lag-2}
f(L)=-
   \langle \widehat{\Sigma},LL^\top \rangle +\frac{\lambda}{2}\|L^\top \widehat{\Sigma}_0L-I_r\|_\mathrm{F}^2.
\end{equation}
Effectively, minimizing the first term in \eqref{eq:obj-lag-2} maximizes the original objective function
whereas minimizing the second term controls the deviation from the normalization constraint in \eqref{eq:pop-gca-mat2}. 
Here $\lambda$ is a tuning parameter. 
Intuitively, the larger it is, the more penalty we put on deviating from the normalization constraint and the closer $L^\top \widehat{\Sigma}_0 L$ is to $I_r$. 

To deal with the constraint in \eqref{eq:obj-lagrangian}, we shall perform the following hard thresholding.

\begin{definition}
Given a matrix $U$ and a natural number $k$, 
we define the output of hard thresholding function $HT(U,k)$ 
to be the matrix obtained by keeping the $k$ rows with the largest $l_2$ norms and replacing all other rows with zeros, 
that is
\begin{equation*}
    HT(U,k)_{i *}= 
\begin{cases}
    U_{i*},& \text{if } i\in C_k\\
    0,              & \text{otherwise}
\end{cases}
\end{equation*}
where $C_k$ 
is the index set of $k$ rows of $U$ with largest $l_2$ norms.
When there is a tie, we always pick the smaller/smallest index.
\end{definition}

\subsection{Algorithm}
\label{sec:tgd_alg}
Let $X^{(1)},X^{(2)}...,X^{(n)}$ be i.i.d.~observations generated from the latent variable model with covariance matrix $\Sigma$ and its block diagonal part $\Sigma_0$. 
Let $\widehat\Sigma$ and $\widehat\Sigma_0$ be the sample covariance matrix and its block diagonal part computed on $\{ X^{(1)},X^{(2)}...,X^{(n)} \}$.

Algorithm \ref{TGD} describes the proposed thresholded gradient descent procedure.
Given a proper initial estimator which we shall specify later, each iteration first performs a step of gradient descent on $f$ defined in \eqref{eq:obj-lagrangian}, then it keeps the $s'$ rows with the largest $l_2$ norms and thresholds the remaining rows to zero.
By iteratively performing these two steps, the algorithm can be viewed as heuristics for solving the non-convex optimization problem \eqref{eq:pop-sgca}. 
Here $s'$ is a user-specified tuning parameter and is not necessarily equal to the true sparsity level $s$.
 Step 2 in Algorithm \ref{TGD} is a re-normalization step. 
This is due to the fact that the stationary point of the objective function $f(L)$ is not the same as the solution to the original optimization problem \eqref{eq:pop-sgca}. 
Thus, we first transform it for iterations and eventually transform back in Step 7 to define the final estimator.

\begin{algorithm}[H]
    \caption{Thresholded gradient descent for sparse GCA}\label{TGD}
    \hspace*{\algorithmicindent}\textbf{Input}: Covariance matrix estimator $\ssecond$ and its block diagonal part $\szerosecond$; Initialization $\bA_0$. \\
    \hspace*{\algorithmicindent}\textbf{Tuning Parameters}: 
      Step size $\eta$; Penalty $\lambda$;  Sparsity level $\hs$; Number of iterations $T$; 
\begin{algorithmic}[1]
      \State $\norminit \gets \bA_0(\bA_0^\top\so\bA_0)^{-1/2}$
      \State $\gd_1\gets\norminit (I+\frac{1}{\lambda}\norminit ^\top\ssecond\norminit )^{1/2}$
    \For {$t=1,2,3,...,T$} 
    \State $ V_{t+1} \gets \gd_t-\eta\nabla f(\gd_t)=\gd_t-2\eta(-\ssecond\gd_t+\lambda \szerosecond\gd_t(\gd_t^\top\szerosecond\gd_t-I_r))$
    \State $\gd_{t+1}\gets HT(V_{t+1},\hs)$
    \EndFor
\textbf{Output}: $\widehat{A}_T=\gd_T(\gd_T^\top\so\gd_T)^{-1/2}$

    \end{algorithmic}
    \end{algorithm}

In Algorithm \ref{TGD}, each iteration is computationally efficient: line 4 is essentially matrix multiplication and addition, and line 5 requires calculating and sorting row $l_2$ norms. 
Since $\widehat{\Sigma} = \frac{1}{n}\sum_{i=1}^n(X^{(i)}-\bx)(X^{(i)}-\bx)^\top$ where $\bx = \frac{1}{n}\sum_{i=1}^n X^{(i)}$, an efficient way to calculate $\widehat{\Sigma}\gd_t$ would be to calculate 
$Y_t^{(i)} = (X^{(i)}-\bx)^\top \gd_t$ first, followed by multiplying $ (X^{(i)}-\bx) Y_t^{(i)}$,
resulting in $O(npr)$ flops. 
Line 5 requires $O(rp+p\log p)$ flops for selecting top $s'$ rows. 
Hence a single iteration of thresholded gradient descent will require 
$O(npr + p\log p)$ flops in total. 

\subsection{Initialization via generalized Fantope projection}
\label{sec:init}

Success of Algorithm \ref{TGD} depends crucially on the quality of the initial estimator $\bA_0$.
Thus, we need to find an initial estimator that is relatively close to the true generalized eigenspace in some distance so that later iterations could further improve on estimation accuracy. 
We now introduce such an initial estimator based on Fantope projection \citep{vu2013fantope}.

Note that 
\begin{equation*}
    \sum_{i,j=1}^k \Tr(A_{\dc{i}}^\top\Sigma_{\dc{ij}}A_{\dc{j}}) 
      =\langle \Sigma,AA^\top \rangle.
\end{equation*}
The idea is to ``lift'' $AA^\top $ into $\mathbb{S}^p$, the space of $p\times p$ symmetric matrices, and hence to treat it as a single quantity. 
Since we assume that $A$ is row sparse, $F=AA^\top $ have at most $s^2$ nonzero entries which is much smaller than its number of elements. 
To ensure sparsity of its solution,
we impose an entrywise $l_1$ penalty to define the following objective function for initialization:
\begin{equation*}
   \underset{F\in \mathbb{S}^{p}}{\text{minimize}} ~~-\langle \sfirst,F\rangle+\rho \|F\|_{1}.
\end{equation*}
On the other hand, the normalization constraint on $A$ implies that
\begin{equation*}
    (\szerofirst)^{1/2}F(\szerofirst)^{1/2}\in \mathcal{P}(p,r).
\end{equation*} 
where 
$\mathcal{P}(p,r)$ is defined to be the set of rank $r$ projection matrices
\begin{equation*}
    \mathcal{P}(p,r)=\{PP^\top , P\in \mathcal{O}(p,r)\}.
\end{equation*} 
However, this direct generalization leads to a nonconvex feasible set since $\mathcal{P}(p,r)$ is nonconvex.
To obtain a bona fide convex program,
in the light of \cite{vu2013fantope}, we use the Fantope set introduced by \cite{dattorroconvex}: 
\begin{equation*}
    \mathcal{F}_r=\{X:0\preceq X\preceq I \quad\text{and}\quad \Tr{(X)}=r\}.
\end{equation*}
The motivation is the observation from \cite{maxeigenvalue}  that
$\mathcal{F}_r=\text{conv}(\mathcal{P}(p,r))$,
where $\text{conv}(A)$ denotes the convex hull of $A.$
To summarize, our initial estimator is the solution to the following program:
\begin{equation}
      \label{eq:initial-fantope}
\begin{aligned}
    \underset{F\in \mathbb{S}^{p}}{\text{minimize}} &~~-\langle \sfirst,F\rangle+\rho \|F\|_1,\\
      \mbox{subject to} &~~(\szerofirst)^{1/2}F(\szerofirst)^{1/2}\in \mathcal{F}_r.
\end{aligned}
\end{equation}

Upon obtaining $\widehat{F}$ as the solution to \eqref{eq:initial-fantope}, we collect the leading $r$ eigenvectors of $\widehat{F}$ as $\widetilde{U}_r\in \mathbb{R}^{p\times r}$ and the corresponding leading $r$ eigenvalues as entries of the diagonal matrix $\widetilde{D}_r\in \mathbb{R}^{r\times r}$.
Then let $A_0 = \widetilde{U}_r \widetilde{D}_r^{1/2}$ and 
\begin{equation}
      \label{eq:initial-A}
\widehat{A}_0  = HT(A_0,s').
\end{equation}
This finishes initialization for Algorithm \ref{TGD}.
The initialization procedure has $O(p^3)$ computation complexity due to the ADMM step involved in solving the generalized Fantope projection. See \cite{gao2017sparse} for details of the ADMM algorithm.

As we shall show in next section, 
$\widehat{F}$
suffers a relatively large estimation error rate for estimating $AA^\top $ and hence  $\widehat{A}_0$ for $A$. 
However, for Algorithm \ref{TGD} to work, such an estimator serves well as an initial estimator under mild conditions. 
In addition to PCA and the current setting, an analogous initialization via convex relaxation idea has appeared in \cite{gao2017sparse} for performing sparse CCA which is asymmetric.

%

\section{Theoretical Results}
\label{sec:theorem}

We provide theoretical justifications for our algorithms in this section. 
We first state our main result on how each iteration of Algorithm \ref{TGD} improves estimation accuracy,
{followed by a corollary on error bounds achieved by the final estimator}. 
 Analysis of generalized Fantope initialization follows our investigation of the main algorithm.
In addition, we include a lower bound for the finite $k$ setting at the end of this section.

\paragraph{Parameter space}
Under 
covariance structure \eqref{eq:Sigma-decomp}, 
we define 
 $\mathcal{F}(\{s_i\}_1^k,\{p_i\}_1^k,r,\{\lambda_j\}_1^p;\nu)$
as the collection of all covariance matrices $\Sigma$ satisfying \eqref{eq:Sigma-decomp} and the following conditions:
\begin{equation}\label{parameter space}
\begin{aligned}
&\text{(i) Sparsity: $A_{\dc{i}}\in \mathbb{R}^{p_i\times r}$ with $\|A_{\dc{i}}\|_{2,0}\leq s_i$;}\\
&\text{(ii) Bounded spectrum}: 
 \frac{1}{\nu}\leq{\eigen_{\min}(\Sigma)<\eigen_{\max}(\Sigma)}\leq \nu, \frac{1}{\nu}\leq{\eigen_{\min}(\Sigma_0)<\eigen_{\max}(\Sigma_0)}\leq\nu;\\
&\text{(iii) Eigengap}: \lambda_{r}-\lambda_{r+1}>0.\\
\end{aligned}
\end{equation}
Here $\lambda_i$'s are generalized eigenvalues as defined in \eqref{eq:lambda-def}.
By Section \ref{sec:latent},  under model \eqref{eq:latent-var-model} the eigengap condition is satisfied if Assumption \ref{ass:latent} holds.
The parameter space is then defined  as
\begin{equation}
      \label{eq:para-space}
\begin{aligned}
\mathcal{P}_n( \{s_i\}_1^k,\{p_i\}_1^k,r,\{\lambda_j\}_1^p;\nu)& =\big\{\mathcal{L}(X^{(1)},X^{(2)}...,X^{(n)}): X^{(i)}\stackrel{iid}{\sim} N(0,\Sigma),\\ 
&\quad\qquad\Sigma\in\mathcal{F}( \{s_i\}_1^k,\{p_i\}_1^k,r,\{\lambda_j\}_1^p;\nu) 
\big\}.
\end{aligned}
\end{equation}
In what follows, $S$ denotes the true row support of $A$. 
While the parameter space requires normality, our theoretical analysis generalizes directly to sub-Gaussian distributions. We omit the generalization in this work as it is mostly formality.


\paragraph{Matrix distance}
To measure estimation accuracy, we define the distance between two matrices $U$ and $V\in\mathbb{R}^{p\times r}$ as
\begin{equation}\label{eq:mat-dist}
\mathrm{dist}(U,V)=\min_{P\in\mathcal{O}(r)}\|UP-V\|_\mathrm{F}.
\end{equation}
Here $\mathcal{O}(r)$ is the collection of $r\times r$ orthogonal matrices.
The matrix distance has been used previously in  \cite{ge2017no}, \cite{tu2015low} and \cite{golub2012matrix}, among others.

\subsection{Main results}
\label{sec:main-result}

The following theorem characterizes numerical convergence of Algorithm \ref{TGD} when starting at a reasonable initializer. 
\begin{theorem}
\label{mainthe}
In Algorithm \ref{TGD}, set
\begin{equation}
      \label{eq:eta-lambda}
    \eta \leq\frac{c}{12\lambda_1\rev{\nu}(1+c)^2},\quad \lambda = \frac{\lambda_1}{c}
\end{equation}
for some constant $c$, and
\begin{equation}
      \label{eq:s-prime}
\hs \geq \frac{256 s\rev{\nu^2}}{(\lambda_r-\lambda_{r+1})^2\eta^2} \vee s. 
\end{equation} 
For initial estimator $\bA_0$, suppose that it has row sparsity $\hs$, and that for 
\begin{equation}
      \label{eq:V}
\scale=A \left(I+\frac{1}{\lambda}\Lambda_r \right)^{\frac{1}{2}}, 
\end{equation}
$\bA_0$ is so constructed that after the first two lines of Algorithm \ref{TGD}, we have
\begin{equation}\label{eq:radius}
    \mathrm{dist}(\scale,\gd_1)\leq \frac{1}{8\sqrt{\rev{\nu}}}\min\left\{\frac{c(\lambda_r -\lambda_{r+1})}{\sqrt{2}\lambda_1\rev{\nu^2}(42+25c)},\frac{\sqrt{1+c}}{2}\right\}.
\end{equation}
If \eqref{parameter space} holds and
\begin{equation}
      \label{sample size condition}
\frac{\sqrt{1+\lambda_1^2}\sqrt{{1+\lambda_{r+1}^2}}}{(\lambda_r-\lambda_{r+1})^2}\sqrt{\frac{rs\log p}{n}}<c_0
\end{equation}
for some sufficiently small constant $c_0>0$,
then for some constants $C,C'>0$, 
uniformly over $\mathcal{P}_n = \mathcal{P}_n( \{s_i\}_1^k,\{p_i\}_1^k,r,\{\lambda_j\}_1^p;\rev{\nu})$, 
with probability at least $1-\exp(-C'(s'\log(ep/s')))$,
for all $t\geq 1$, 
\begin{equation}\label{stats error and opt error}
    \mathrm{dist}(\scale, \gd_{t+1})\leq \underbrace{C\left(\frac{s'}{s}\right)^{3/2}\frac{\sqrt{1+\lambda_1^2}\sqrt{{1+\lambda_{r+1}^2}}}{\lambda_r-\lambda_{r+1}}\sqrt{\frac{rs\log p}{n}}}_{\text{Statistical Error}}+\underbrace{\xi^t \mathrm{dist}(\scale, \gd_{1})}_{\text{Optimization Error}}
\end{equation}
where $\xi=1-\frac{\eta^2(\lambda_r-\lambda_{r+1})^2}{64\rev{\nu^2}}$.
%
%
\end{theorem}

The foregoing theorem leads to the following corollary on high probability error bounds for estimating $A$.

\begin{corollary}
      \label{cor: normalized estimate}
Suppose the conditions of Theorem \ref{mainthe} hold.
For each $t \geq 1$,
let $\festimate=\gd_t(\gd_t^\top\so\gd_t)^{-1/2}$.
Then for some constants $C,C_1, C'>0$, 
uniformly over $\mathcal{P}_n$, with probability at least $1-\exp(-C'(s'\log(ep/s')))$,
for all $t\geq T_0$, where \begin{equation}\label{iteration}
T_0=\frac{\log\left(C\left(\frac{s'}{s}\right)^{3/2}\frac{\sqrt{1+\lambda_1^2}\sqrt{{1+\lambda_{r+1}^2}}}{\lambda_r-\lambda_{r+1}}\sqrt{\frac{rs\log p}{n}}\right)}{\log(1-\frac{\eta^2(\lambda_r-\lambda_{r+1})^2}{64\rev{\nu^2}})},
\end{equation} 
we have
\begin{equation*}
\mathrm{dist}(A,\festimate)\leq C_1\left(\frac{s'}{s}\right)^{3/2}\frac{\sqrt{1+\lambda_1^2}\sqrt{{1+\lambda_{r+1}^2}}}{\lambda_r-\lambda_{r+1}}\sqrt{\frac{rs\log p}{n}}.
\end{equation*}
In particular, when $T \geq T_0$, the last display holds for the output $\widehat{A}_T$ of Algorithm \ref{TGD}.
\end{corollary}


\par  We give some brief remarks on Theorem \ref{mainthe} and Corollary \ref{cor: normalized estimate}.
The error bound for each step, 
given by the right side of \eqref{stats error and opt error},
is composed of two parts: 
a statistical error term independent of the iteration counter $t$,
and an optimization error that decreases geometrically as $t$ increases. When $t\geq T_0$, the optimization error is dominated by the statistical error and we can achieve the estimation rate in Corollary \ref{cor: normalized estimate}. 
Thus, $T_0$ can be interpreted as the minimum number of iterations needed for achieving the estimation error rate in Corollary \ref{cor: normalized estimate}. 
When the eigengap $\lambda_r-\lambda_{r+1}$ is lower bounded by a positive constant, it is straightforward to verify that $T_0 = O(\log(p+n))$.
In other words, thresholded gradient descent drives down the statistical error of the estimator to the desired rate after $O(\log(p+n))$ iterates.
It is worth noting that the contraction rate $\xi$ does not depend on ambient dimension $p$, which can be much larger than $r$ and $s$.

Condition \eqref{eq:radius} on initial estimator is related to the notion of ``basin of attraction''.
It has previously appeared in the literature of machining learning with nonconvex optimization in other contexts, e.g., \cite{chen2015solving,chi2019nonconvex,wang2014tighten}. 
The global landscape of \eqref{eq:obj-lagrangian} could be hard to handle due to non-convexity and sparsity constraints. 
However, within the basin of attraction, the objective function $f(\cdot)$ is locally smooth and strongly convex, which allows projected gradient descent to find a statistically sound solution at geometric convergence rate. 

\paragraph{Outline of proof}
To prove Theorem \ref{mainthe}, we track the estimator trajectory $\{\overline{V}_t: t\geq 0\}$ over iteration.
To this end, we first characterize a global high probability event under which the entire trajectory would lie inside the basin of attraction within which the objective function is smooth and strongly convex.
On the event, we analyze in sequel the effects of the gradient step and the hard thresholding step in each iteration. 
In particular, 
we show that each gradient step drives down the distance between the target of estimation and the current estimator while the hard thresholding step projects the current estimator onto the feasible set. 
For details, see Propositions \ref{prop: gradana} and \ref{prop: ht}, respectively.
Combining the two propositions, we obtain a recursive inequality that characterizes the estimator trajectory as in \eqref{stats error and opt error}.
This completes the major steps in the proof. 
Finally, 
Corollary \ref{cor: normalized estimate} is a direct consequence of normalization on the same event that we performed the foregoing analysis of iteration. 
For proof details, see Appendix \ref{sec:proof-main}.


\subsection{Analysis of initialization} 
The following theorem characterizes the estimation accuracy of the proposed initial estimator via generalized Fantope projection.

\begin{theorem}\label{theo:init}
Suppose 
\begin{equation}
      \label{init: sample size}
    \frac{s^2\log p}{n(\lambda_r-\lambda_{r+1})^2}\leq \epsilon
\end{equation} 
 for some sufficiently small $\epsilon > 0$.
Let $\widehat{F}$ be the solution to 
\eqref{eq:initial-fantope} 
 where $\rho=\gamma\sqrt{\frac{\log p}{n}}$ for $\gamma\in[\gamma_1,\gamma_2]$ for some positive constants $\gamma_1<\gamma_2$.
There exist constants $C,C'>0$ such that 
uniformly over $\mathcal{P}_n$, with probability at least $1-\exp(-C'(s+\log(ep/s)))$,
\begin{equation}
\|\widehat{F}-AA^\top  \|_\mathrm{F}^2\leq C\frac{(\sum_{i=1}^{k}s_i)^2\log(\sum_{i=1}^{k}p_i)}{n(\lambda_r-\lambda_{r+1})^2}
=O\left(
\frac{s^2\log p}{n(\lambda_r-\lambda_{r+1})^2} 
\right).
\end{equation}
\end{theorem}

The following corollary further bounds the estimation accuracy of $\widehat{A}_0$ through the lens of $\overline{V}_1$ defined in line 2 of Algorithm \ref{TGD}. 

\begin{corollary}
      \label{cor:init}
Suppose that the conditions of Theorem \ref{theo:init} hold.
{Under the choice of $s'$ and condition \eqref{sample size condition} in Theorem \ref{mainthe},}
there exist constants $C,C'>0$ such that 
uniformly over $\mathcal{P}_n$, with probability at least $1-\exp(-C'(s+\log(ep/s)))$,
\begin{equation*}
\dist(\scale,\gd_1)\leq 
\frac{C s}{\lambda_r-\lambda_{r+1}}\sqrt{\frac{\log p}{n}}
\end{equation*}
where $V$ is defined in \eqref{eq:V}.
\end{corollary}

Combining Theorems \ref{mainthe} and \ref{theo:init} and Corollaries \ref{cor: normalized estimate} and \ref{cor:init}, 
we obtain the following corollary on the whole procedure: Algorithm \ref{TGD} with initialization via generalized Fantope projection.

\begin{corollary}
      \label{cor: whole procedure}
Suppose that 
\begin{equation}
\label{eq:overall-n-req}
n \geq C_0 
\max\left\{
\frac{(1+\lambda_1^2)(1+\lambda_{r+1}^2)}{(\lambda_r - \lambda_{r+1})^4}
rs\log p,\,
\frac{s^2\log p }{(\lambda_r-\lambda_{r+1})^2}
\right\}
\end{equation}
for some sufficiently large positive constant $C_0$,
that $T \geq T_0$ with $T_0$ in \eqref{iteration},
and that all the other conditions in Theorems \ref{mainthe} and \ref{theo:init} are satisfied.
There exist constants $C,C'>0$ such that 
uniformly over $\mathcal{P}_n$, with probability at least $1-\exp(-C'(s+\log(ep/s)))$,
the final output satisfies
\begin{equation*}
\mathrm{dist}(A,\widehat{A}_T)\leq C_1\left(\frac{s'}{s} \right)^{3/2}\frac{\sqrt{1+\lambda_1^2}\sqrt{{1+\lambda_{r+1}^2}}}{\lambda_r-\lambda_{r+1}}\sqrt{\frac{rs\log p}{n}}.
\end{equation*}
\end{corollary}

\subsection{A lower bound for finite $k$}
\label{sec:lower_bound}
Corollary  \ref{cor: normalized estimate} provides upper bounds for the proposed procedure in Algorithm \ref{TGD}. 
For finite $k$, we have the following information-theoretic lower bound result.

\begin{theorem}
\label{thm:lower_bound_gca}
Assume that $1\leq r \leq \frac{\min_i s_i}{2}$, and that \begin{equation}
    n(\lambda_r-\lambda_{r+1})^2\geq C_0\left(r+\max_i\log\frac{ep_i}{s_i}\right)
    \label{eq:lower_bound_condition}
\end{equation}
for some sufficiently large positive constant $C_0$. Then there exist positive constants $c$ and $c_0$ such that the minimax risk for estimating $A$ satisfies\begin{equation}
    \inf_{\widehat{A}}\sup_{\Sigma\in\mathcal{F}} \E_\Sigma \dist^2(A, \widehat{A})\geq c_0 \wedge  \frac{c}{n(\lambda_r-\lambda_{r+1})^2}\left(rs+\sum_{i=1}^k s_i\log\frac{ep_i}{s_i}\right).    \label{eq:lower_bound}
\end{equation}
\end{theorem}

The proof of Theorem \ref{thm:lower_bound_gca} is given in Appendix \ref{sec:proof_lower_bound}. 
Comparing Theorem \ref{thm:lower_bound_gca} and Corollary \ref{cor: normalized estimate}, we see that when $r$ is finite and $\lambda_1\leq C$ for some large positive constant $C$, as long as there exists $i\in \{1,\dots, k\}$ such that $s_i$ and $p_i$ are of the same order as $s$ and $p$ simultaneously, the lower and upper bounds match.
Otherwise, they differ by at most a multiplicative factor of $\log p$.

\section{Numerical Results}

This section reports numerical results on synthetic datasets.
Except for the settings in Section \ref{sec:misspec} and Section \ref{sec:exp_general_covariance},
$r$, the latent dimension in model \eqref{eq:latent-var-model} is assumed to be known. 
Choices of tuning parameters are specified in each setting.
In practice,
they can also be selected using cross-validation on grids. 
The rest of this section is organized as follows. 
We first study numerical errors for estimating generalized eigenspaces of different dimensions. 
{For the special case of $r=1$, we also compare our method with the Rifle method in \cite{tan2018sparse}.}
Next, we assess the performance of Algorithm \ref{TGD} in the context of sparse CCA by comparing it to the CoLaR method in \cite{gao2017sparse}. 
Furthermore, we consider the potential model mis-specification scenario where the input latent dimension of the algorithm is different from the true value.
We then apply Algorithm \ref{TGD} to perform sparse PCA of correlation matrices. 
Finally, we investigate the performance of Algorithm \ref{TGD} under a general covariance structure. 

\subsection{Sparse GCA with different latent dimensions}
\label{sec:exp_sparse_gca}

We first consider sparse GCA of three high dimensional datasets. 
In particular, we set $n=500$, $k=3$, $p_2=p_3=200$, $p_1=500$, and $s_1=s_2=s_3=5$. 
To generate covariance matrices $\Sigma$ and $\Sigma_0$, we use the latent variable model specified in Section \ref{sec:latent}. 
Specifically, $\Sigma$ is  a block matrix with
\begin{equation*}
\Sigma_{\dc{ij}}=T_{\dc{i}}U_{\dc{i}}U_{\dc{j}}^\top T_{\dc{j}},\quad \Sigma_{\dc{ii}}=T_{\dc{i}},\quad\text{for $i\neq j \in\{1,2,3\}$},
\end{equation*}\begin{equation*}
    U_{\dc{1}}^\top T_{\dc{1}}U_{\dc{1}}=U_{\dc{2}}^\top T_{\dc{2}}U_{\dc{2}}=U_{\dc{3}}^\top T_{\dc{3}}U_{\dc{3}}=I.
\end{equation*} 
Here each $\Sigma_{\dc{ii}}=T_{\dc{i}}$ is a Toeplitz matrix, defined by setting $(T_{\dc{k}})_{ij}=\sigma_{k_{ij}}$ where $\sigma_{k_{ij}}=a_k^{|i-j|}$ for all $i,j\in [p_k]$ with $a_1=0.5, a_2=0.7, a_3=0.9$. 
To generate $U_{\dc{i}}\in\mathbb{R}^{p_i\times r}$, we first randomly select a support of size $5$. 
For each row in the support, we generate its entries as i.i.d.~standard normal random variables.  
Then all $U_{\dc{i}}$'s are normalized with respect to $T_{\dc{i}}$. 
With the foregoing construction, it is straightforward to verify that $\lambda_{r}-\lambda_{r+1}=2$ and that $\text{Col}(A_{\dc{i}})=\text{Col}(U_{\dc{i}})$. 
Finally, $\Sigma_0 = \mathrm{diag}(\Sigma_{\dc{11}},\Sigma_{\dc{22}},\Sigma_{\dc{33}})$ contains the block diagonal elements of $\Sigma$.

We vary $r$ in $\{1,2,3,4,5\}$ { and report squared matrix distances defined as $\mathrm{dist}^2(A,\widehat{A}) = \min_{O\in\mathcal{O}(r)}\|\widehat{A}O-A\|_\mathrm{F}^2$ for both initial and final estimators based on $50$ repetitions in each setting.}
For tuning parameters in Algorithm \ref{TGD}, we set 
$\hs=20$, $\eta=0.001$, $\lambda = 0.01$, and $T=15000$.
The tuning parameter for generalized Fantope initialization is set to be $\rho = \frac{1}{2}\sqrt{\frac{\log p}{n}}$. 
{The truncation parameter for initialization is also set to be $s' = 20$. }

Table \ref{simgca} reports the results of the aforementioned simulation study. 
For all latent dimensions, we observe a significant decrease in estimation error after Algorithm \ref{TGD} is applied. 
This corroborates the theory in Section \ref{sec:theorem}.
{To better understand how each iteration of Algorithm \ref{TGD} improves estimation, we plot $\mathrm{dist}(V,\overline{V}_t)$ {and $\mathrm{dist}(A,\widehat{A}_t)$} in logarithmic scale 
against the iteration counter $t$ for $r=1,2, 3, 4,5$ and $n = 500$. 
For any $t$, we set $\widehat{A}_t = \overline{V}_t (\overline{V}_t^\top \widehat{\Sigma}_0 \overline{V}_t)^{-1/2}$ which agrees with the definition of $\widehat{A}_T$ in the last line of Algorithm \ref{TGD}.
From Figure \ref{fig:log_plot}, we observe an approximate linear decay trend at the beginning in all cases, which corresponds to exponential decay in the original scale.
{Moreover, after sufficiently many iterations, all error curves plateau, which suggests that the performance of the resulting estimators have stabilized.}}
Both phenomena agree well with the theoretical findings in Theorem \ref{mainthe}.

\begin{table}[h]
\centering
\begin{tabular}{c|ccccc}
\hline
\multicolumn{1}{c|}{Dimension} & $r=1$ & $r=2$ & $r=3$ & $r=4$ & $r=5$ \\ 
\hline
Initial Error & 0.1319(0.0737) & 0.2308(0.0591) & 0.2746(0.0652) & 0.2354(0.0483) &  0.1969(0.0237) \\ 
Final Error & 0.0015(0.0030) & 0.0072(0.0209) & 0.0098(0.0301) & 0.0121(0.0071) & 0.0171(0.0061) \\ \hline
\end{tabular}
\caption{Median errors of initial (generalized Fantope) and final (Algorithm \ref{TGD}) estimators in squared matrix distance out of 50 repetitions. 
Median absolute deviations of errors are reported in parentheses. }
\label{simgca}
\end{table}

\begin{figure}[!tb]
\centering
\includegraphics[width = 0.48\textwidth]{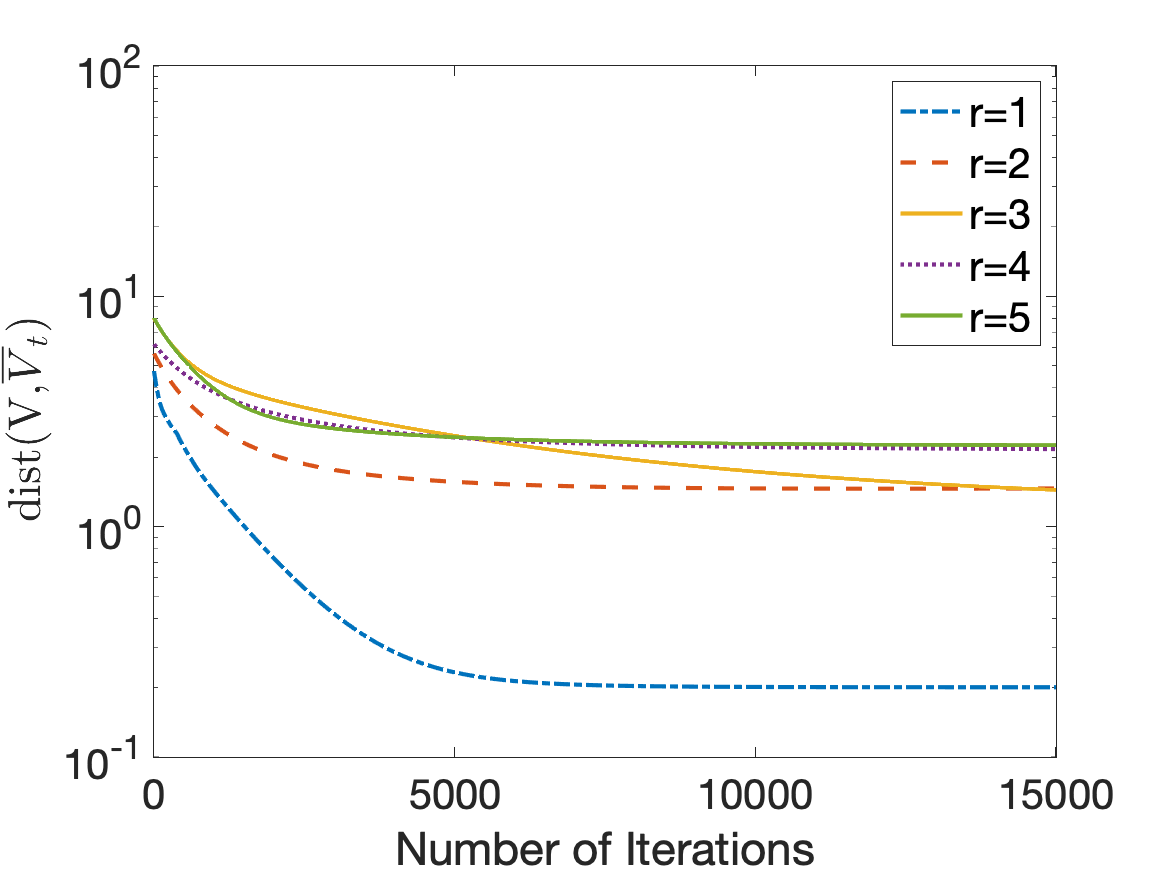}
\includegraphics[width = 0.48\textwidth]{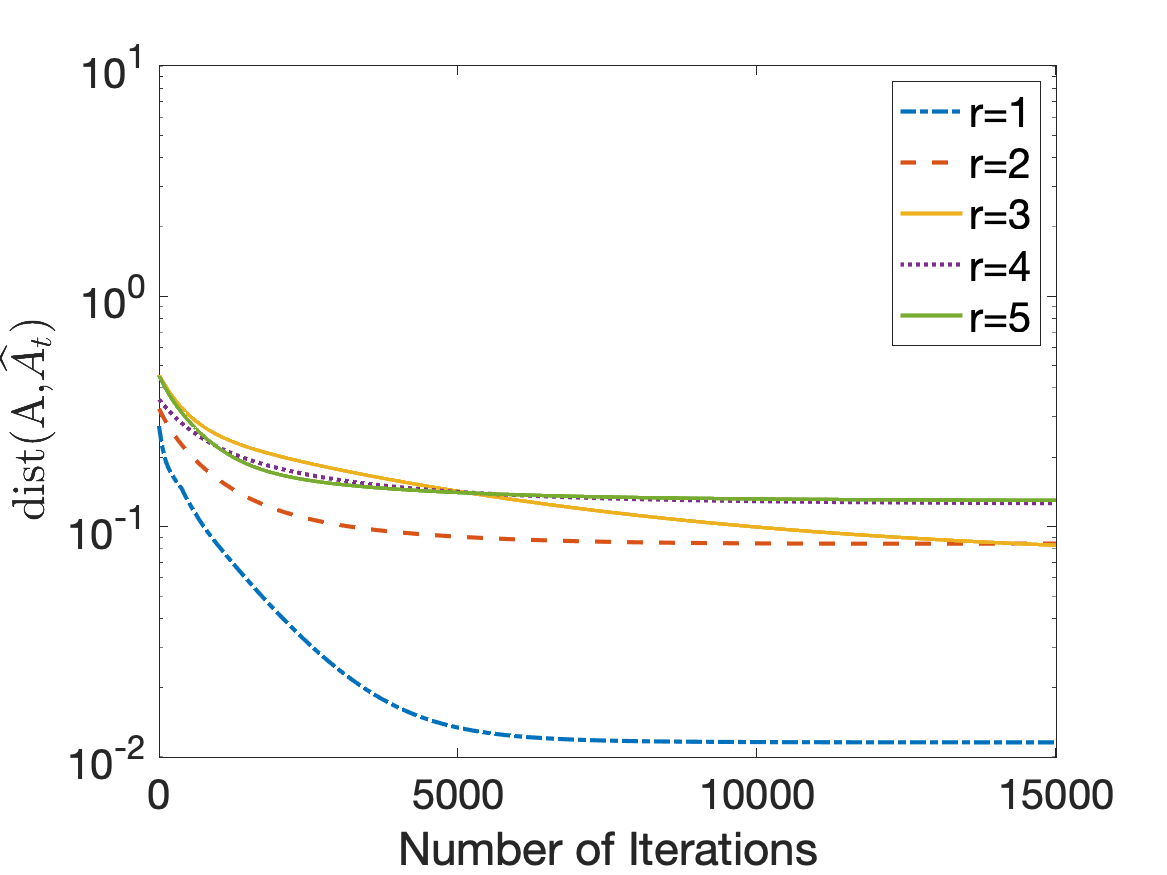}
\caption{Plots of $\mathrm{dist}(V,\overline{V}_t)$ (left) and $\mathrm{dist}(A,\widehat{A}_t)$ (right) (both in logarithmic scale) vs.~iteration counter $t$. 
}
\label{fig:log_plot}
\end{figure}

\citet{tan2018sparse} proposed a truncated Rayleigh flow method, called Rifle, to solve the sparse generalized eigenvalue problem when $r=1$. 
Here we compare our method and Rifle under the same experiment setting as above, with $n = 500$ and $1000$. 
{For TGD we set the tuning parameters to be $s'=20, \eta = 0.01, \lambda = 0.01$ and $T = 15000$. For Rifle we also use $s'=20$ for truncation, the step size $\eta$ is set to be the default value $0.01$ and we run it for sufficiently many iterations until estimator errors no longer improve.}
We vary the sample size $n$ in $\{500, 1000\}$ and each experiment setting is repeated 50 times. 
We report median estimation errors of both initial and final estimators of both methods measured in squared matrix distances in Table \ref{simgca_rilfe}. 
{Both Rifle and our method produce highly accurate final estimators, which significantly decrease errors of respective initializers. 
Final estimators by our algorithm are slightly more accurate in both settings.}

\begin{table}[!tb]
\centering
\begin{tabular}{c|cc|cc}
\hline
\multirow{2}{*}{} & \multicolumn{2}{c|}{Rifle} & \multicolumn{2}{c}{TGD} \\ \cline{2-5} 
 & Initial Error & Final Error & Initial Error & Final Error \\ \hline
$n=500$ & 0.1225 (0.0562) & 0.0006 (0.0005) & 0.1458 (0.0713) & 0.0003 (0.0011) \\ \cline{1-1}
$n=1000$ & 0.0805 (0.0449) & 0.0003 (0.0012) & 0.0920 (0.0504) & 0.0002 (0.0003) \\ \hline
\end{tabular}
\caption{
Median errors of initial and final estimators of both Rifle and Algorithm \ref{TGD} in squared matrix distance out of 50 repetitions. 
Median absolute deviations of errors are reported in parentheses.
}
\label{simgca_rilfe}
\end{table}

\subsection{Choices of tuning parameters}
\label{sec:tuning_parameter_exp}
In this section, we study the selection of tuning parameters. 
For the sparsity level $s'$ in the hard thresholding step of Algorithm \ref{TGD}, we show that a five-fold cross validation approach works reasonably well.
For the Lagrangian multiplier parameter $\lambda$, we show that the estimation procedure is robust with respect to it and we recommend fixing it at some small positive constant.

For simulations below, we fix $n=500$ and $p = 900$. 
In addition, we fix other tuning parameters throughout the experiments as follows: $\rho = \frac{1}{2}\sqrt{\frac{\log p}{n}}$ for Fantope initialization and $\eta=0.001$, $T=15000$ for Algorithm \ref{TGD}.

\paragraph{Choice of sparsity level $s'$.} 
The procedure for the selection of $s'$ is as follows. 
We first randomly split the data $X$ into five folds of equal sizes. 
For $l=1, \dots, 5$, we use one fold as the test set $X^{\mathrm{test}}_{(l)}$ and the other four folds combined as the training set $X^{\mathrm{train}}_{(l)}$. For each value of $s'$ in a pre-specified grid $G'$, we apply Algorithm \ref{TGD} on $X^{\mathrm{train}}_{(l)}$ to obtain an estimator $\widehat{A}^{\mathrm{train}}_{(l)}$. 
Then we compute the test GCA score defined as 
\begin{equation*}
\mathrm{CV}_{(l)}(s') = \Tr\left(\widehat{A}^{\mathrm{train}^\top}_{(l)} \widehat{\Sigma}_{(l)}^{\mathrm{test}}\widehat{A}^{\mathrm{train}}_{(l)}\right).
\end{equation*}
Here $\widehat{\Sigma}_{(l)}^{\mathrm{test}}$ is the sample covariance matrix of the test data $X^{\mathrm{test}}_{(l)}$. 
Finally, we compute the cross-validation score for $s'$, defined as $\mathrm{CV}(s') = \frac{1}{5}\sum_{l=1}^{5} \mathrm{CV}_{(l)} (s')$. 
Upon obtaining the cross-validation scores for all $s'\in G'$, we select the sparsity level that maximizes $\mathrm{CV}(s')$.

In the experiments here, we set the 
grid $G'$ as $\{5i:i=1,\dots, 20\}$. 
We fix $r=3, \lambda = 0.01$ and consider three cases: $s_1=s_2=s_3=5$ (Case I), $s_1=s_2=s_3=15$ (Case II) and $s_1=s_2=s_3=20$ (Case III). 
The true sparsity levels in three cases are then $s = 15, 45$, and $60$, respectively. 
The results from the 5-fold cross validation procedure for the three cases are plotted in Figure \ref{Figure:tuning-sparsity}. 
Here, we plot $\mathrm{CV}(s')$ against $s'\in G'$ and use error-bars to indicate the range of $\pm$ one standard deviation of cross-validation scores. 

\begin{figure}[htp]

\centering
\includegraphics[width=.33\textwidth]{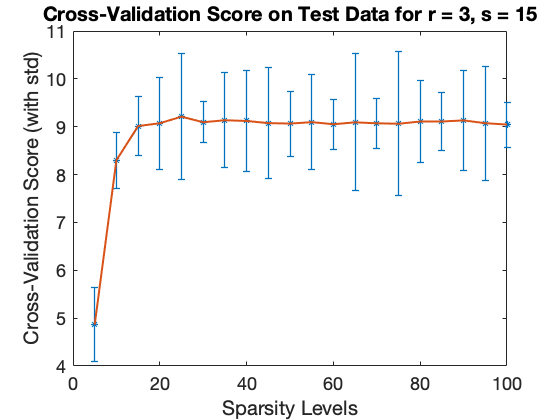}\hfill
\includegraphics[width=.33\textwidth]{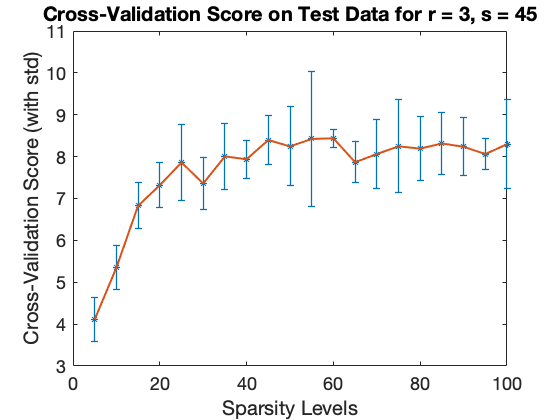}\hfill
\includegraphics[width=.33\textwidth]{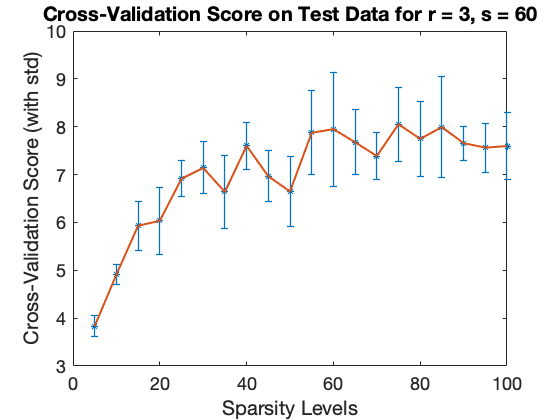}\hfill

\caption{Plots of cross-validation score $\mathrm{CV}(s')$ against $s'\in G'$ for case I ($s=15$, left), case II ($s=45$, middle), and case III ($s=60$, right).
Curves indicate average over five folds and error bars indicate one standard deviations.
}
\label{Figure:tuning-sparsity}

\end{figure}

\par Figure \ref{Figure:tuning-sparsity} shows that in all three cases, the cross-validation scores increase with $s'$ until reaching a plateau level and fluctuate around some constant when $s'$ is set to larger values. 
The sparsity levels selected by this procedure are slightly larger than the true sparsity levels in all three settings. 
Specifically, the selected sparsity levels for the three cases are $25, 60$, and $75$, respectively. 
This suggests that our previous choice of $s'=20$ for $s=15$ in Section \ref{sec:exp_sparse_gca} is reasonable. 

\paragraph{Choice of penalty $\lambda$.} We now study the impact of different choices of $\lambda$ on our estimator. 
To this end, we vary $\lambda$ on a log scale and consider all values in the set $\{0.0001, 0.001, 0.01, 0.1, 1\}$ and compare the final estimation errors across these choices. 
We fix $s'=20$ and test using the same simulation setting as in Section \ref{sec:exp_sparse_gca} with $r=3$ and $4$. 
We set all other tuning parameters in the same way and only change $\lambda$.
Table \ref{table:tune_lambda} reports squared distances from estimators to true parameter values based on $50$ repetitions in each $(r,\lambda)$ value combination. 
From Table \ref{table:tune_lambda}, we notice that for both $r=3$ and $r=4$, the final estimation error is robust with respect to the choice of $\lambda$. 

\begin{table}[!tb]
\centering
\begin{tabular}{c|cccccc}
\hline
\multicolumn{1}{c|}{Penalty $\lambda$} & $0.0001$ & $0.001$ & $0.01$ & $0.1$ & $1$  \\ 
\hline
$r=3$ & 0.0121(0.0364) & 0.0136(0.0252) & 0.0098(0.0301) & 0.0093(0.0140)  & 0.0099(0.0108) \\ \cline{1-1}
$r=4$ & 0.0145(0.0068) & 0.0152(0.0102) & 0.0121(0.0071) & 0.0101(0.0055) & 0.0119(0.0093) & \\ \hline
\end{tabular}
\caption{Median errors of the final (Algorithm \ref{TGD}) estimators in squared matrix distance out of 50 repetitions for $r= 3, 4$ and different choices of  $\lambda$. 
Median absolute deviations of errors are reported in parentheses. }
\label{table:tune_lambda}
\end{table}

\subsection{Sparse CCA}

To further compare the performance of our algorithm with a benchmark, we compare its performance in sparse CCA problems with the CoLaR method proposed by \cite{gao2017sparse}. 
Instead of using gradient descent and hard thresholding, \citet{gao2017sparse} refined the initial estimate by a linear regression with group Lasso penalty. In addition, their estimator was shown to achieve the optimal estimation rate for canonical loading matrix under a prediction loss that is different from \eqref{eq:mat-dist}.

For fair comparison, we use the same simulation settings as in \cite{gao2017sparse}.
Using the notation in Remark \ref{rmk:cca},
we set 
\begin{equation*}
    p_1=p_2,\quad \Sigma_x=\Sigma_y=M,\quad r=2,\quad \theta_1=0.9,\quad \theta_2=0.8.
\end{equation*}
The covariance matrix between $X$ and $Y$ is generated from the canonical pair model $\Sigma_{xy}=\Sigma_xV \Theta_2 W^\top \Sigma_y$ 
{Recall that from the relationship between $\Theta_2$ and $\Lambda_2$ we obtain that $\lambda_1=\theta_1+1=1.9$ and $\lambda_2=\theta_2+1=1.8$. Hence the leading two generalized eigenvalues are $1.9$ and $1.8.$}
Moreover, the row supports for both $V$ and $W$ are set to be $\{1,6,11,16,21\}$. The values at the nonzero coordinates are obtained by normalizing (with respect to $M = \Sigma_x = \Sigma_y$) random numbers drawn from the uniform distribution on the finite set $\{-2,-1,0,1,2\}$. 
We consider three different choices of $M$ as follows: 
(i) Identity: $M=I$; 
(ii) Toeplitz: $M=(\sigma_{ij})$ where $\sigma_{ij}=0.3^{|i-j|}$ for all $i,j$; 
(iii) SparseInv: $M={\sigma^0_{ij}}/{\sqrt{\sigma_{ii}^0\sigma_{jj}^0}}$ where $\Sigma^0=(\sigma_{ij}^0)=\Omega^{-1}$ for $\Omega=\omega_{ij}$ with
\begin{equation*}
        \omega_{ij}=\mathbf{1}_{i=j}+0.5\times\mathbf{1}_{|i-j|=1}+0.4\times \mathbf{1}_{|i-j|=2}\quad i,j\in[p_1].
   \end{equation*}
We consider the same four configurations of $(n,p_1,p_2)$ as in \cite{gao2017sparse}. 
For CoLaR and its initialization, we set all tuning parameters to those values used in \cite{gao2017sparse}. 
For tuning parameters in Algorithm \ref{TGD} and its initialization, we set $\rho = \frac{1}{2}\sqrt{\frac{\log p}{n}}$,
$\hs=20$, $\lambda=0.01$, $\eta=0.001$, and $T=10000$.
Upon obtaining the final estimate $\widehat{A}=[\widehat{A}_{\dc{1}}^\top, \widehat{A}_{\dc{2}}^\top]^\top$ from Algorithm \ref{TGD}, 
we calculate estimators for $V$ and $W$ as $\widehat{V} = \widehat{A}_{\dc{1}}(\widehat{A}_{\dc{1}}^\top\widehat{\Sigma}_x\widehat{A}_{\dc{1}})^{-1/2}$ and $\widehat{W} = \widehat{A}_{\dc{2}}(\widehat{A}_{\dc{2}}^\top\widehat{\Sigma}_y\widehat{A}_{\dc{2}})^{-1/2}$.

We summarize simulation results of the foregoing settings in Tables 2--4, 
each corresponding to a different choice of $M$. 
Following \cite{gao2017sparse}, we measure estimation error by prediction loss defined as $L(V,\widehat{V})=\inf_{O\in \mathcal{O}(r)}\|\Sigma_x^{1/2}(\widehat{V} O-V)\|_\mathrm{F}^2$
 and 
$L(W,\widehat{W})=\inf_{O\in \mathcal{O}(r)}\|\Sigma_y^{1/2}(\widehat{W} O-W)\|_\mathrm{F}^2$. 
Each reported number is the median error out of $50$ independent repetitions. 

\begin{table}[!tb]
\begin{tabular}{c|cc|cc|cc|cc}
\hline
 {($n$,$p_1$,$p_2$)} &  {$V$-init} &  {$W$-init} &  {$V$-CoLaR} &  {$W$-CoLaR} &  {$V$-GFP} &  {$W$-GFP} &  {$V$-TGD} &  {$W$-TGD} \\ \hline
 {(300,300,200)}      & 0.2885              & 0.1706              & 0.0511           & 0.0601           & 0.1391              & 0.2636              & 0.0118         & 0.0118         \\ 
 {(600,600,200)}      & 0.3236              & 0.2004              & 0.0638           & 0.0764           & 0.3420              & 0.1655              & 0.0237         & 0.0485         \\ 
 {(300,300,500)}      & 0.1202              & 0.0664              & 0.0135           & 0.0166           & 0.2646              & 0.3078              & 0.0100         & 0.0076         \\ 
 {(600,600,500)}      & 0.1408              & 0.0811              & 0.0176           & 0.0209           & 0.2459              & 0.1876              &         0.0213       &     0.0157           \\ \hline
\end{tabular}
\caption{Comparison of CoLaR and Algorithm \ref{TGD}: median errors out of 50 repetitions with Toeplitz covariance structure. }
\label{simcca1}
\end{table}
\begin{table}[!tb]
\begin{tabular}{c|cc|cc|cc|cc}
\hline
  {($n$,$p_1$,$p_2$)} &  {$V$-init} &  {$W$-init} &  {$V$-CoLaR} &  {$W$-CoLaR} &  {$V$-GFP} &  {$W$-GFP} &  {$V$-TGD} &  {$W$-TGD} \\ \hline
 {(300,300,200)}      & 0.2653              & 0.1712              & 0.0498           & 0.0646           & 0.3413              & 0.5688              & 0.0454         & 0.0211         \\ 
 {(600,600,200)}      & 0.3167              & 0.2087              & 0.0671           & 0.0776           & 0.5016              & 0.7134              & 0.0231         & 0.0559         \\ 
 {(300,300,500)}      & 0.1207              & 0.0665              & 0.0135           & 0.0159           & 0.3580              & 0.1713              & 0.0124         & 0.0117         \\ 
 {(600,600,500)}      & 0.1448              & 0.0817              & 0.0166           & 0.0203           & 0.3906              & 0.4095              & 0.0082         & 0.0129         \\ \hline
\end{tabular}
\caption{Comparison of CoLaR and Algorithm \ref{TGD}: median errors out of 50 repetitions with Identity covariance structure. 
}
\label{simcca2}
\end{table}
\begin{table}[!tb]
\begin{tabular}{c|cc|cc|cc|cc}
\hline
 {($n$,$p_1$,$p_2$)} &  {$V$-init} &  {$W$-init} &  {$V$-CoLaR} &  {$W$-CoLaR} &  {$V$-GFP} &  {$W$-GFP} &  {$V$-TGD} &  {$W$-TGD} \\ \hline
 {(300,300,200)}      & 0.5552              & 0.5718              & 0.1568           & 0.1194           & 0.2021        & 0.4924        & 0.0719         & 0.0421         \\ 
 {(600,600,200)}      & 0.5596              & 0.6133              & 0.2123           & 0.1572           & 0.4653        & 0.2222        & 0.1555         & 0.0686         \\ 
 {(300,300,500)}      & 0.2695              & 0.1917              & 0.0242           & 0.0219           & 0.1592        & 0.3058        & 0.0042         & 0.0056         \\ 
 {(600,600,500)}      & 0.3068              & 0.2368              & 0.0338           & 0.0271           & 0.1853        & 0.1404        & 0.0454         & 0.0491         \\ \hline
\end{tabular}
\caption{Comparison of CoLaR and Algorithm \ref{TGD}: median errors out of 50 repetitions with SparseInv covariance structure. 
}
\label{simcca3}
\end{table}

\par In Tables \ref{simcca1}--\ref{simcca3}, the first four columns collect results of CoLaR. 
Columns ``$V$-GFP'' and ``$W$-GFP'' collect errors of initial estimators by generalized Fantope projection, while 
columns ``$V$-TGD'' and ``$W$-TGD'' report errors of final estimators by Algorithm \ref{TGD}. 
In all settings, both CoLaR and Algorithm \ref{TGD} yield good final estimators, which significantly improve over their respective initializers. 
In addition, Algorithm \ref{TGD} outperforms CoLaR in most settings. 

In summary, both CoLaR and Algorithm \ref{TGD} consistently estimate the leading $r$ canonical loading vectors, while Algorithm \ref{TGD} has a slight advantage. 
Moreover,
Algorithm \ref{TGD} is more desirable due to its generality beyond the sparse CCA setting. 

\subsection{Model mis-specification}
\label{sec:misspec}
To investigate the robustness of our method, we consider one possible mis-specification of the model in the case of two datasets: 
There are 3 pairs of nontrivial canonical correlations present but we only set $r=2$.  
As we are in the CCA setting, we consider three types of covariance matrices specified in the previous subsection. 
For all three cases, the first two canonical loading vectors and generalized eigenvalues are generated in the same way as the previous subsection.
In addition, we also set the support of the third vector to be $\{1,6,11,16,21\}$. 
We fix $\theta_3=0.3$ (i.e., $\lambda_3 = 1.3$)
and focus on the configuration $(n,p_1,p_2)=(300,300,500)$. 
All tuning parameters are set in the same way as in last subsection. 

\begin{table}[!tb]
\centering
\begin{tabular}{c|cc|cc}
\hline
{{{Covariance structure}}} &  {$V$-CoLaR} &  {$W$-CoLaR} &  {$V$-TGD} &  {$W$-TGD} \\ \hline
{Toeplitz}                  & 0.0197           & 0.0197           & 0.0291         & 0.0157         \\ 
{Identity}                  & 0.0190           & 0.0195           &        0.0113        &    0.0155            \\ 
{SparseInv}                 & 0.0348           & 0.0263           & 0.0112         & 0.0488         \\ \hline
\end{tabular}
\caption{Comparison of CoLaR and Algorithm \ref{TGD}: median errors of estimating first two generalized eigenvectors out of 100 repetitions in mis-specified models.}
\label{simcca4}
\end{table}

Simulation results under these settings can be found in Table \ref{simcca4}, and we continue to use prediction loss as in the previous subsection.
Similar to CoLaR, Algorithm \ref{TGD} continues to produce accurate estimates when the input latent dimension $r_{\mathrm{in}}$ is smaller than the true value $r$ and there is an eigengap between {$\theta_{r_{\mathrm{in}}}$ and $\theta_{r_{\mathrm{in}}+1}$, and thus between  $\lambda_{r_{\mathrm{in}}}$ and $\lambda_{r_{\mathrm{in}}+1}$.}


\subsection{Sparse PCA of correlation matrices}

We now apply Algorithm \ref{TGD} to performing sparse PCA of correlation matrices.
See Remark \ref{rmk:pca} for a detailed account on how sparse PCA of correlation matrices can be cast as a special case of sparse GCA.
By Remark \ref{rmk:pca}, let $\widehat{\Sigma}_0$ be the diagonal matrix with $p$ sample variances on the diagonal, we could use $\widehat{\Sigma}_0^{1/2}\widehat{A}$ as our estimator of the $r$ leading eigenvectors of correlation matrix $R = \Sigma_0^{-1/2}\Sigma\Sigma_0^{-1/2}$, where $\widehat{A}$ is the estimator produced by Algorithm \ref{TGD}.
{We report the squared matrix distance for estimating the leading eigenspace $L(E_R,\widehat{A}) = \min_{O\in\mathcal{O}(r)}\|\widehat{\Sigma}_0^{1/2}\widehat{A}O-E_R\|_\mathrm{F}^2$ for initialization and final error, where $E_R$ is the leading eigenspace of $R$.}


Given ambient dimension $p$, sparsity $s$, latent dimension $r$ and generalized eigenvalues $\lambda_1\geq \dots\geq \lambda_r > 1$, we generate a $p\times p$ correlation matrix $R$ with leading eigenvalues $\lambda_1\geq \dots\geq \lambda_r > \lambda_{r+1} = 1$ and $s$-sparse leading eigenvectors according to the procedure described in Section \ref{sec:simu-pca-detail}.
To further construct the covariance matrix $\Sigma$, we 
generate $\Sigma_0$ as a diagonal matrix with i.i.d.~uniform random numbers in $[0.1,1]$ as diagonal elements. 
Finally, we set $\Sigma=\Sigma_0^{1/2}R\Sigma_0^{1/2}$.

In our simulation study, we set $p=500, s=20$, $r=3$, and let sample size $n$ be either $500$ or $2000$. 
The tuning parameter for initialization is set as $\rho=\frac{1}{2}\sqrt{\frac{\log p}{n}}$. 
Tuning parameters of Algorithm \ref{TGD} are set to be
$s'=40$, $T=20000$, $\lambda=0.01$, and $\eta=0.001$.
The leading generalized eigenvalues are set as $\{5,5,5\}$ (case I) or $\{7,5,3\}$ (case II).
For both cases, our construction of $\Sigma$ ensures that $\lambda_4 = 1$.

\begin{table}[h]
\centering
\begin{tabular}{c|c|c|c}
\hline
$n$ & Case & Initial Error  & Final Error  \\ \hline
\multirow{2}{*}{$n=500$} & I & 0.4603 (0.1404)   & 0.0851 (0.0306)   \\ \cline{2-4} 
 & II & 0.6379 (0.0567)   & 0.1001 (0.0098)   \\ \hline
\multirow{2}{*}{$n=2000$} & I & 0.0882 (0.0177)   & 0.0184 (0.0025)   \\ \cline{2-4} 
 & II & 0.1060 (0.0697) &   0.0258 (0.0025)   \\ \hline
\end{tabular}
\caption{Sparse PCA of correlation matrices by Algorithm \ref{TGD}: 
median errors of initial and final estimators out of 50 repetitions.
MADs are reported in parentheses.
}
\label{spca_sim}
\end{table}

We report results from 50 repetitions for each case and each sample size in Table \ref{spca_sim}.
From Table \ref{spca_sim}, it can be seen that in all settings, Algorithm \ref{TGD} yields accurate estimators of leading eigenspaces of correlation matrices. 
Due to a larger eigengap in the first case, it has slightly better estimation accuracy which corroborates our theory.

\subsection{Sparse GCA with general covariance structure}
\label{sec:exp_general_covariance}
In this section, we investigate the performance of our algorithm under a general covariance structure and demonstrate empirically that our algorithm works beyond latent variable models as long as
the eigengap condition holds.

We consider sparse GCA of three high-dimensional datasets with $p_1=500$, $p_2=p_3=200$ and let the rank of all off-diagonal blocks of $\Sigma$ be $p_{\min} = \min_i p_i = 200$. 
To generate covariance matrices $\Sigma$ and $\Sigma_0$,
we first define a diagonal matrix $\Theta$ with $\Theta_{ii}=2/i$ for $i=1, 2, \dots, p_{\min}$.
We then define $\Sigma$ as a block matrix with
\begin{equation*}
\Sigma_{\dc{ij}}=T_{\dc{i}}U_{\dc{i}}\Theta U_{\dc{j}}^\top T_{\dc{j}},\quad \Sigma_{\dc{ii}}=2T_{\dc{i}},\quad\text{for $i\neq j \in\{1,2,3\}$},
\end{equation*}\begin{equation*}
    U_{\dc{1}}^\top T_{\dc{1}}U_{\dc{1}}=U_{\dc{2}}^\top T_{\dc{2}}U_{\dc{2}}=U_{\dc{3}}^\top T_{\dc{3}}U_{\dc{3}}=I.
\end{equation*} 
Here, each $\Sigma_{\dc{ii}}=2T_{\dc{i}}$ and each $T_{\dc{i}}$ is a Toeplitz matrix, defined by setting $(T_{\dc{k}})_{ij}=\sigma_{k_{ij}}$ where $\sigma_{k_{ij}}=a_k^{|i-j|}$ for all $i,j\in [p_k]$ with $a_1=0.5, a_2=0.7, a_3=0.9$. 
The first five columns of each $U_{\dc{i}}$ have row sparsity level $s_i = 5$. 
As a result, the support sizes of the first three GCA loading vectors are at most $s = 15$.
For each $U_{\dc{i}}\in\mathbb{R}^{p_i\times p_{\min}}$, we define $U_{\dc{i}} = [U_{\dc{i}}(1), U_{\dc{i}}(2)]$ where $U_{\dc{i}}(1)\in\mathbb{R}^{p_i\times s_i}$ and $U_{\dc{i}}(2)\in\mathbb{R}^{p_i\times (p_{\min}-s_i)}$. 
The submatrices $U_{\dc{i}}(1), U_{\dc{i}}(2)$ are generated as follows.
To generate each $U_{\dc{i}}(1)$, we first randomly select a support of size $s_i$.
For each row in the support of $U_{\dc{i}}(1)$, we generate its entries as i.i.d.~standard normal random variables and fill the remaining entries with zeros. 
 Then, we normalize all $U_{\dc{i}}(1)$'s with respect to $T_{\dc{i}}$.
To generate $U_{\dc{i}}(2)$, we first perform SVD on $\left(U_{\dc{i}}(1)\right)^\top T_{\dc{i}}^{1/2}=P_{\dc{i}}D_{\dc{i}}Q_{\dc{i}}^\top$ to obtain 
$Q_{\dc{i}}\in \mathcal{O}(p_i)$. 
We then take $\widetilde{Q}_{\dc{i}} = Q_{\dc{i}, *J}$ for $J = \{s_i+1, \dots, p_{\min}\}$ which is the submatrix of ${Q}_{\dc{i}}$ consisting of the $(s_i+1)$th to the $p_{\min}$th columns of $Q_{\dc{i}}$. 
By construction, we have $\left(U_{\dc{i}}(1)\right)^\top T_{\dc{i}}^{1/2}\widetilde{Q}_{\dc{i}}=0$ and $\widetilde{Q}_{\dc{i}}\in\mathbb{R}^{p_i\times (p_{\min}-s_i)}$. 
Finally, we define $U_{\dc{i}}(2) = T_{\dc{i}}^{-1/2}\widetilde{Q}_{\dc{i}}$.
By construction, we have $\left(U_{\dc{i}}(2)\right)^\top T_{\dc{i}}U_{\dc{i}}(2) = I$ and $\left(U_{\dc{i}}(2)\right)^\top T_{\dc{i}}U_{\dc{i}}(1) = 0$. 
Thus, the leading five columns of each $U_{\dc{i}}$ are sparse and we have $U_{\dc{i}}^\top T_{\dc{i}}U_{\dc{i}} = I$.
With the foregoing construction, it is straightforward to verify that the $i$th generalized eigenvalue is given by $\lambda_i = 1 +2i^{-1}$ for $i =1, 2, \dots, p_{\min}$ and the non-trivial GCA loadings are given by $[U_{\dc{1}}^\top, U_{\dc{2}}^\top, U_{\dc{3}}^\top]^\top$. Therefore, there are eigengaps between any consecutive generalized eigenvalues above one and only the first five GCA loading vectors are sparse. 

Under the foregoing setup, we aim at estimating the leading GCA loading vector (Case I), the matrix with leading two GCA loading vectors (Case II), and the matrix with leading three GCA loading vectors (Case III). 
For tuning parameters, we set  $s'=20$ for sparsity level, 
 $\eta=0.001$, $\lambda = 0.01$, and $T=15000$ in Algorithm \ref{TGD}, and
$\rho = \frac{1}{2}\sqrt{\frac{\log p}{n}}$ in \eqref{eq:initial-fantope}. 
We consider two different sample sizes: $n=500$ and $n=2000$.

Table \ref{table:sgca_general_covariance} reports the squared matrix distances for the initial and the final estimators from true value defined as $\mathrm{dist}^2(A,\widehat{A}) = \min_{O\in\mathcal{O}(r)}\|\widehat{A}O-A\|_\mathrm{F}^2$ for all three cases and two different sample sizes.
Numbers in each cell are the median and the MAD of results over $50$ repetitions. Population parameters are also generated independently in each repetition.
To gauge the scale of errors, we also report squared Frobenius norm $\|A\|_\mathrm{F}^2$ of the estimand in the last column of Table \ref{table:sgca_general_covariance}.

\begin{table}[h!]
\centering
\begin{tabular}{c|c|c|c|c}
\hline
$n$ & Case & Initial Error  & Final Error & $\|A\|_\mathrm{F}^2$ \\ \hline
\multirow{3}{*}{$n=500$} & I & 0.0500 (0.0358)   & 0.0005 (0.0007) & 0.5669 
\\ \cline{2-5} 
 & II & 0.1586 (0.0420) &   0.0605 (0.0313)&  1.2433 
 \\ \cline{2-5} 
  & III &  0.3097  (0.1484) &    0.1831 (0.1107) & 1.8378 
  \\ \hline
\multirow{3}{*}{$n=2000$} & I &  0.0174  (0.0064)   & 0.0001 (0.0002)&  0.5669 
\\ \cline{2-5} 
 & II & 0.0390 (0.0268)   & 0.0136 (0.0063)&  1.2433 
 \\ \cline{2-5} 
   & III & 0.0741 (0.0719) &   0.0369 (0.0178)&  1.8378 
   \\ \hline
\end{tabular}
\caption{ Sparse GCA with general covariance structure: median errors of initial and final estimators out of $50$ repetitions with MADs reported in parentheses. Medians of squared Frobenius norms of estimands are listed in the last column.}

\label{table:sgca_general_covariance}
\end{table}
From Table \ref{table:sgca_general_covariance}, it can be seen that in all settings, Algorithm \ref{TGD} consistently yields improved estimates of leading GCA loading vectors whenever there is an eigengap. 

\appendix
\section{Proof of Lemmas in Section \ref{sec:latent}}

\subsection{Proof of Lemma \ref{lemma:latent_var}}

\begin{proof}
For $\Lambda_r = \mathrm{diag}(\lambda_1,\dots, \lambda_r)$. 
Since $A$ solves \eqref{eq:pop-gca-mat2}, we have 
\begin{equation}
      \label{eq:right-eig}
\Sigma_0^{-1}\Sigma A = A \Lambda_r.      
\end{equation}
By model \eqref{eq:latent-var-model}--\eqref{eq:Sigma-i-2}, we have
$\Sigma = \Psi + UU^\top$ and $\Sigma_0 = \mathrm{diag}(\Sigma_{\dc{11}},\dots, \Sigma_{\dc{kk}})$, and so
\begin{equation*}
\Sigma_0^{-1}\Sigma = 
\begin{bmatrix}
I_{p_1} & \Sigma^{-1}_{\dc{11}} U_{\dc{1}} U^\top_{\dc{2}} & \cdots & \Sigma^{-1}_{\dc{11}}U_{\dc{1}} U^\top_{\dc{k}}\\
\Sigma^{-1}_{\dc{22}} U_{\dc{2}} U^\top_{\dc{1}} & I_{p_2} & \cdots & \Sigma^{-1}_{\dc{22}} U_{\dc{2}} U^\top_{\dc{k}}\\
\vdots & \vdots & \ddots & \vdots \\
\Sigma^{-1}_{\dc{kk}} U_{\dc{k}} U^\top_{\dc{1}} & \Sigma^{-1}_{\dc{kk}} U_{\dc{k}} U^\top_{\dc{2}} & \cdots & I_{p_k}
\end{bmatrix}.
\end{equation*}
Thus, \eqref{eq:right-eig} leads to
\begin{equation*}
\Sigma_0^{-1}\Sigma A =
\begin{bmatrix}
A_{\dc{1}} + \Sigma^{-1}_{\dc{11}}U_{\dc{1}} \sum_{j\neq 1} U^\top_{\dc{j}}A_{\dc{j}}\\
\vdots \\
A_{\dc{k}} + \Sigma^{-1}_{\dc{kk}}U_{\dc{k}} \sum_{j\neq k} U^\top_{\dc{j}} A_{\dc{j}}\\
\end{bmatrix} 
=
\begin{bmatrix}
A_{\dc{1}} \Lambda_r \\ \vdots \\ A_{\dc{k}} \Lambda_r
\end{bmatrix} = A \Lambda_r.
\end{equation*}
Collecting terms, we obtain
\begin{equation*}
\begin{bmatrix}
\Sigma^{-1}_{\dc{11}}U_{\dc{1}} \sum_{j\neq 1} U^\top_{\dc{j}} A_{\dc{j}}\\
\vdots \\
\Sigma^{-1}_{\dc{kk}}U_{\dc{k}} \sum_{j\neq k} U^\top_{\dc{j}} A_{\dc{j}}\\
\end{bmatrix} 
=
\begin{bmatrix}
A_{\dc{1}} (\Lambda_r - I_r) \\ \vdots \\ A_{\dc{k}} (\Lambda_r-I_r)
\end{bmatrix} = A (\Lambda_r - I_r).
\end{equation*}
Since both $A$ and $\Lambda_r - I_r$ are of rank $r$ under the assumption of the lemma, we complete the proof.
\end{proof}

\subsection{Proof of Lemma \ref{lem:eigengap}}
Before the proof of the lemma, we recall Weyl's inequality \citep{weyl1912asymptotische} stated as follows.

\begin{lemma}(Weyl's inequality)
      \label{lemma:weyl}
Let $A,B$ be two $p\times p$ Hermitian matrices {and $s_i(A)$ be the $i$th eigenvalue of $A$}. 
Then for $1\leq l,j\leq p$, we have 
\begin{align*}
\eigen_l(A+B)& \leq \eigen_j(A)+\eigen_k(B),
\quad 
\text{for $l\geq j+k-1$},\\
\eigen_j(A)+\eigen_l(B) & \leq \eigen_{j+l-p}(A+B),
\quad 
\text{for $j+l\geq p$}. 
\end{align*}
\end{lemma}

\begin{proof}
The proof is adapted from \cite{fan2019estimating}. 
It is composed of three parts.
First we show that if $i\geq r+1$, $\lambda_i\leq 1$. 
Then we will prove that $\lambda_r>1$.
Finally we prove the result on multiplicity of 1. 
Recall that the generalized eigenvalues are identical to the eigenvalues of $R=\Sigma_0^{-1/2}\Sigma\Sigma_0^{-1/2}$. 
From now on we study the eigenvalues of $R$ directly. 

\paragraph{Step (1):} 
Note that we can write $R$ as 
\begin{equation*}
R=\Sigma_0^{-1/2}\Sigma\Sigma_0^{-1/2}=\Sigma_0^{-1/2}(UU^\top +\Psi)\Sigma_0^{-1/2}=Q_1Q_1^\top +Q_2Q_2^\top 
\end{equation*} 
where 
\begin{equation*}
Q_1=\Sigma_0^{-1/2}U\quad Q_2=\Sigma_0^{-1/2}\Psi^{1/2}.
\end{equation*}
We first show that $\|Q_2 Q_2^\top\|_{\op} \leq 1$.
To this end, note that 
$\|Q_2 Q_2^\top\|_{\op} = \|Q_2^\top Q_2\|_{\op} = \|\Psi^{1/2}\Sigma_0^{-1}\Psi^{1/2}\|_{\op}$.
Note that $\Psi^{1/2}\Sigma_0^{-1}\Psi^{1/2}$ is a block diagonal matrix with the $i$th block given by $\Psi^{1/2}_{\dc{ii}}(\Psi_{\dc{ii}} + U_{\dc{i}} U^\top_{\dc{i}})^{-1} \Psi^{1/2}_{\dc{ii}}$.
Thus, it suffices to show that 
\begin{equation*}
\|\Psi^{1/2}_{\dc{ii}}(\Psi_{\dc{ii}} + U_{\dc{i}} U^\top_{\dc{i}})^{-1} \Psi^{1/2}_{\dc{ii}}\|_{\op} \leq 1, \quad
\mbox{for $i\in [k]$}.
\end{equation*}
To this end, by Woodbury matrix identity,
\begin{align*}
&\Psi^{1/2}_{\dc{ii}}(\Psi_{\dc{ii}} + U_{\dc{i}} U^\top_{\dc{i}})^{-1} \Psi^{1/2}_{\dc{ii}}\\
=&    
\Psi^{1/2}_{\dc{ii}}[ \Psi^{-1}_{\dc{ii}} - \Psi^{-1}_{\dc{ii}}U_{\dc{i}} (I_r + U^\top_{\dc{i}}\Psi^{-1}_{\dc{ii}}U_{\dc{i}})^{-1} U^\top_{\dc{i}} \Psi^{-1}_{\dc{ii}} ]\Psi^{1/2}_{\dc{ii}} \\
=& 
I_{p_i} - \Psi^{-1/2}_{\dc{ii}}U_{\dc{i}} (I_r + U^\top_{\dc{i}}\Psi^{-1}_{\dc{ii}}U_{\dc{i}})^{-1} U^\top_{\dc{i}} \Psi^{-1/2}_{\dc{ii}}.
\end{align*}
Therefore, $0 \preceq \Psi^{1/2}_{\dc{ii}}(\Psi_{\dc{ii}} + U_{\dc{i}} U^\top_{\dc{i}})^{-1} \Psi^{1/2}_{\dc{ii}}  \preceq I_{p_i}$, and so $\|Q_2 Q_2^\top\|_{\op} \leq 1$.

By Lemma \ref{lemma:weyl}, we have, for $i\geq j+k-1$,
\begin{equation*}
\lambda_i= \eigen_i(R) \leq \eigen_j(Q_1Q_1^\top )+\eigen_k(Q_2Q_2^\top ).
\end{equation*}
Now we set $j = r+1, k=1$ and this gives\begin{equation*}
\lambda_{i}{ = \eigen_i(R) }\leq \eigen_{r+1}(Q_1Q_1^\top )+\eigen_1(Q_2Q_2^\top )=\eigen_1(Q_2Q_2^\top )\quad \text{for $i\geq r+1$}.
\end{equation*}
The last equality holds since we have $\rank(Q_1Q_1^\top )=r$ and so $\eigen_{r+1}(Q_1Q_1^\top )=0$. 
As a result, \begin{equation*}
\lambda_i{ = \eigen_i(R) }\leq \eigen_1(Q_2Q_2^\top )=\|Q_2Q_2^\top \|_\op
\leq 1.
\end{equation*}
 This finishes the proof of the first step.

\paragraph{Step (2):} 
By Lemma \ref{lemma:weyl} and Assumption \ref{ass:latent}, 
we have 
\begin{align*}
\lambda_r { = \eigen_r(R) }
& =\eigen_r(Q_1Q_1^\top +Q_2Q_2^\top )
\geq  \eigen_r(Q_1 Q_1^\top) + \eigen_p(Q_2 Q_2^\top) \\
& > \eigen_r(Q_1 Q_1^\top)
= \sigma_r^2(\Sigma_0^{-1/2}U)
\ge 1.
\end{align*}
Here, the strict inequality holds since $Q_2Q_2^\top$ is of full rank.
This finishes the proof for the $r$th eigenvalue.

\paragraph{Step (3):}
Now we calculate the multiplicity of $1$. 
To this end, we resort to Theorem 2.2 in \cite{tian2004rank}, which gives
\begin{align*}
\rank(U^\top -U^\top YU^\top )&=\rank(U-UY^\top U)=\rank(U^\dagger  U-U^\dagger_{\dc{1}} U_{\dc{1}}-\dots-U^\dagger_{\dc{k}} U_{\dc{k}})\\
&=\rank(\widetilde{U})+\rank(U)-\sum_{i=1}^k \rank(U_{\dc{i}}).
\end{align*}
Here 
\begin{equation*}
\widetilde{U}=
\begin{bmatrix}
         0 & \Sigma_{\dc{12}} & \Sigma_{\dc{13}}& ... & \Sigma_{\dc{1k}} \\
     \Sigma_{\dc{21}}  & 0 & \Sigma_{\dc{23}} & ... & \Sigma_{\dc{2k}} \\
      \Sigma_{\dc{31}} & \Sigma_{\dc{32}} & 0 & ... & \Sigma_{\dc{3k}} \\
      ... & ... & ... & ... & ... \\
      \Sigma_{\dc{k1}} & \Sigma_{\dc{k2}} & \Sigma_{\dc{k3}} & ... & 0
\end{bmatrix}.
\end{equation*}
Note that we have \begin{equation*}
\Sigma_{\dc{ij}}=U_{\dc{i}}U^\top_{\dc{j}}, \quad \rank(U)=r,
\end{equation*}
by the latent variable model. 
Moreover, it is easy to verify that multiplicity of $1$ as an eigenvalue of $R$ is the same as 
multiplicity of $0$ as an eigenvalue of $\widetilde{U}$, which in turn equals $p-\rank(\widetilde{U})$. 
Thus, multiplicity of $1$ is
\begin{align*}
p-\rank(\widetilde{U})& =p-\rank(U^\top -U^\top YU^\top )+\rank(U)-\sum_{i=1}^k \rank(U_{\dc{i}})\\
&=p-\sum_{i=1}^k \rank(U_{\dc{i}})+r- \rank(U-UY^\top U).
\end{align*}
This finishes the proof of the lemma. 
\end{proof}

\section{Proof of Main Results}
\label{sec:proof-main}

We first provide a lemma about the distance defined in \eqref{eq:mat-dist} which is adapted from Lemmas 5.3 and 5.4 from \cite{tu2015low}.
\begin{lemma}\label{sub dis lemma}
For any $U,V\in \mathbb{R}^{p\times r}$, we have\begin{equation*}
    \mathrm{dist}^2(U,V)\leq \frac{1}{2(\sqrt{2}-1)\sigma_r^2(V)}\|UU^\top  -VV^\top  \|_\mathrm{F}^2.
\end{equation*}
If further that $\mathrm{dist}(U,V)\leq\frac{1}{4}\|V\|_\op$, we have\begin{equation*}
    \|UU^\top  -VV^\top  \|_\mathrm{F}\leq\frac{9}{4}\|U\|_\op\mathrm{dist}(U,V).
\end{equation*}
\end{lemma}

Before diving into the proof, we present a list of nice events  on the intersection of which all desired results hold deterministically. 
Recall that $S$ is the true row support of $A$ with cardinality $s$. 
Let $\cI\subseteq \{1,2,3,...p\}$ be an index set. 
We use $\widehat{A}(\mathcal{I})\in\mathbb{R}^{p\times r}$ to denote the leading $r$ generalized eigenvectors that solve 
\begin{equation}\label{def: restricted eigenspace}
    \max \quad\langle\widehat{\Sigma}, LL^\top   \rangle\quad \text{such that} \quad L^\top  \widehat{\Sigma}_{0}L=I_r,\quad\mathrm{supp}(L)\subseteq \mathcal{I}.
\end{equation}
We also denote the diagonal matrix formed by the first $r$ restricted sample generalized eigenvalues by $\widehat{\Lambda}_r(\mathcal{I})$.

 For some sufficiently large constant $C>0$, define
\begin{equation}
      \label{eq:event-B1}
\begin{aligned}
B_1  = \bigg\{\mathrm{dist}(A, \widehat{A}(\cI))
& \leq C\sqrt{r}\frac{\sqrt{1+\lambda_1^2}\sqrt{{1+\lambda_{r+1}^2}}}{\lambda_r-\lambda_{r+1}}\sqrt{\frac{|\cI|\log p}{n}},\\
& \qquad\qquad \mbox{for all $\cI\supset S$ such that $|\cI|\leq 2s'+s$}
\bigg\}.    
\end{aligned}
\end{equation}
Note that the event includes all index sets containing true support $S$ and the sizes of which are at most $2s'+s$ .

We define event $B_2$ as 
\begin{equation}
      \label{eq:event-B2} 
\begin{aligned}
B_2 = 
\bigg\{
\|\widehat{\Sigma}_{\cI\cI}-\Sigma_{\cI\cI}\|_{\mathrm{op}}\vee
\|\widehat{\Sigma}_{0,\cI\cI}-\Sigma_{0,\cI\cI}\|_{\mathrm{op}}
& \leq C\sqrt{\frac{(2s'+s)\log p}{n}},\\
& \text{for all $\cI\subseteq [p]$ with $|\cI|\leq 2s'+s$}
\bigg \}.   
\end{aligned}
\end{equation} 
Further we define $B_3$ as
\begin{equation}
\label{eq:event-B3}
\begin{aligned}
B_3=\bigg\{ 
& \|\widehat{\Sigma}-\Sigma\|_\infty+\|\Sigma_0A\Lambda_rA^\top\Sigma_0-\so A\Lambda_r A^\top\so\|_\infty \\
& \quad
+\lambda_{r+1}\|\Sigma_0-\widehat{\Sigma}_0\|_\infty+\lambda_{r+1}\|\Sigma_0AA^\top\Sigma_0-\widehat{\Sigma}_0AA^\top\so\|_\infty\leq C\sqrt{\frac{\log p}{n}}
 \bigg\}.
\end{aligned}
\end{equation}
Let 
\begin{equation}
      \label{eq:Atilde}
\tA=A(A^\top  \widehat{\Sigma}_0A)^{-1/2}, \quad
\tlambda=(A^\top  \widehat{\Sigma}_0A)^{1/2}\Lambda_r(A^\top  \widehat{\Sigma}_0A)^{1/2}.
\end{equation} 
The event $B_4$ is defined as
\begin{equation}
\label{eq:event-B4}
B_4 = \left\{  \|\tlambda-\Lambda_r\|_\mathrm{F}+\lambda_{r+1}\|A^\top\so A-I\|_\mathrm{F}+ \|\Sigma_0^{1/2}(\tA-A)\|_\mathrm{F}\leq C\sqrt{\frac{r(s+\log (ep/s))}{n}}\right\}
\end{equation}
for some sufficiently large constant $C>0$. 

The following lemmas guarantee that all these events occur with high probabilities, uniformly over $\mathcal{P}_n$.
In addition, since $s'\ge s$, Lemma 12 in \cite{gao2015minimax} implies
that
$B_2$ happens with probability at least $1-\exp(-C's'\log (ep/s'))$ for some positive constant $C'$, uniformly over $\mathcal{P}_n$.

\begin{lemma}\label{lemma: oracle}
{ Suppose condition \eqref{sample size condition} holds, there exist constants $C, C' > 0$ such that uniformly over $\mathcal{P}_n$,} 
$B_1$ happens with probability at least $1-\exp(-C's'\log(ep/s'))$.
\end{lemma}

\begin{lemma}\label{prob B3}
Suppose $r\sqrt{\log p/n}\leq c'$ for some sufficiently small constant $c'\in(0, 1)$. Then there exist positive constants $C, C'$
such that { uniformly over $\mathcal{P}_n$,} $B_3$ happens with probability at least $1-p^{-C'}$.
\end{lemma} 

\begin{lemma}\label{prob B4}
Suppose $\sqrt{\frac{s+\log(ep/s)}{n}}\leq c'$ for some sufficiently small constant $c'\in(0,1)$. Then { there exist constants $C, C'>0$ such that uniformly over $\mathcal{P}_n$,} $B_4$ happens with probability at least $1-\exp(-C'(s+\log(ep/s)))$. 
\end{lemma}


In the rest of this section, 
we will show that Theorem \ref{mainthe} and Corollary \ref{cor: normalized estimate} hold on event $B_1\cap B_2$. 
Theorem \ref{theo:init} holds on event $B_3\cap B_4$,
and Corollary \ref{cor:init} holds  on event $B_2 \cap B_3\cap B_4$.
As a result, the entire algorithm with generalized Fantope initialization yields an estimator satisfying the upper bound in { Corollary \ref{cor: whole procedure}} 
on $B_1\cap B_2\cap B_3\cap B_4$,
 which holds with probability { at least $1-\exp(-C'(s+\log(ep/s)))$ for some constant $C'>0$  uniformly over $\mathcal{P}_n$,} by the foregoing lemmas and the union bound. This argument proves Corollary  \ref{cor: whole procedure} about the whole procedure.

{

Before proving the main Theorem, we first specifically study effect of the gradient descent step and the hard thresholding step.


The following proposition characterizes the progress in the gradient step.
\begin{proposition}\label{prop: gradana}
Let $S_t=\mathrm{supp}(\gd_t)\bigcup \mathrm{supp}(\gd_{t+1})\bigcup \mathrm{supp}(A)$ be a super set of the row support of $\gd_t$. Define \begin{equation*}
    \widehat{V}(S_t)=\widehat{A}(S_t)\left(I+\frac{\widehat{\Lambda}_r(S_t)}{\lambda}\right)^{1/2}.
\end{equation*} 
Set \begin{equation}\label{tune}
    \eta \leq \frac{c}{12\lambda_1\rev{\nu}(c+2)^2}\quad \text{and}\quad \lambda = \frac{\lambda_1}{c},
\end{equation}
for some constant $c<1$.
On event $B_2$, suppose $\gd_t$ satisfies \begin{equation}\label{radius}
    \mathrm{dist}( \gd_t,\widehat{V}(S_t))\leq \frac{1}{4\sqrt{\rev{\nu}}}\min\left\{\frac{c(\lambda_r -\lambda_{r+1})}{\sqrt{2}\lambda_1\rev{\nu}^2(42+25c)},\frac{\sqrt{1+c}}{2}\right\},
\end{equation}
then after the gradient step, we have
 \begin{equation*}
    \mathrm{dist}^2(V_{t+1}^o, \hU(S_t))\leq\left(1-\frac{\eta(\lambda_r-\lambda_{r+1})}{4\rev{\nu}}\right)\mathrm{dist}^2(\gd_t,\hU(S_t)),
\end{equation*}
here we use $V_{t+1}^o \in\mathbb{R}^{p\times r}$ to denote a matrix that has the same entries as those in $V_{t+1}$ on $S_t\times [r]$ and zeros elsewhere. 

\end{proposition}
The following proposition characterizes the effect of hard thresholding.
\begin{proposition}\label{prop: ht}
Define $V$ as in \eqref{eq:V}. Then if we perform hard thresholding by selecting the top $s'$ elements of $V_{t+1}$, we have \begin{equation}
    \mathrm{dist}^2(\scale,\gd_{t+1})\leq \left(1+2\sqrt{\frac{s}{\hs}} \left(1+\sqrt{\frac{s}{\hs}} \right) \right)\mathrm{dist}^2(\scale, V_{t+1}).
\end{equation}
\end{proposition}
}

\subsection{Proof of Theorem \ref{mainthe}}

We first prove the theorem assuming that Propositions \ref{prop: gradana} and \ref{prop: ht} hold. 
Proofs of the two propositions will be given later in Sections \ref{sec:proof-gradana} and \ref{sec:proof-ht}.
It is worth noting that the results of both propositions hold deterministically on event $B_2$.


\par The first step is to define the effective support in each step. We first define $F_t=\mathrm{supp}(\gd_t)$, { recall that $S$ is the true row support of $A$}, then we define the effective restricted set to be 
\begin{equation*}
    S_t=F_t\cup F_{t+1}\cup { S}.
\end{equation*}
The gradient descent step restricted to $S_t$ can be viewed as 
\begin{equation*}
    V_{t+1,S_t*}=\gd_{t,S_t*}-2\eta(-\widehat{\Sigma}_{S_tS_t}\gd_{t,S_t*}+\lambda \widehat{\Sigma}_{0,S_tS_t} \gd_{t,S_t*}(\gd_{t,S_t*}^\top\widehat{\Sigma}_{0,S_tS_t}\gd_{t,S_t*}-I_r)).
\end{equation*}
Note that applying hard thresholding on $V_{t+1}^o$ (as defined in Proposition \ref{prop: gradana}) is equivalent to applying hard thresholding on the original $V_{t+1}$. 
This allows us to replace the intermediate update by $V_{t+1}^o$ and still obtain the same output sequence $\gd_{t+1}$.
Thus, we will prove instead for the update using $\widehat{\Sigma}_{0,S_tS_t}, \widehat{\Sigma}_{S_tS_t}$. 


\begin{proof}
Throughout the proof, we assume the event $B_1\cap B_2$ happens, 
which occurs with probability at least $1-\exp(-C's'\log (ep/s'))$ for some constant $C'>0$, {  uniformly over $\mathcal{P}_n$}. 
As we have mentioned, the conclusions of Propositions \ref{prop: gradana} and \ref{prop: ht} hold on this event,
and the remaining arguments in this proof proceed in a deterministic fashion.

We argue by induction on the iteration counter $t$. 
Specifically, for $t=1,2,...$, we will prove that $\gd_t$ satisfies condition \eqref{radius} and that \begin{equation}\label{main induction}
\mathrm{dist}(\gd_{t},\scale)
 \leq\xi^{t-1}\,\mathrm{dist}(\gd_1,\scale)
+\frac{C_1}{1-\xi}
\frac{\sqrt{1+\lambda_1^2}\sqrt{{1+\lambda_{r+1}^2}}}{\lambda_r-\lambda_{r+1}}\sqrt{\frac{rs'\log p}{n}}
\end{equation}
for some positive constant $C_1$.

\par\textit{Base Case:}
By \eqref{sample size condition}, \eqref{eq:radius}, 
and definition of $B_1$ in \eqref{eq:event-B1},
condition \eqref{radius} is satisfied when $t = 1$. 
Moreover, when $t=1$ equation \eqref{main induction} holds trivially.

\smallskip

\par\textit{Induction Step:}
Suppose that $\gd_{t}$ satisfies the radius condition \eqref{radius} and that the induction hypothesis \eqref{main induction} is satisfied at step $t$. 
We are to show that \eqref{radius} and \eqref{main induction} hold for $\gd_{t+1}$.

In the gradient step, Proposition \ref{prop: gradana} shows that under radius condition \eqref{radius} on $\dist(\gd_t, \hU(S_t))$,
if we choose the step-size to be $\eta\leq\frac{1}{\beta}$, 
\begin{equation*}
\mathrm{dist}(V_{t+1}^o, \hU(S_t))\leq\sqrt{1-\alpha\eta}\,\mathrm{dist}(\gd_t,\hU(S_t))
\leq \left(1-\frac{\alpha\eta}{2} \right)\mathrm{dist}(\gd_t,\hU(S_t)),
\end{equation*}
where \begin{equation*}
   \alpha=\frac{(\lambda_r-\lambda_{r+1})}{4\rev{\nu}},\quad \beta=12\lambda\rev{\nu}\left(2+\frac{\lambda_1}{\lambda}\right)^2.
\end{equation*}
Recall that 
\begin{equation*}
    \scale=A\left(I+\frac{1}{\lambda}\Lambda_r\right)^{\frac{1}{2}},\quad \widehat{V}(S_t)=\ha(S_t)\left(I+\frac{1}{\lambda}\hl(S_t)\right)^{\frac{1}{2}}.
\end{equation*}
By \eqref{eq:event-B1}, on event $B_1\cap B_2$, since $\lambda = {\lambda_1}/{c}$, we have
\begin{equation*}
\mathrm{dist}(\scale, \widehat{V}(S_t))\leq C_0\sqrt{r}\frac{\sqrt{1+\lambda_1^2}\sqrt{{1+\lambda_{r+1}^2}}}{\lambda_r-\lambda_{r+1}}\sqrt{\frac{|S_t|\log p}{n}},
\end{equation*}
for some positive constant $C_0$.
Triangle inequality then leads to
\begin{equation*}
\mathrm{dist}(V_{t+1}^o, \scale)\leq\left(1-\frac{\alpha\eta}{2} \right)\mathrm{dist}(\gd_t,\scale)+2C_0\sqrt{r}\frac{\sqrt{1+\lambda_1^2}\sqrt{{1+\lambda_{r+1}^2}}}{\lambda_r-\lambda_{r+1}}\sqrt{\frac{|S_t|\log p}{n}}.
\end{equation*}
Turn to the thresholding step.
By Proposition \ref{prop: ht}, we have
\begin{align*}
\mathrm{dist}(\gd_{t+1},\scale)
& \leq  
\sqrt{1+2\sqrt{\frac{s}{\hs}}\left(1+\sqrt{\frac{s}{\hs}}\right)} \mathrm{dist}(V_{t+1}^o,\scale)\\
& \leq\left(1+\sqrt{\frac{s}{\hs}}\left(1+\sqrt{\frac{s}{\hs}}\right)\right)\mathrm{dist}(V_{t+1}^o,\scale). 
\end{align*}

Combining the last two displays, we obtain that \begin{equation*}
    \mathrm{dist}(\gd_{t+1},\scale)\leq \left(1+\sqrt{\frac{s}{\hs}}\left(1+\sqrt{\frac{s}{\hs}}\right)\right)\left[\left(1-\frac{\alpha\eta}{2} \right)\mathrm{dist}(\gd_t,\scale)+2C_0\sqrt{r}\frac{\sqrt{1+\lambda_1^2}\sqrt{{1+\lambda_{r+1}^2}}}{\lambda_r-\lambda_{r+1}}\sqrt{\frac{|S_t|\log p}{n}}\right].
\end{equation*}
Define \begin{equation*}
    \xi= \left(1+\sqrt{\frac{s}{\hs}}\left(1+\sqrt{\frac{s}{\hs}}\right)\right)\left(1-\frac{\alpha\eta}{2} \right)<1.
\end{equation*}
Note that we can always ensure $\xi< 1$ by enlarging $\hs$ appropriately.
Specifically, the choice in the theorem, $s' \geq \frac{16}{\alpha^2\eta^2}s$, ensures that 
\begin{equation}
\label{eq:condition_s_prime}
      \xi \leq\left(1+2\sqrt{\frac{s}{\hs}}\right)\left(1-\frac{\alpha\eta}{2} \right) \leq \left(1+\frac{\alpha\eta}{2}\right) \left(1-\frac{\alpha\eta}{2}\right)= 1-\frac{\alpha^2\eta^2}{4}.
\end{equation} 
Since $|S_t| \leq 2s'+s$ and $s'\ge s$, we further obtain that 
\begin{equation*}
\mathrm{dist}(\gd_{t+1},\scale)
\leq\xi\,\mathrm{dist}(\gd_t,\scale)+C_1\frac{\sqrt{1+\lambda_1^2}\sqrt{{1+\lambda_{r+1}^2}}}{\lambda_r-\lambda_{r+1}}\sqrt{\frac{rs'\log p}{n}}\,,
\end{equation*}
for some constant $C_1 > 0$. 
Note that in the above equation $C_1$ does not depend on $t$. 
We now show that it is identical with the constant $C_1$ in \eqref{main induction}. 
To this end, note that 
\begin{align*}
\mathrm{dist}(\gd_{t+1},\scale)
&\leq\xi\,\mathrm{dist}(\gd_t,\scale)+C_1\frac{\sqrt{1+\lambda_1^2}\sqrt{{1+\lambda_{r+1}^2}}}{\lambda_r-\lambda_{r+1}}\sqrt{\frac{rs'\log p}{n}}\\
&\leq \xi \left(\xi^{t-1}\,\mathrm{dist}(\gd_1,\scale)
+\frac{C_1}{1-\xi}
\frac{\sqrt{1+\lambda_1^2}\sqrt{{1+\lambda_{r+1}^2}}}{\lambda_r-\lambda_{r+1}}\sqrt{\frac{rs'\log p}{n}} \right)\\
& \qquad +C_1\frac{\sqrt{1+\lambda_1^2}\sqrt{{1+\lambda_{r+1}^2}}}{\lambda_r-\lambda_{r+1}}\sqrt{\frac{rs'\log p}{n}}\\
&\leq \xi^{t}\,\mathrm{dist}(\gd_1,\scale)
+\frac{C_1}{1-\xi}
\frac{\sqrt{1+\lambda_1^2}\sqrt{{1+\lambda_{r+1}^2}}}{\lambda_r-\lambda_{r+1}}\sqrt{\frac{rs'\log p}{n}}.
\end{align*}
Here the second inequality is due to the induction hypothesis at the $t$th iteration.
Moreover, it can be ensured that $\gd_{t+1}$ also satisfies \eqref{radius} since $\xi<1$ and the extra term is bounded by a sufficiently small constant due to \eqref{sample size condition}. 
As a result, we have shown that both \eqref{radius} and \eqref{main induction} are satisfied for all $t\geq 1$.

In summary, on $B_1\cap B_2$, for any $t\geq 1$,
\begin{align*}
\mathrm{dist}(\gd_{t+1},\scale)
& \leq\xi^t\,\mathrm{dist}(\gd_1,\scale)
+\frac{C_1}{1-\xi}
\frac{\sqrt{1+\lambda_1^2}\sqrt{{1+\lambda_{r+1}^2}}}{\lambda_r-\lambda_{r+1}}\sqrt{\frac{rs'\log p}{n}}\\
& \leq\xi^t\,\mathrm{dist}(\gd_1,\scale)
+\frac{4C_1}{\alpha^2\eta^2}
\frac{\sqrt{1+\lambda_1^2}\sqrt{{1+\lambda_{r+1}^2}}}{\lambda_r-\lambda_{r+1}}\sqrt{\frac{rs'\log p}{n}}\\
& \leq \xi^t \, \mathrm{dist}(\gd_1,\scale)+C\left(\frac{s'}{s}\right)^{3/2}\frac{\sqrt{1+\lambda_1^2}\sqrt{{1+\lambda_{r+1}^2}}}{\lambda_r-\lambda_{r+1}}\sqrt{\frac{rs\log p}{n}}\numberthis \label{eq:condition_s_prime_2}\,.
\end{align*}
Here, the second inequality holds since $\frac{1}{1-\xi}\leq \frac{4}{\alpha^2\eta^2}$ and the last inequality is due to $s' \geq \frac{16}{\alpha^2\eta^2}s$.
This completes the proof.
\end{proof}

\subsection{Proof of Corollary \ref{cor: normalized estimate}}
\begin{proof}
We prove the desired result on $B_1\cap B_2$, which happens with probability at least $1-\exp(-C'(s'\log(ep/s')))$ for some constant $C'>0$,  uniformly over $\mathcal{P}_n$.  
For notational convenience,
denote the statistical error rate in \eqref{stats error and opt error} by $\epsilon_n$, that is
\begin{equation*}
\epsilon_n = \left(\frac{s'}{s}\right)^{3/2}\frac{\sqrt{1+\lambda_1^2}\sqrt{{1+\lambda_{r+1}^2}}}{\lambda_r-\lambda_{r+1}}\sqrt{\frac{rs\log p}{n}}.
\end{equation*} 

Under condition \eqref{iteration}, the statistical error dominates optimization error in \eqref{stats error and opt error}. 
Hence by Theorem \ref{mainthe} $\dist(\gd_t,V)$ is bounded by a constant multiple of $\epsilon_n$, that is, 
\begin{equation*}
\dist\left(\gd_t,A\left(I+\frac{\Lambda_r}{\lambda}\right)^{1/2}\right)\leq C_0\epsilon_n 
\end{equation*}
for some constant $C_0 > 0$. 
Let $P$ be the orthogonal matrix that minimizes the distance, then we can write
\begin{equation}
      \label{eq:Q}
\gd_t=A\left(I+\frac{\Lambda_r}{\lambda}\right)^{1/2}P+Q\,,
\end{equation}
where $\|Q\|_\mathrm{F}\leq C_0\epsilon_n.$

Since $\gd_t$ is $s'$ row sparse and $A$ is $s$ row sparse, 
$Q$ is $s+s'$ row sparse. 
The remaining proof is composed of two steps: 
(1) bounding $\|(\gd_t^\top\so\gd_t)^{-1/2}-P^\top(I+\frac{\Lambda_r}{\lambda})^{-1/2}P\|_\mathrm{F}$ and 
(2) bounding $\|\gd_t(\gd_t^\top\so\gd_t)^{-1/2}-AP\|_\mathrm{F}$.
The bound in step (2) then gives the desired bound in the statement of the corollary.
In the rest of the proof, let $\Delta=(\gd_t^\top\so\gd_t)^{-1/2}-P^\top(I+\Lambda_r/\lambda)^{-1/2}P$.

\paragraph{Step (1):} 
By definition we have 
 \begin{align*}
\gd_t^\top  \szerothird \gd_t & = 
\left(A\left(I+\frac{\Lambda_r}{\lambda}\right)^{1/2}P+Q\right)^\top \szerothird \left(A\left(I+\frac{\Lambda_r}{\lambda}\right)^{1/2}P+Q\right)\\
& = P^\top  \left(I+\frac{\Lambda_r}{\lambda}\right)^{1/2}A^\top  \szerothird A\left(I+\frac{\Lambda_r}{\lambda}\right)^{1/2}P+Q^\top  \szerothird A\left(I+\frac{\Lambda_r}{\lambda}\right)^{1/2}P \\
&~~~ +P^\top\left(I+\frac{\Lambda_r}{\lambda}\right)^{1/2}A^\top  \szerothird Q+Q^\top  \szerothird Q.
\end{align*}
Since $A^\top \Sigma_0 A = I_r$, we have\begin{equation*}
P^\top  \left(I+\frac{\Lambda_r}{\lambda}\right)^{1/2}A^\top  \Sigma_0 A\left(I+\frac{\Lambda_r}{\lambda}\right)^{1/2}P=P^\top\left(I+\frac{\Lambda_r}{\lambda}\right)P.
\end{equation*}
Then we can bound the Frobenius norm for $\gd_t^\top  \szerothird \gd_t-P^\top(I+\frac{\Lambda_r}{\lambda})P$ as \begin{align*}
\|\gd_t^\top  \szerothird \gd_t-P^\top(I+\frac{\Lambda_r}{\lambda})P\|_\mathrm{F}\leq&\underbrace{\|P^\top  (I+\frac{\Lambda_r}{\lambda})^{1/2}A^\top  (\szerothird-\Sigma_0) A(I+\frac{\Lambda_r}{\lambda})^{1/2}P\|_\mathrm{F}}_\text{{Term I}}\\
&+\underbrace{2\|Q^\top \szerothird A(I+\frac{\Lambda_r}{\lambda})^{1/2}\|_\mathrm{F}}_{\text{Term II}}+\underbrace{\|Q^\top  \szerothird Q\|_\mathrm{F}}_{\text{Term III}}.
\end{align*}
We bound the three terms on the right side of the last display separately. 
Since $A$ is $s$ sparse, we have that on event $B_2$,
\begin{align*}
\|P^\top  (I+\frac{\Lambda_r}{\lambda})^{1/2}A^\top  (\szerothird-\Sigma_0) A(I+\frac{\Lambda_r}{\lambda})^{1/2}P\|_\op&\leq \|I+\frac{\Lambda_r}{\lambda}\|_\op\|A_{S*}^\top(\widehat{\Sigma}_{0, SS}-\Sigma_{0,SS}) A_{S*}\|_\op\\
&\leq C_1\sqrt{\frac{s \log p }{n}}.
\end{align*}
Since $A$ is of rank $r$, we can then bound the Frobenius norm of Term I as \begin{align*}
\|P^\top  (I+\frac{\Lambda_r}{\lambda})^{1/2}A^\top  (\szerothird-\Sigma_0) A(I+\frac{\Lambda_r}{\lambda})^{1/2}P\|_\mathrm{F}
& \leq C_1\sqrt{\frac{rs \log p }{n}} \leq C_1 \epsilon_n.
\end{align*} 
For Term II, since $Q$ is $s+s'$ sparse, 
we notice that on event $B_2$ and under condition \eqref{sample size condition}, we can bound it as \begin{equation*}
\|Q^\top \szerothird A(I+\frac{\Lambda_r}{\lambda})^{1/2}\|_\mathrm{F}\leq C_2 \|Q\|_\mathrm{F}
\end{equation*}
for some constant $C_2>0$. Term III is also dominated by the same upper bound. 
On $B_2$, the operator norms of $\gd_t^\top\so\gd_t$ and $(\gd_t^\top\so\gd_t)^{-1}$ are both upper bounded by a positive constant.
In addition, by \eqref{eq:eta-lambda}, the operator norms of $I+{\Lambda_r}/{\lambda}$ and $(I+{\Lambda_r}/{\lambda})^{-1}$ are also upper bounded by a positive constant.
So we conclude that 
\begin{align*}
\|\Delta\|_\mathrm{F}
&\leq \|(\gd_t^\top\so\gd_t)^{-1/2}\|_\op\|(\gd_t^\top\so\gd_t)^{1/2}-P^\top(I+\frac{\Lambda_r}{\lambda})^{1/2}P\|_\mathrm{F}\|P^\top(I+\frac{\Lambda_r}{\lambda})^{-1/2}P\|_\op\\
&\leq C_3\|\gd_t^\top\so\gd_t-P^\top(I+\frac{\Lambda_r}{\lambda})P\|_\mathrm{F}\leq C_4\epsilon_n.
\end{align*}
where the second to last inequality is due to Lemma 2 in Supplement of \cite{gao2017sparse}. 

\paragraph{Step (2):} Now we bound $\|\festimate-AP\|_\mathrm{F}$ as follows. Recall we have defined that $\Delta=(\gd_t^\top\so\gd_t)^{-1/2}-P^\top(I+\frac{\Lambda_r}{\lambda})^{-1/2}P$. \begin{align*}
\|\festimate-AP\|_\mathrm{F}&=\|\gd_t(\gd_t^\top\so\gd_t)^{-1/2}-AP\|_\mathrm{F}\\
&=\|(A(I+\frac{\Lambda_r}{\lambda})^{1/2}P+Q)(P^\top(I+\frac{\Lambda_r}{\lambda})^{-1/2}P+\Delta)-AP\|_\mathrm{F}\\
&=\|AP+QP^\top(I+\frac{\Lambda_r}{\lambda})^{-1/2}P+A(I+\frac{\Lambda_r}{\lambda})^{1/2}P\Delta+Q\Delta-AP\|_\mathrm{F}\\
&\leq \|QP^\top(I+\frac{\Lambda_r}{\lambda})^{-1/2}P\|_\mathrm{F}+\|A(I+\frac{\Lambda_r}{\lambda})^{1/2}P\Delta\|_\mathrm{F}+\|Q\Delta\|_\mathrm{F}\\
&\leq \|Q\|_\mathrm{F}\|P^\top(I+\frac{\Lambda_r}{\lambda})^{-1/2}P\|_\op+\|A(I+\frac{\Lambda_r}{\lambda})^{1/2}P\|_\op\|\Delta\|_\mathrm{F}+\|Q\|_\mathrm{F}\|\Delta\|_\op\\
&\leq C_5 \epsilon_n
\end{align*}
for some constant $C_5>0$, due to the bounds on $\|Q\|_{\mathrm{F}}$ and $\|\Delta\|_\mathrm{F}$ that we have established in step (1). 

Combining the results above we deduce that 
\begin{equation*}
\dist(\festimate,A)\leq \|\festimate-AP\|_\mathrm{F}\leq C_5\epsilon_n
\end{equation*}
 with probability at least $1-\exp(-C'(s'\log(ep/s')))$ for some positive constants $C_5$ and $C'$, {  uniformly over $\mathcal{P}_n$}.  
 This finishes our proof.
\end{proof}


\subsection{Proof of Proposition \ref{prop: gradana}}
\label{sec:proof-gradana}
\begin{proof}
Throughout the whole proof, we work on event $B_2$ defined in \eqref{eq:event-B2} which happens with probability at least $1-\exp(-C's\log(ep/s'))$ for some constant $C'>0$, uniformly over $\mathcal{P}_n$.
As mentioned before, the gradient step is equivalent to replacing all covariance matrices by $\rescov,\resdiag$ respectively due to sparsity of $\gd_t,\gd_{t+1}$ (since the output sequence $\gd_t$ remains unaltered after this substitution). 
Effectively, at $t$ th step of gradient descent the relavant support is $S_t$ since outside this set the output matrix has row equal to 0.
We denote principal submatrices $\rescov,\resdiag$ by $\widehat{\Sigma}_t\in\mathbb{R}^{|S_t|\times |S_t|},\widehat{\Sigma}_{0,t}\in\mathbb{R}^{|S_t|\times |S_t|}$ for simplicity in the proof and define the restricted Lagrangian function $f_t$ to be \begin{equation*}
f_t(L)=-\langle \widehat{\Sigma}_t,LL^\top\rangle+\frac{\lambda}{2}\|L^\top\widehat{\Sigma}_{0,t}L-I_r\|_\mathrm{F}^2,
\end{equation*}
where $L\in \mathbb{R}^{|S_t|\times r}$. As a result,
\begin{equation}
      \label{eq:div-ft}
\frac{1}{2}\nabla f_t(L) = -\widehat{\Sigma}_t L +\lambda\sot L(L^\top\sot L-I_r).
\end{equation}
Throughout the proof of Proposition \ref{prop: gradana}, we will work on the restricted function and its gradient. We denote $L_{t+1}=L_t-\eta \nabla f_t(L_t)$. 
Then we notice that $L_t = \gd_{t,S_t*},L_{t+1}=V_{t+1,S_t*}$ by our submatrix notation respectively. 
Our proof will work on distance involving $L_t,L_{t+1}$ which transfers to the desired bound as stated in this proposition. 

\par Before the proof we revisit a lemma characterizing the effect of gradient descent.
The following lemma is adapted from Lemma 4 in \cite{chi2019nonconvex} and it is an extension of the gradient descent condition from vectors to rank $r$ matrices. For notational convenience, we define $L^*_t$ to be a global minimizer of function $f_t(L)$ (which will be calculated later) and \begin{equation*}
H_X=\text{argmin}_{H\in{\mathcal{O}(r)}}\|XH-L^*_t\|_\mathrm{F}.
\end{equation*}
We define a function $f(L)$ to be \textit{$\beta$ smooth} at $L$ if for all $Z$, we have \begin{equation*}
\ve(Z)^\top \nabla^2f(L)\ve(Z)\leq \beta\|Z\|_\mathrm{F}^2.
\end{equation*}

\begin{lemma}\label{lemma:gradient}
Suppose that $f_t$ is $\beta$ smooth within a ball $B(L^*_t)=\{L:\|L-L^*_t\|_\mathrm{F}\leq R\}$ 
and that $\nabla f_t(L) P=\nabla f_t(LP)$ for any orthonormal matrix $P$. Assume that for any $L\in B(L^*_t)$ and any $Z$, we have\begin{equation*}
\ve (ZH_Z-L^*_t)^\top  \nabla^2f_t(L)\ve (ZH_Z-L^*_t)\geq \alpha \|ZH_Z-L^*_t\|_\mathrm{F}^2.
\end{equation*}
In addition, if $\eta \leq \frac{1}{\beta}$, then using gradient descent with $\dist(L_t,L^*_t)\leq R$, we have\begin{equation*}
\mathrm{dist}^2(L_{t+1}, L^*_t)\leq\left(1-\alpha\eta\right)\mathrm{dist}^2(L_t,L^*_t).\end{equation*}
Moreover, with $\dist(L_0,L^*_t)\leq R$, we have\begin{equation*}
\mathrm{dist}^2(L_t, L^*_t)\leq\left(1-\alpha\eta\right)^t \mathrm{dist}^2(L_0,L^*_t).
\end{equation*}
 
\end{lemma}
In Lemma 4 of \cite{chi2019nonconvex} the condition on gradient descent is $L_t\in B(L^*_t)$. Here we generalize it to $\dist(L_t,L^*_t)\leq R$ and the proof follows without change of the original proof.
By \eqref{eq:div-ft}, it is straightforward to verify the condition $\nabla f_t(L) P=\nabla f_t(LP)$.
The remaining proof is composed of three steps: (1) deriving the expression for $\ve(Z)^\top  \nabla^2f_t(L)\ve(Z)$, (2) verifying the smoothness condition, and (3) verifying the condition on strong convexity. 
We check the radius condition at the end of the proof.  

\paragraph{Step (1)} 
Recall \eqref{eq:div-ft}.
As a result, 
\begin{equation*}
\frac{1}{2}\ve \nabla f_t(L)=-(I_{r}\otimes \widehat{\Sigma}_t)\ve (L)+\lambda(I_{r}\otimes\sot LL^\top  \sot)\ve (L)-\lambda (I_{r}\otimes \sot)\ve(L).
\end{equation*} The main calculation is to deal with\begin{equation*}
\ve(\sot LL^\top  \sot L)=(I\otimes\sot LL^\top  \sot)\ve (L).
\end{equation*}
We now directly compute this expression as follows: since $L\in\mathbb{R}^{|S_t|\times r}$, we have $\ve(L)=[l_1^\top  , l_2^\top,...,l_r^\top  ] ^\top $ where $l_i$ is the $i$ th column of the matrix. Following this notation, we can write\begin{equation*}
\ve(\sot LL^\top  \sot L)=(I\otimes\sot LL^\top  \sot)\ve (L)=\begin{bmatrix}
\sot LL^\top  \sot & 0 &...&0\\
0 & \sot LL^\top  \sot &...&0\\
0&0& ...&0\\
0 & 0 &...&\sot LL^\top  \sot
\end{bmatrix}\begin{bmatrix}
l_1 \\
l_2\\
l_3\\
...\\
l_r
\end{bmatrix}.
\end{equation*}
As a result, \begin{equation*}
\ve(\sot LL^\top  \sot L)=\begin{bmatrix}
\sot LL^\top  \sot l_1 \\
\sot LL^\top  \sot l_2\\
\sot LL^\top  \sot l_3\\
...\\
\sot LL^\top  \sot l_r\\
\end{bmatrix}=\begin{bmatrix}
\sot \sum_{i=1}^r l_il_i^\top  \sot l_1 \\
\sot \sum_{i=1}^r l_il_i^\top  \sot l_2\\
\sot \sum_{i=1}^r l_il_i^\top  \sot l_3\\
...\\
\sot\sum_{i=1}^r l_il_i^\top  \sot l_r
\end{bmatrix}.
\end{equation*}
Now we can calculate the derivative $\frac{\partial\ve \nabla f_t(L)}{\partial\ve(L)}$, note that we can do this block by block: the $jk$ block entry of the Hessian is just $\frac{\partial\ve \nabla f_t(L)_j}{\partial l_k}$. Since\begin{equation*}
\frac{1}{2}\ve \nabla f_t(L)=-(I_{r}\otimes \widehat{\Sigma}_t)\ve (L)+\lambda(I_r\otimes\sot LL^\top  \sot)\ve (L)-\lambda (I_{r}\otimes \sot)\ve(L),
\end{equation*}
the only term we have to deal with is the middle one. By previous calculations, we have 
\begin{equation*}
\frac{\partial \sot \sum_{i=1}^r l_il_i^\top  \sot l_j}{\partial l_k}=\frac{\partial \sot l_kl_k^\top  \sot l_j}{\partial l_k}=l_j^\top  \sot l_k\sot+\sot l_j l_k^\top\sot,
\end{equation*}
when $j\neq k$ and \begin{align*}
\frac{\partial \sot \sum_{i=1}^r l_il_i^\top  \sot l_j}{\partial l_j}&=\sum_{i\neq j}\sot l_il_i^\top  \sot + 2\sot l_jl_j^\top  \sot+l_j^\top  \sot l_j\sot\\&=\sot LL^\top  \sot+\sot l_jl_j^\top  \sot+l_j^\top  \sot l_j\sot.
\end{align*}
As a result, we can combine the above results to obtain the Hessian \begin{align*}
\frac{1}{2}\nabla^2f_t(L)=&\frac{1}{2}\frac{\partial\ve \nabla f_t(L)}{\partial\ve(L)}=-\begin{bmatrix}
\widehat{\Sigma}_t & ... & 0\\
0 & \widehat{\Sigma}_t & 0\\
0&...&\widehat{\Sigma}_t
\end{bmatrix}-\lambda\begin{bmatrix}
\sot & .. & 0\\
0 & \sot & 0\\
0&...&\sot
\end{bmatrix}\\
&+\begin{bmatrix}
\sot LL^\top  \sot+\sot l_1l_1^\top  \sot+l_1^\top  \sot l_1\sot & ... & l_1^\top  \sot l_r\sot+\sot l_1 l_r^\top\sot\\
...&...&...\\
\sot l_r l_1^\top\sot+l_r^\top  \sot l_1\sot & ...&\sot LL^\top  \sot+\sot l_rl_r^\top  \sot+l_r^\top  \sot l_r\sot 
\end{bmatrix}.
\end{align*}
Now we proceed to calculate $\ve(Z)^\top  \nabla^2 f_t(L)\ve(Z)$. We have \begin{equation*}
\ve(Z)^\top  \nabla^2 f_t(L)\ve(Z)=\sum_{i,j}z_i^\top\nabla^2f_t(L)_{ij}z_j
\end{equation*}
where $z_i$ is the $i$ th column of $Z$. Substituting the expression on the Hessian, we have\begin{align*}
\frac{1}{2}\ve(Z)^\top  \nabla^2 f_t(L) \ve(Z)=&-\sum_{i=1}^rz_i^\top  \widehat{\Sigma}_tz_i-\lambda(\sum_{i=1}^rz_i^\top  \sot z_i)+\lambda(\sum_{i\neq j} l_i^\top  \sot l_jz_i^\top  \sot z_j+\sum_{i\neq j}z_i^\top \sot l_j l_i^\top \sot z_j\\
&+\sum_{i=1}^r (z_i^\top  \sot l_i)^2+\sum_{i=1}^r z_i^\top\sot LL^\top  \sot z_i+\sum_{i=1}^r (l_i^\top  \sot l_i)(z_i^\top  \sot l_i)).
\end{align*}
Now we claim the following simplification:
\begin{align*}
\frac{1}{2}\ve(Z)^\top  \nabla^2 f_t(L)\ve(Z)=& g_t(L,Z)\\=&-\langle \widehat{\Sigma}_t,ZZ^\top  \rangle-\lambda\langle \sot,ZZ^\top  \rangle +\lambda\langle ZZ^\top  ,\sot LL^\top  \sot\rangle +\lambda\langle L^\top  \sot L,Z^\top  \sot Z\rangle\\
&+\lambda\langle  Z^\top\sot L,L^\top\sot Z  \rangle.
\end{align*}
Now we begin to prove this claim. We expand the above expression as follows:\begin{align*}
g_t(L,Z)=&-\langle \widehat{\Sigma}_t,ZZ^\top  \rangle-\lambda(\langle \sot,ZZ^\top  \rangle -\langle ZZ^\top  ,\sot LL^\top  \sot\rangle \\
& ~~~-\langle L^\top  \sot L,Z^\top  \sot Z\rangle-\langle Z^\top\sot L,L^\top\sot Z\rangle) \\
=&-\sum_{i=1}^rz_i^\top  \widehat{\Sigma}_tz_i-\lambda\sum_{i=1}^rz_i^\top  \sot z_i+\lambda \sum_{i,j}z_i^\top  \sot l_j l_j^\top \sot z_i+\lambda\sum_{i,j}z_i^\top\sot l_j l_i^\top\sot z_j \\
&+\lambda \sum_{i\neq j}l_i^\top  \sot l_jz_i^\top  \sot z_j+\lambda \sum_{i=1}^r (l_i^\top  \sot l_i)(z_i^\top  \sot z_i)\\
=&-\sum_{i=1}^rz_i^\top  \widehat{\Sigma}_tz_i-\lambda\sum_{i=1}^rz_i^\top  \sot z_i+\lambda(\sum_{i\neq j} l_i^\top  \sot l_jz_i^\top  \sot z_j+\sum_{i\neq j}z_i^\top\sot l_j l_i^\top\sot z_j \\
&+\sum_{i=1}^r (l_i^\top  \sot l_i)(z_i^\top  \sot z_i)+\sum_{i=1}^r z_i^\top  \sot \sum_{j=1}^r l_jl_j^\top  \sot z_i+\sum_{i=1}^r (z_i^\top  \sot l_i)^2 )\\
=&-\sum_{i=1}^rz_i^\top  \widehat{\Sigma}_tz_i-\lambda(\sum_{i=1}^rz_i^\top  \sot z_i)+\lambda(\sum_{i\neq j} l_i^\top  \sot l_jz_i^\top  \sot z_j+\sum_{i\neq j}z_i^\top \sot l_j l_i^\top \sot z_j\\
&+\sum_{i=1}^r (z_i^\top  \sot l_i)^2+\sum_{i=1}^r z_i^\top\sot LL^\top  \sot z_i+\sum_{i=1}^r (l_i^\top  \sot l_i)(z_i^\top  \sot l_i))\\
=&\frac{1}{2}\ve(Z)^\top  \nabla^2 f_t(L) \ve(Z).
\end{align*}
This completes the first step. 

\paragraph{Step (2)} 
In view of the lemma, our next task would be to bound the smoothness parameter in a neighborhood of $L_t^*$.  The neighborhood will be defined by the distance $\|\sot^{1/2}L_t^*-\sot^{1/2} L\|_\mathrm{F}\leq \delta$, (we define in this unusual way due to the normalization constraint and specific $\delta$ given by the condition will be explained later), which by triangle inequality gives\begin{equation*}
\|\sot^{1/2}L^*_t\|_\op-\delta\leq  \|\sot^{1/2}L\|_{\op}\leq \|\sot^{1/2}L^*_t\|_\op+\delta.
\end{equation*}
We first find this global minimizer $L_t^*$ up to a rotation matrix. Setting the gradient equal to 0, we have any critical point must satisfy the following equation\begin{equation*}
    \widehat{\Sigma}_tL=\lambda\sot L(L^\top \sot L-I_r).
\end{equation*}
From the above equation, we deduce that the column space of global minimizer should coincide with some generalized eigenvectors (not necessarily leading ones). Hence without loss of generality we assume the global minimizer of function is achieved at when $L^*_t=\widehat{L}_tD_t$ for $D_t$ being an invertible matrix and $\widehat{L}_t$ being the generalized eigenvectors with eigenvalues $\widehat{\lambda}_\cI$ for $\cI\subseteq[p], |\cI|=r$ for sample covariance matrices. Then we have \begin{equation*}
    \widehat{\Sigma}_t\widehat{L}_tD_t=\lambda\sot \widehat{L}_tD_t(D_t^2-I_r)
\end{equation*}
and this gives\begin{equation*}
    \sot \widehat{L}_t\widehat{\Lambda}_\cI D_t=\lambda\sot \widehat{L}_tD_t(D_t^2-I_r).
\end{equation*}
Here we abuse the notation a bit to denote diagonal matrices with entries $\widehat{\lambda}_\cI$ to be $\widehat{\Lambda}_\cI$. 
We deduce that $D_t=(I_r+\frac{1}{\lambda}\widehat{\Lambda}_\cI)^{\frac{1}{2}}$. 
This is true for any critical point and now we will show that the global minimizer is achieved at $\cI$ being $\{1,2,...,r\}$. To see this, note that\begin{align*}
f_t(L^*_t)&=f_t(\widehat{L}_tD_t)=-\langle \widehat{\Sigma}_t,\widehat{L}_tD_tD_t^\top \widehat{L}_t^\top\rangle+\frac{\lambda}{2}\|D_t^\top \widehat{L}_t^\top \so\widehat{L}_tD_t-I_r\|_\mathrm{F}^2\\
&=-\Tr(D_t^\top \widehat{\Lambda}_\cI D_t)+\frac{\lambda}{2}\|D_t^\top D_t-I_r\|_\mathrm{F}^2=-\Tr(\widehat{\Lambda}_\cI)-\frac{1}{\lambda}\Tr(\widehat{\Lambda}_\cI^2)+\frac{1}{2\lambda}\Tr(\widehat{\Lambda}_\cI^2).
\end{align*}
We notice that the above quantity is minimized only when $\widehat{\Lambda}_\cI=\hl(S_t)$, that is, when we are selecting the leading $r$ generalized eigenvectors of $\widehat\Sigma_t$ with respect to $\widehat\Sigma_{0,t}$. 
As a result, 
\begin{equation*}
    D_t=\left(I_r+\frac{1}{\lambda}\hl(S_t)\right)^{\frac{1}{2}},\quad L^*_t=\widehat{A}(S_t)_{S_t*}\left(I_r+\frac{1}{\lambda}\hl(S_t)\right)^{\frac{1}{2}}=\widehat{V}(S_t)_{S_t*},
\end{equation*}
where $\hl(S_t)$ is the diagonal matrix with entries being first $r$ generalized eigenvalues for sample covariance matrices as specified before. According to our definition, $\widehat{A}(S_t)_{S_t*}$ and $\widehat{V}(S_t)_{S_t*}$ are
of size $|S_t|\times r$.
In the rest of the proof, we denote $\widehat{V}(S_t)_{S_t*}$ by $\widehat{V}$ for simplicity. Similarly, we slightly abuse notation and abbreviate $\hl(S_t),\widehat{A}(S_t)_{S_t*},\widehat{B}(S_t)_{S_t*}$ as $\hl,\ha,\hb$ in the rest of this proof. 
Then we have $\|\sot^{1/2} L_t^*\|_\op=\|\sot^{1/2}\hU\|_\op=\sqrt{1+\frac{\widehat{\lambda}_1}{\lambda}}\leq1+\frac{\lambda_1}{\lambda}$, on event $B_2$ and assumption in the theorem. Here with slight abuse of notation, we use $\widehat{\lambda}_i$ to denote the $i$th restricted sample generalized eigenvalue.

Now we can bound the smoothness parameter from above by controlling each term in Hessian matrix.  
Since we are working on event $B_2$, we have that
$\frac{1}{2}\rev{\nu}\leq\|\widehat{\Sigma}_t\|_\op\leq 2\rev{\nu}$ and $\frac{1}{2}\rev{\nu}\leq\|\sot\|_\op\leq2\rev{\nu}$. 
Similar bound holds for minimum restricted sample generalized eigenvalue. 
When $\|\sot^{1/2}L^*_t-\sot^{1/2} L\|_\mathrm{F}\leq \delta$,
we have
\begin{align*}
g_t(L,Z) & =
-\langle \widehat{\Sigma}_t,ZZ^\top  \rangle-\lambda\langle \sot,ZZ^\top  \rangle +\lambda\langle ZZ^\top  ,\sot LL^\top  \sot\rangle +\lambda\langle L^\top  \sot L,Z^\top  \sot Z\rangle\\
&~~~+\lambda\langle Z^\top \sot L, L^\top\sot Z \rangle \\
& \leq 0+0+\lambda\|Z\|_\mathrm{F}^2\|L^\top  \sot\|_{\op}^2 
+\lambda\|L^\top\sot LZ^\top\|_{\mathrm{F}}\|\sot Z \|_{\mathrm{F}}+\lambda \|Z^\top\sot L\|_\mathrm{F}\|L^\top\sot Z\|_\mathrm{F}\\ 
& \leq 4\lambda\rev{\nu}(\delta+\|\sot^{1/2}\widehat{V}\|_{\op} 
)^2\|Z\|_\mathrm{F}^2+2\lambda (\delta+\|\sot^{1/2}\widehat{V}\|_{\op})^2\rev{\nu}\|Z\|_\mathrm{F}^2
\\
& \leq (6\lambda\rev{\nu}(\delta+\|\sot^{1/2}\widehat{V}\|_{\op})^2)\|Z\|_\mathrm{F}^2=\frac{1}{2}\beta \|Z\|_\mathrm{F}^2.
\end{align*}
 Then we can upper bound the Hessian eigenvalue by
\begin{equation*}
    \beta=12\lambda\rev{\nu}\left(\delta+1+\frac{\lambda_1}{\lambda}\right)^2.
\end{equation*}

\paragraph{Step (3)} 
To derive the strong convexity parameter $\alpha$, we start from the function evaluated at the global minimizer $L=\hU$. Define $\tZ = ZH_Z-\hU$, we now lower bound $g_t(\hU,\tZ)$ by a constant multiple of $\frac{1}{2}\|\tZ\|_\mathrm{F}^2$. To do this, we substitute the expression into $g_t$ and obtain\begin{align*}
    g_t(\widehat{V},\tZ)&=-\langle \widehat{\Sigma}_t+\lambda\sot,\tZ\tZ^\top  \rangle +\lambda\langle \hU^\top  \sot\hU,\tZ^\top  \sot \tZ\rangle +\lambda\langle \tZ^\top  \sot \hU, \tZ^\top  \sot\hU\rangle+\lambda \langle \tZ^\top\sot \hU,\hU^\top\sot \tZ\rangle\\
    &=-\langle \widehat{\Sigma}_t+\lambda\sot,\tZ\tZ^\top  \rangle+\lambda\langle I+\frac{1}{\lambda}\hl,\tZ^\top  \sot \tZ\rangle +\lambda\langle \tZ\tZ^\top  ,\sot \hU\hU^\top  \sot\rangle +\lambda \langle \tZ^\top\sot \hU,\hU^\top\sot \tZ\rangle\\
    &=\underbrace{-\langle \widehat{\Sigma}_t,\tZ\tZ^\top  \rangle +\langle \hl,\tZ^\top  \sot \tZ\rangle +\lambda\langle \tZ\tZ^\top  ,\sot \hU\hU^\top  \sot\rangle}_{\text{Term I}}+\underbrace{\lambda \langle \tZ^\top\sot \hU,\hU^\top\sot \tZ\rangle}_{\text{Term II}}.
\end{align*}
We deal with Term I and Term II separately. To simplify Term I, we recall that 
\begin{equation*}
    \sot^{-1}\widehat{\Sigma}_t\sot^{-1}  =\ha\hl\ha^\top  +\hb\has\hb^\top,  
\end{equation*}
where we have defined the remaining generalized eigenvectors by $\hb$ and the rest of eigenvalues by diagonal entries of $\has$, both on the restrict set $S_t$. 
Here we also omit the dependence on $t$. 
By the definition of restricted sample generalized eigenvectors, we have 
\begin{align*}
\text{Term I}
& = \langle \hl ,\tZ^\top  \sot \tZ\rangle +\langle \tZ\tZ^\top  ,\lambda\sot \hU\hU^\top  \sot-\widehat{\Sigma}_t\rangle \\
& = \langle \tZ\hl \tZ^\top  ,\sot\rangle +\langle \tZ\tZ^\top  ,\sot(\lambda \ha (I+\hl/\lambda)\ha ^\top  - \sot^{-1}\widehat{\Sigma}\sot^{-1}) \sot\rangle\\
& \geq \widehat{\lambda}_r\langle \tZ\tZ^\top  ,\sot\rangle+\langle \tZ\tZ^\top  ,\sot(\lambda \ha (I+\hl/\lambda)\ha ^\top  -(\ha \hl\ha ^\top  +\hb \has\hb ^\top  )) \sot\rangle\\
& = \langle \tZ\tZ^\top  , \sot^{1/2}(\widehat{\lambda}_r I+\sot^{1/2}\ha (\lambda I )\ha^\top \sot^{1/2}-\sot^{1/2}\hb \has\hb \sot^{1/2})\sot^{1/2}\rangle \\
& = \langle \tZ\tZ^\top  , \sot^{1/2}(\widehat{\lambda}_r (\sot^{1/2}\ha\ha^\top  \sot^{1/2}+\sot^{1/2}\hb\hb^\top  \sot^{1/2})+\lambda\sot^{1/2}\ha \ha ^\top  \sot^{1/2}\\
& ~~~ -\sot^{1/2}\hb \has \hb^\top  \sot^{1/2})\sot^{1/2}\rangle \\
& = \langle \tZ\tZ^\top  , \sot^{1/2}(\sot^{1/2}\ha (\lambda+\widehat{\lambda}_r )I\ha ^\top  \sot^{1/2} +\sot^{1/2}\hb (\widehat{\lambda}_r -\has)\hb^\top\sot^{1/2})\sot^{1/2}\rangle \\
& \geq \frac{1}{2\rev{\nu}}\langle \tZ\tZ^\top  , (\widehat{\lambda}_r -\widehat{\lambda}_{r+1}) I\rangle = \frac{1}{2\rev{\nu}}(\widehat{\lambda}_r -\widehat{\lambda}_{r+1})\|\tZ\|_\mathrm{F}^2.
\end{align*}
The first inequality is due to the fact we only select the first $r$ generalized eigenvalues, the second inequality follows from the fact that the second term can be bounded using the eigengap since $\has$ is a diagonal matrix containing the last $|S_t| -r$ generalized eigenvalues. 
As a result, Term I can be bounded in terms of eigengap at sample level. Note that the above inequality holds for any $\tZ$. 

Now we bound Term II. Here we use the fact that $H_Z$ is the solution to $\min_{P\in\mathcal{O}(r)} \|ZP-\hU\|_\mathrm{F}$. By definition of $\tZ$, we hope to bound $\lambda\langle (ZH_Z-\widehat{V})^\top\sot \hU,\hU^\top \sot (ZH_Z-\widehat{V})\rangle.$ We will prove now that this can be lower bounded by 0, as is the case in Example 1 (46) in \cite{chi2019nonconvex}. Note that we have 
\begin{align*}
& \langle \tZ^\top\sot \hU,\hU^\top \sot \tZ\rangle\\
&=\Tr(\hU^\top\sot(ZH_Z-\widehat{V})\hU^\top \sot (ZH_Z-\widehat{V})\rangle\\
&=\Tr(\sot (ZH_Z-\hU)\hU^\top\sot (ZH_Z-\hU)\hU^\top )\\
&= \Tr(\sot ZH_Z\hU^\top\sot ZH_Z\hU^\top)-2\Tr(\sot \hU\hU^\top\sot ZH_Z\hU^\top)+\Tr(\sot \hU\hU^\top\sot \hU\hU^\top).
\end{align*}
By Lemma 2 in \cite{ten1977orthogonal}, we know that $ ZH_Z\hU^\top\succeq 0$ and it is also symmetric. 
If we denote $ZH_Z\hU^\top=L_0L_0^\top$ and write $Z_1 = \sot^{1/2}L_0, Z_2=\sot^{1/2}\hU$, we have
\begin{equation*}
\langle \tZ^\top\sot \hU,\hU^\top \sot \tZ\rangle
=\|Z_1Z_1^\top \|_\mathrm{F}^2+\|Z_2Z_2^\top\|_\mathrm{F}^2-2\langle Z_1Z_1^\top,Z_2Z_2^\top \rangle\geq 0
\end{equation*} by Cauchy Schwarz Inequality. 
This proves that Term II is non-negative. Finally we conclude that we have \begin{equation*}
g_t(\widehat{V},\tZ)\geq \frac{1}{2\rev{\nu}} (\widehat{\lambda}_r -\widehat{\lambda}_{r+1})\|\tZ\|_\mathrm{F}^2.
\end{equation*}

Now we argue for general $L$. 
First we show that for any $L$ in the neighborhood of $\widehat{V}$ 
\begin{align*}
    |g_t(L,\tZ)-g_t(\hU,\tZ)| 
      & =
      |\lambda\langle \tZ\tZ^\top,\sot (LL^\top  -\hU\hU^\top  )\sot\rangle +\lambda\langle L^\top  \sot L-\hU^\top  \sot\hU,\tZ^\top  \sot \tZ\rangle\\
    &~~~ +\lambda (\langle \tZ^\top \sot L, L^\top\sot \tZ\rangle-\langle \tZ^\top \sot \hU, \hU^\top\sot \tZ\rangle)|\\
    & \leq c_1 \|L-\hU\|_\mathrm{F}\|\tZ\|_\mathrm{F}^2,
\end{align*}
for some positive constant $c_1$ depending on $\lambda$ and $\rev{\nu}$. Specifically, this bound can be obtained by bounding each term. The first term can be bounded by \begin{align*}
    \lambda|\langle \tZ\tZ^\top  ,\sot (LL^\top  -\hU\hU^\top  )\sot\rangle|&\leq \lambda\|\tZ\|_\mathrm{F}^2\|\sot (LL^\top  -\hU\hU^\top  )\sot\|_{\op}\\
    &\leq 2\lambda\rev{\nu}\|\sot^{1/2}LL^\top\sot^{1/2}  -\sot^{1/2}\hU\hU^\top  \sot^{1/2}\|_\op\|\tZ\|_\mathrm{F}^2\\
    &\leq \frac{9}{2}\lambda\rev{\nu}\|\sot^{1/2}L\|_\op\|\sot^{1/2}L-\sot^{1/2}\hU\|_\mathrm{F}\|\tZ\|_\mathrm{F}^2\\
    &\leq \frac{9}{2}\sqrt{2}\lambda\rev{\nu}^{3/2}\left(1+\frac{\lambda_1}{\lambda}+\delta\right)\|L-\hU\|_\mathrm{F}\|\tZ\|_\mathrm{F}^2,
\end{align*}
under the assumption that $\dist(\sot^{1/2}L,\sot^{1/2}\hU)\leq\frac{1}{4}\|\sot^{1/2}\hU\|_\op=\frac{1}{4}\sqrt{1+\frac{\widehat{\lambda}_1}{\lambda}}.$
The first inequality is the property of matrix norm, the second inequality is due to the bound on sample covariance matrix. The third inequality follows from Lemma \ref{sub dis lemma} and the last inequality follows from the neighborhood assumption. 
The second term can be bounded in a similar fashion as
\begin{align*}
&    ~~~~\lambda|\langle L^\top  \sot L-\hU^\top  \sot\hU,\tZ^\top  \sot \tZ\rangle|\\
&\leq \lambda\|(L^\top  \sot L-\hU^\top  \sot\hU)\tZ^\top\|_\mathrm{F}\|  \sot \tZ\|_\mathrm{F}\\
    &\leq 2\lambda\rev{\nu}\|L^\top  \sot L-\hU^\top  \sot\hU\|_\op\|\tZ\|_\mathrm{F}^2\\
    &= 2\lambda\rev{\nu}\|\hU^\top  \sot(L-\hU)+(L-\hU)^\top  \sot \hU+(L-\hU)^\top  \sot(L-\hU)\|_\op\|\tZ\|_\mathrm{F}^2\\
    &\leq 2\lambda\rev{\nu}(2\|\sot^{1/2}\hU\|_\op\|\sot^{1/2}(L-\hU)\|_\op+\delta \|\sot^{1/2}(L-\hU)\|_\op
      )\|\tZ\|_\mathrm{F}^2\\
    &= 2\left(\delta+2\sqrt{1+\frac{\widehat{\lambda}_1}{\lambda}}\right)\lambda\rev{\nu}\|\sot^{1/2}(L-\hU)\|_\op\|\tZ\|_\mathrm{F}^2\\
    &\leq 2\sqrt{2}\left(\delta+2\left(1+\frac{\lambda_1}{\lambda}\right)\right)\lambda\rev{\nu}^{3/2}\|L-\hU\|_\mathrm{F}\|\tZ\|_\mathrm{F}^2.
\end{align*}
The first line is standard matrix inequality and the rest follows  from the neighborhood assumption as well as the assumption on the operator norm of sample covariance matrices on event $B_2$. 
Finally we bound the last term as
\begin{align*}
& ~~~~\lambda |(\langle \tZ^\top \sot L, L^\top\sot \tZ\rangle-\langle \tZ^\top \sot \hU, \hU^\top\sot \tZ\rangle)| \\
&=\lambda |\Tr((\tZ^\top\sot L)^2-(\tZ^\top\sot\hU)^2)|\\
&= \lambda |\Tr(\tZ^\top\sot(L+\hU)\tZ^\top\sot (L-\hU))|\\
&\leq \lambda \|\tZ^\top\sot(L+\hU)\|_\mathrm{F}\|\tZ^\top\sot(L-\hU)\|_\mathrm{F}\\
&\leq 2\sqrt{2}\lambda\rev{\nu}^{3/2}\|\tZ\|_\mathrm{F}^2\|L-\hU\|_\mathrm{F} \|\sot^{1/2}(L+\hU)\|_\op\\
&\leq 2\sqrt{2}\lambda\rev{\nu}^{3/2}\left(2\left(1+\frac{\lambda_1}{\lambda}\right)+\delta\right)\|L-\hU\|_\mathrm{F}\|\tZ\|_\mathrm{F}^2.
\end{align*}
In this way, we see that we can choose the constant to be\begin{equation*}
    c_1=4\sqrt{2}\left(\delta+2\left(1+\frac{\lambda_1}{\lambda}\right)\right)\lambda\rev{\nu}^{3/2}+\frac{9}{2}\sqrt{2}\lambda\rev{\nu}^{3/2}\left(1+\frac{\lambda_1}{\lambda}+\delta\right)=\sqrt{2}\lambda\rev{\nu}^{3/2}\left(\frac{25}{2}\left(1+\frac{\lambda_1}{\lambda}\right)+\frac{17}{2}\delta\right).
\end{equation*} Thus for any $L$ within a $\delta$ neighborhood of $\hU=L^*_t$ as defined above, we have, by triangle inequality, \begin{align*}
    g_t(L,\tZ)&\geq g_t(\hU,\tZ)-|g_t(L,\tZ)-g_t(\hU,\tZ)|\\
    &\geq \frac{1}{2\rev{\nu}}(\widehat{\lambda}_r -\widehat{\lambda}_{r+1})\|\tZ\|_\mathrm{F}^2-c_1 \|L-\hU\|_\mathrm{F}\|\tZ\|_\mathrm{F}^2\\
    &\geq \frac{1}{4\rev{\nu}}(\widehat{\lambda}_r -\widehat{\lambda}_{r+1})\|\tZ\|_\mathrm{F}^2\geq \frac{1}{8\rev{\nu}}(\lambda_r-\lambda_{r+1})\|\tZ\|_\mathrm{F}^2,
\end{align*}
as long as $\|L-\hU\|_\mathrm{F}\leq \frac{\widehat{\lambda}_r -\widehat{\lambda}_{r+1}}{4\rev{\nu}c_1}$, which is guaranteed by the assumption in the theorem and the event $B_2$. 

Recall that $g_t(L,\tZ)=\frac{1}{2}\ve(\tZ)^\top\nabla^2f_t(L) \ve(\tZ)$. 
This motivates us to pick 
\begin{equation*}
\alpha = \frac{(\lambda_r-\lambda_{r+1})}{4\rev{\nu}}
\end{equation*} 
under appropriate radius conditions. 

To finally find the radius of attraction region such that $\dist(L_t,L_t^*)=\dist(\gd_t, \hU(S_t))\leq R$, we notice that throughout the proof for smoothness and strongly convexity to hold, we require 3 conditions on the distance\begin{equation*}
\|L-L_t^*\|_\mathrm{F}\leq \frac{(\widehat{\lambda}_r -\widehat{\lambda}_{r+1})}{4\rev{\nu}c_1},\quad \|\sot^{1/2} L-\sot^{1/2}L_t^*\|_\mathrm{F}\leq \delta, \quad\|\sot^{1/2} L-\sot^{1/2}L_t^*\|_\mathrm{F}\leq\frac{1}{4}\sqrt{1+\frac{\widehat{\lambda}_1}{\lambda}}.
\end{equation*}
Hence to ensure all of the above three conditions to hold, we can just set\begin{equation}\label{Radius Condition}
R=\min\left\{ \frac{(\lambda_r -\lambda_{r+1})}{8\rev{\nu}c_1},\frac{1}{8}\frac{\sqrt{1+\frac{\lambda_1}{\lambda}}}{\sqrt{\rev{\nu}}},\frac{\delta}{\sqrt{2\rev{\nu}}}\right\}.
\end{equation}
Here recall that $\frac{1}{\rev{\nu}}$ is the lower bound of minimum eigenvalue of $\Sigma_0$.
\par In summary, we conclude that there is a constant $\alpha$ such that for any $L$ in the neighborhood of $L_t^*$ and for any $Z$, \begin{equation*}
    \ve (ZH_Z-L^*_t)^\top  \nabla^2f_t(L)\ve (ZH_Z-L^*_t)\geq \alpha \|ZH_Z-L^*_t\|_\mathrm{F}^2.
\end{equation*}
Combine with the upper bound and resort to Lemma \ref{lemma:gradient}, we conclude that if we choose the step-size to be $\eta\leq\frac{1}{\beta}$, under the radius conditions,\begin{equation*}
\mathrm{dist}^2(L_{t+1}, \hU)\leq\left(1-\alpha\eta\right)\mathrm{dist}^2(L_t,\hU).
\end{equation*}
Moreover, we have\begin{equation*}
    \alpha=\frac{(\lambda_r-\lambda_{r+1})}{4\rev{\nu}},\quad \beta=12\lambda\rev{\nu}\left(\delta+1+\frac{\lambda_1}{\lambda}\right)^2.
\end{equation*}
This finishes our analysis for the gradient descent step. We comment here that Proposition \ref{prop: gradana} is then proved by simply taking $\delta=1$ and calculate the radius in \eqref{Radius Condition} and $\alpha,\beta$ accordingly. We keep $\delta$ in the proof for the sake of generality.  
\par Specifically, under the choice that $\lambda=\frac{\lambda_1}{c}$, we have \begin{equation*}
\eta \leq  \frac{c}{12\lambda_1\rev{\nu}(c+2)^2}=\frac{1}{\beta},
\end{equation*}
and we have, by lifting the matrix of size $|S_t|\times r$ to $p\times r$ by filling in 0s, \begin{equation*}
\mathrm{dist}^2(V_{t+1}^o, \hU(S_t))\leq\left(1-\frac{\eta(\lambda_r-\lambda_{r+1})}{4\rev{\nu}}\right)\mathrm{dist}^2(\gd_t,\hU(S_t)).
\end{equation*}
\end{proof}

\subsection{Proof of Proposition \ref{prop: ht}}
\label{sec:proof-ht}
\begin{proof}
Recall that we have defined 
\begin{equation*}
\mathrm{supp}(\scale)=S,\quad \mathrm{supp}(\gd_{t+1})=F_{t+1}.
\end{equation*}
Now we define sets $F_1=S\symbol{92}F_{t+1}$, $F_2=S\cap F_{t+1}$ and $F_3=F_{t+1}\symbol{92}S$, the sizes of which are denoted by $k_1,k_2$ and $k_3$ respectively. 
We also define 
\begin{equation*}
    x_1=\|\scale_{F_1*}\|_\mathrm{F},\quad x_2=\|\scale_{F_2*}\|_\mathrm{F},
    \end{equation*}
\begin{equation*}
    y_1=\|V_{t+1,F_1*}\|_\mathrm{F},\quad
   y_2=\|V_{t+1,F_2*}\|_\mathrm{F},\quad
    y_3 = \|V_{t+1,F_3*}\|_\mathrm{F}.
\end{equation*}
Finally, we define 
\begin{equation*}
    x_1^2+x_2^2=\|\scale\|_\mathrm{F}^2=X^2,\quad y_1^2+y_2^2+y_3^2\leq \|V_{t+1}\|_\mathrm{F}^2=Y^2.
\end{equation*}
Let $\Delta=\Tr(|{\scale}^\top  V_{t+1}|)$ be the quantity we are interested in. 
Here we let $\Tr(|M|)$ denote the trace norm
$\Tr(|M|)=\Tr(\sqrt{M^\top  M})$
which is simply the sum of singular values (a.k.a.~nuclear norm). 
This is a special case of Schatten $p$ norm, defined as $\|T\|_p = \Tr(|T|^p)^{\frac{1}{p}}$ which is essentially vector $p$ norm of a vector composed of singular values. 

We first prove the following claim about the effect of truncation step:
\begin{equation}\label{claim: truncation step}
\Tr(|{\scale}^\top  V_{t+1}|)-  \Tr(|{\scale}^\top  \gd_{t+1}|)\leq\sqrt{\frac{s}{\hs}}
\min\left( \sqrt{X^2Y^2-\Delta^2},
\frac{1+\sqrt{{s}/{\hs}}}{XY}(X^2Y^2-\Delta^2) \right).
\end{equation}
\begin{proof}
Since in the hard thresholding step we greedily pick the rows with largest $l_2$ norms, we have
\begin{equation*}
\frac{y_1^2}{k_1}\leq \frac{y_3^2}{k_3}.
\end{equation*}
In addition, since $k_1 + k_2 = s\leq s' = k_2 + k_3$, we also have $k_1\leq k_3$. 
By applying Holder's inequality on Schatten norm and the definition of support set, we have
\begin{align*}
\Delta^2 &\leq (\|\scale_{F_1*}\|_\mathrm{F}\|V_{t+1,F_1*}\|_\mathrm{F}+\|\scale_{F_2*}\|_\mathrm{F}\|V_{t+1,F_2*}\|_\mathrm{F})^2
= (x_1y_1+x_2y_2)^2\\&\leq(y_1^2+y_2^2)X^2\leq(Y^2-y_3^2)X^2\leq X^2Y^2-\frac{k_3}{k_1}X^2y_1^2,
\end{align*}
and this gives \begin{equation*}
    y_1^2\leq \frac{k_1}{k_3X^2}(X^2Y^2-\Delta^2)\leq
      \frac{k_1 + k_2}{k_3 + k_2}(Y^2-\Delta^2/X^2)
      = \frac{s}{\hs}(Y^2-\Delta^2/X^2),
\end{equation*}
where the second inequality holds since $k_1 \leq k_3$. 
Now we split the arguments into two cases.

\par \textit{Case I:} $\Delta< XY\sqrt{\frac{s}{s+\hs}}$.  
In this case, we obtain that \begin{equation*}
\Delta<\sqrt{\frac{s}{\hs}}\sqrt{X^2Y^2-\Delta^2}.
\end{equation*}
This implies \eqref{claim: truncation step} immediately.
 
\par \textit{Case II:} 
We can now assume that $\Delta\geq XY\sqrt{\frac{s}{s+\hs}}$. Then we have\begin{equation*}
    y_1^2\leq  \frac{s}{\hs}(Y^2-\Delta^2/X^2)\leq\frac{\Delta^2}{X^2}.
\end{equation*}
By definition we have
\begin{equation*}
    x_1y_1+\sqrt{Y^2-y_1^2}\sqrt{X^2-x_1^2}\geq x_1y_1+x_2y_2\geq \Delta.
\end{equation*}
Solving this inequality as a quadratic inequality in $x_1$ yields
\begin{equation*}
    x_1\leq \frac{\Delta y_1+\sqrt{(X^2Y^2-\Delta^2)(Y^2-y_1^2})}{Y^2}.
\end{equation*}
By definition of $X^2=x_1^2+x_2^2$ we have  $x_1\leq X$. Also we have $\Delta\leq XY$ by Holder inequality on Schatten norm.

\par Combining the above inequalities together, we have 
\begin{equation*}
    x_1\leq \min\left( X,\frac{Xy_1+\sqrt{X^2Y^2-\Delta^2}}{Y} \right).
\end{equation*}
Note that $y_1\leq\sqrt{\frac{s}{\hs}}\frac{\sqrt{X^2Y^2-\Delta^2}}{X}$ by previous calculation. 
Substituting it into the above equation, we have
\begin{equation*}
    x_1\leq \min\left(X,\frac{\sqrt{\frac{s}{\hs}}{\sqrt{X^2Y^2-\Delta^2}}+\sqrt{X^2Y^2-\Delta^2}}{Y}
      \right).
\end{equation*}
Finally, we can compute that \begin{align*}
x_1y_1&\leq \sqrt{\frac{s}{\hs}}\frac{1}{X}\sqrt{X^2Y^2-\Delta^2}\,
\min\left(X,\frac{1}{Y}\sqrt{X^2Y^2-\Delta^2}(1+\sqrt{\frac{s}{\hs}})\right)\\
&=\sqrt{\frac{s}{\hs}}\min\left( \sqrt{X^2Y^2-\Delta^2},\,
\frac{1+\sqrt{\frac{s}{\hs}}}{XY}(X^2Y^2-\Delta^2) \right).
\end{align*}
This is the error induced by the truncation step, finally we have  \begin{align*}
    \Tr(|{\scale}^\top  V_{t+1}|)-  \Tr(|{\scale}^\top  \gd_{t+1}|)\leq&\Tr(|{\scale}^\top  (\gd_{t+1}-V_{t+1})|)\\\leq&x_1y_1\leq\sqrt{\frac{s}{\hs}}
      \min\left( \sqrt{X^2Y^2-\Delta^2},
    \frac{1+\sqrt{\frac{s}{\hs}}}{XY}(X^2Y^2-\Delta^2) \right).
\end{align*}
This finishes the proof for the claim.
\end{proof}
We now switch to work with the distance metric $\mathrm{dist}(U,V)=\min_{P\in{ \mathcal{O}(r)}}\|UP-V\|_\mathrm{F}$. 
We first expand the expression \begin{align*}
    \|UP-V\|_\mathrm{F}^2&=\|UP\|_\mathrm{F}^2+\|V\|_\mathrm{F}^2-2\langle UP,V\rangle\\
    &=\|U\|_\mathrm{F}^2+\|V\|_\mathrm{F}^2-2\Tr(P^\top  U^\top  V).
 \end{align*} 
So, to minimize the distance metric we defined is to maximize $\Tr(P^\top  U^\top  V)$. To this end, let $ADB^\top  $ denote the singular value decomposition of $U^\top  V$, then we have \begin{equation*}
    \Tr(P^\top  U^\top  V)=\Tr(P^\top  ADB^\top  )=\Tr(B^\top  P^\top  AD)=\Tr(ZD)
\end{equation*}
for $Z$ being an orthogonal matrix. Since $D$ is a diagonal matrix, we have $\Tr(ZD)=\sum_i Z_{ii}D_{ii}$ and to maximize this value we would like to have $Z_{ii}=1$ for all $i$ since $D_{ii}$ are non-negative. As a result, we have the optimal $P$ being $AB^\top  $ and the optimal value is given by\begin{equation*}
    \min \|UP-V\|_\mathrm{F}^2=\|U\|_\mathrm{F}^2+\|V\|_\mathrm{F}^2-2\Tr(D)=\|U\|_\mathrm{F}^2+\|V\|_\mathrm{F}^2-2\Tr(|U^\top  V|).
\end{equation*}

Now we consider the subspace distance we have defined. 
Recall that $\Delta=\Tr(|{\scale}^\top  V_{t+1}|)$. 
Then the subspace distance can be written as
\begin{align*}
   & \mathrm{dist}^2(\gd_{t+1},\scale)\\
   &=\|\scale\|_\mathrm{F}^2+\|\gd_{t+1}\|_\mathrm{F}^2-2\Tr|{\scale}^\top  \gd_{t+1}|\\
    &\leq \|\scale\|_\mathrm{F}^2+\|\gd_{t+1}\|_\mathrm{F}^2-2\Tr|{\scale}^\top  V_{t+1}|+2\sqrt{\frac{s}{\hs}}
      \min\left( \sqrt{X^2Y^2-\Delta^2},
    \frac{1+\sqrt{\frac{s}{\hs}}}{XY}(X^2Y^2-\Delta^2) \right)\\
    &\leq \|\scale\|_\mathrm{F}^2+\|V_{t+1}\|_\mathrm{F}^2-2\Tr|{\scale}^\top  V_{t+1}|+2\sqrt{\frac{s}{\hs}}\frac{1+\sqrt{\frac{s}{\hs}}}{XY}(X^2Y^2-\Delta^2)\\
 & =\mathrm{dist}^2(V_{t+1},\scale)+2\sqrt{\frac{s}{\hs}}\frac{1+\sqrt{\frac{s}{\hs}}}{XY}(X^2Y^2-\Delta^2),
\end{align*}
where the first equality follows from the previous expansion of distance, the first inequality follows from Claim \eqref{claim: truncation step} and we use the fact that truncation reduces the Frobenius norm in the second inequality. In the last equality, we use again the relationship between trace norm and distance. To deal with the extra error term, we use the following bound\begin{align*}
    X^2Y^2-\Delta^2&=(XY+\Delta)(XY-\Delta)\leq 2XY(XY-\Delta)\leq XY(X^2+Y^2-2\Delta)\\
    &=XY(\|\scale\|_\mathrm{F}^2+\|V_{t+1}\|_\mathrm{F}^2-2\Tr(|{\scale}^\top  V_{t+1}|)=XY \mathrm{dist}^2(V_{t+1},\scale).
\end{align*}
Combining the inequalities, we have\begin{align*}
    \mathrm{dist}^2(\gd_{t+1},\scale)&\leq\mathrm{dist}^2(V_{t+1},\scale)+2\sqrt{\frac{s}{\hs}}\frac{1+\sqrt{\frac{s}{\hs}}}{XY}(X^2Y^2-\Delta^2)\\
    &\leq \mathrm{dist}^2(V_{t+1},\scale)+2\sqrt{\frac{s}{\hs}}\frac{1+\sqrt{\frac{s}{\hs}}}{XY}XY \mathrm{dist}^2(V_{t+1},\scale)\\
    &=\left(1+2\sqrt{\frac{s}{\hs}} \left(1+\sqrt{\frac{s}{\hs}} \right) \right)\mathrm{dist}^2(V_{t+1},\scale).
\end{align*}
This finishes our analysis for the hard thresholding step. 
\end{proof}

\subsection{Proof of Theorem \ref{theo:init}}
\begin{proof}
We present the proof of initialization using generalized Fantope in this section. Some proof arguments originate from \cite{gao2017sparse}.

Recall that we have defined 
\begin{equation}
      \label{eq:tA-tF}
\tA=A(A^\top  \widehat{\Sigma}_0A)^{-1/2}, \quad
\tlambda=(A^\top  \widehat{\Sigma}_0A)^{1/2}\Lambda_r(A^\top  \widehat{\Sigma}_0A)^{1/2},\quad
\mbox{and}
\quad
\tF=\tA\tA^\top.
\end{equation}
Throughout the proof, we work on the event  $B_3\cap B_4$, which by Lemmas \ref{prob B3} and \ref{prob B4} happens with probability at least $1-\exp(-C'(s+\log(ep/s)))$ for some constant $C'>0$, {  uniformly over $\mathcal{P}_n$}.


We shall need the following lemma, which is a simplified version of Lemma 6.3 in \cite{gao2017sparse}. 
The proof is essentially the same and so we omit it.
\begin{lemma}(Curvature of Fantope)\label{cur fantope}
Let $P\in \mathcal{O}(p,r)$ and $D=\mathrm{diag}(d_1,d_2,...,d_r)$ with $d_1\geq d_2\geq...\geq d_r\geq 0.$ If $F\in \mathcal{F}_r$, then\begin{equation*}
    \langle PDP^\top, PP^\top-F\rangle\geq \frac{d_r}{2}\|PP^\top-F\|_\mathrm{F}^2.
\end{equation*}
\end{lemma}
Recall that
\begin{equation*}
    \Sigma=\Sigma_0 K\Lambda K^\top  \Sigma_0=\Sigma_0 A\Lambda_r A^\top  \Sigma_0+\Sigma_0 B\LambdaRest B^\top  \Sigma_0.
\end{equation*}
For notational simplicity, for any positive semi-definite matrix $B$, we define\begin{equation*}
    \phi_{\text{max}}^B(k)=\max_{\|u\|_0\leq k, u\neq 0}\frac{u^\top  Bu}{u^\top  u},\quad
    \phi_{\text{min}}^B(k)=\max_{\|u\|_0\leq k, u\neq 0}\frac{u^\top  Bu}{u^\top  u}.
\end{equation*}

In the rest of this proof, 
let $\Delta=\widehat{F}-\tF$ with $\tF$ defined in \eqref{eq:tA-tF}. 
As in \cite{gao2017sparse}, the main proof consists of two steps. 
The first step is to derive upper bound of $\|\widehat{\Sigma}_0^{1/2}\Delta\widehat{\Sigma}_0^{1/2}\|_\mathrm{F}$ and the second step is to lower bound $\|\widehat{\Sigma}_0^{1/2}\Delta\widehat{\Sigma}_0^{1/2}\|_\mathrm{F}$ by $\|\Delta\|_\mathrm{F}$. 

\paragraph{Step (1)}
First note that by Lemma \ref{prob B4}, $\tA$ is well defined on event $B_4$, so $\tF=\tA\tA^\top  $ is also well defined on event $ B_4$.
In addition, on event $B_4$, we have
\begin{equation}
      \label{eq:oracle-error}
    \|\Sigma_0^{1/2}(\tF-AA^\top  )\Sigma_0^{1/2}\|_{\mathrm{F}}\leq C\sqrt{\frac{r(s+\log p)}{n}}.
\end{equation}
\par We first show that $\tF$ is a feasible solution, that is, $\widehat{\Sigma}_0^{1/2}\tA\tA^\top  \widehat{\Sigma}_0^{1/2}\in\mathcal{F}_r$. 
By \eqref{eq:tA-tF},
\begin{equation*}
    \tF=\tA\tA^\top  =A(A^\top  \widehat{\Sigma}_0A)^{-1}A^\top.
\end{equation*}
Let $M= \widehat{\Sigma}_0^{1/2} \tF \widehat{\Sigma}_0^{1/2} =\widehat{\Sigma}_0^{1/2}A(A^\top  \widehat{\Sigma}_0A)^{-1}A^\top  \widehat{\Sigma}_0^{1/2}$. 
Then 
\begin{equation*}
    \Tr(M)=\Tr(\widehat{\Sigma}_0^{1/2}A(A^\top  \widehat{\Sigma}_0A)^{-1}A^\top  \widehat{\Sigma}_0^{1/2})=r.
\end{equation*}
Next we check that $0\preceq M\preceq I$. It is obvious that $0\preceq M$.
Moreover,
\begin{equation*}
    \|M\|_{\op}\leq \|\widehat{\Sigma}_0^{1/2}\tA\|_{\op}\|\widehat{\Sigma}_0^{1/2}\tA\|_{\op}\leq 1
\end{equation*}
by the construction of $\tA$.  
Hence, $\tF\in \mathcal{F}_r$.

Recall that $\widehat{F}$ is the solution to \eqref{eq:initial-fantope}.
The basic inequality implies 
\begin{equation*}
    {-\langle \widehat{\Sigma},\widehat{F}\rangle +\rho\|\widehat{F}\|_1\leq -\langle \widehat{\Sigma},\widetilde{F}\rangle +\rho\|\widetilde{F}\|_1.}
\end{equation*}
Rearranging terms, we have
\begin{equation}
      \label{eq:basic-ineq-1}
    0\leq \rho(\|\tF\|_1-\|\tF+\Delta\|_1)+\langle \widehat{\Sigma},\Delta\rangle.
\end{equation}
We deal with each term on the right side separately. 
For the first term on the right, we have
\begin{equation}
      \label{eq:basic-ineq-2}
\begin{aligned}    \|\tF\|_1-\|\tF+\Delta\|_1&=\|\tF_{SS}\|_1-\|\tF_{SS}+\Delta_{SS}\|_1-\|\Delta_{(SS)^c}\|_1\\
    &\leq \|\Delta_{SS}\|_1-\|\Delta_{{(SS)}^c}\|_1.
\end{aligned}
\end{equation}
Here the first equality holds since $\tF$ is supported on $S\times S$.
For the second term, we have
\begin{equation}
      \label{eq:basic-ineq-3}
    \langle \widehat{\Sigma},\Delta\rangle = \langle\widehat{\Sigma}-\Sigma,\Delta\rangle+\underbrace{\langle \Sigma_0A\Lambda_rA^\top\Sigma_0,\Delta\rangle}_{\text{Term I}} +\underbrace{\langle \Sigma_0B\LambdaRest B^\top\Sigma_0,\Delta\rangle}_{\text{Term II}}.
\end{equation}
We now bound Term I and Term II separately. 

\smallskip

\textit{Bound for Term I}: For Term I, we decompose it in the following way: \begin{align*}
\langle \Sigma_0A\Lambda_rA^\top\Sigma_0,\Delta\rangle&= \langle \Sigma_0A\Lambda_rA^\top\Sigma_0-\so A\Lambda_r A^\top\so,\Delta\rangle+\langle \so A\Lambda_r A^\top\so,\Delta\rangle. 
\end{align*}
Now we bound the second term on the right as follows\begin{align*}
\langle \so A\Lambda_r A^\top\so,\Delta\rangle &= \langle \so^{1/2}A\Lambda_r A^\top\so^{1/2},\so^{1/2}(\widehat{F}-\tF)\so^{1/2} \rangle\\
&=  \langle \so^{1/2}\tA\tlambda \tA^\top\so^{1/2},\so^{1/2}(\widehat{F}-\tF)\so^{1/2} \rangle\\
&= \langle \so^{1/2}\tA\Lambda_r \tA^\top\so^{1/2},\so^{1/2}(\widehat{F}-\tF)\so^{1/2} \rangle+\langle \so^{1/2}\tA(\tlambda-\Lambda_r)\tA^\top\so^{1/2}, \so^{1/2}\Delta\so^{1/2} \rangle \\
&\leq \langle \so^{1/2}\tA\Lambda_r \tA^\top\so^{1/2},\so^{1/2}(\widehat{F}-\tF)\so^{1/2} \rangle+\|\tlambda-\Lambda_r\|_\mathrm{F}\|\so^{1/2}\Delta\so^{1/2} \|_\mathrm{F}.
\end{align*}
As a result, Term I can be bounded in the following way\begin{equation*}
\text{Term I}\leq \langle \Sigma_0A\Lambda_rA^\top\Sigma_0-\so A\Lambda_r A^\top\so,\Delta\rangle+\langle \so^{1/2}\tA\Lambda_r \tA^\top\so^{1/2},\so^{1/2}(\widehat{F}-\tF)\so^{1/2} \rangle+\delta_1\|\so^{1/2}\Delta\so^{1/2} \|_\mathrm{F},
\end{equation*}
for $\delta_1 = \|\tlambda-\Lambda_r\|_\mathrm{F}$.

\smallskip

\textit{Bound for Term II}: 
To bound Term II, we first notice that by definition, $B^\top\Sigma_0A$ is zero matrix since $A,B$ are normalized with respect to $\Sigma_0$. As a result, we have 
\begin{align*}
\langle \Sigma_0B\LambdaRest B^\top\Sigma_0,\Delta\rangle
& =\langle\Sigma_0B\LambdaRest B^\top\Sigma_0,\widehat{F}  \rangle \\
& \leq \lambda_{r+1}\langle\Sigma_0B B^\top\Sigma_0,\widehat{F}  \rangle \\
& =\lambda_{r+1}\langle \Sigma_0BB^\top\Sigma_0,\widehat{F}-\tA\tA^\top\rangle
\\
& =\lambda_{r+1}\langle \Sigma_0^{1/2}(I-\Sigma_0^{1/2}AA^\top\Sigma_0^{1/2})\Sigma_0^{1/2},\widehat{F}-\tA\tA^\top\rangle\\
& = 
\lambda_{r+1}\langle \Sigma_0-\Sigma_0AA^\top\Sigma_0,  \widehat{F}-\tA\tA^\top\rangle
.
\end{align*}
Here, the second last equality holds since $\Sigma_0^{1/2}AA^\top\Sigma_0^{1/2} + \Sigma_0^{1/2}BB^\top\Sigma_0^{1/2} = I$.
We further decompose the rightmost side as
\begin{align*}
& \lambda_{r+1}\langle \Sigma_0-\Sigma_0AA^\top\Sigma_0,  \widehat{F}-\tA\tA^\top\rangle \\
& = \lambda_{r+1}\langle \Sigma_0-\widehat{\Sigma}_0,\Delta\rangle -\lambda_{r+1}\langle\Sigma_0AA^\top\Sigma_0-\widehat{\Sigma}_0AA^\top\so,\Delta \rangle \\
&\quad +\underbrace{\lambda_{r+1}\langle \so, \widehat{F}-\tA\tA^\top \rangle}_{\text{Term A}}-\underbrace{\lambda_{r+1}\langle\so^{1/2}AA^\top\so^{1/2},\so^{1/2}(\widehat{F}-\tA\tA^\top)\so^{1/2}\rangle}_{\text{Term B}}.
\end{align*} 
We now deal with Term A and Term B separately. By definition we have $\so^{1/2}\widehat{F}\so^{1/2}\in\mathcal{F}_r$, hence\begin{equation*}
\langle \so, \widehat{F}-\tA\tA^\top \rangle=\Tr(\so^{1/2}\widehat{F}\so^{1/2})-\Tr(\tA^\top\so\tA)=r-r=0.
\end{equation*}
For Term B, we have 
\begin{align*}
& -\lambda_{r+1}\langle\so^{1/2}AA^\top\so^{1/2},\so^{1/2}(\widehat{F}-\tA\tA^\top)\so^{1/2}\rangle \\
& = -\lambda_{r+1}\langle\so^{1/2}\tA\tA^\top\so^{1/2},\so^{1/2}(\widehat{F}-\tA\tA^\top)\so^{1/2}\rangle\\
&\quad -\lambda_{r+1}\langle \so^{1/2}\tA(A^\top\so A-I)\tA^\top\so^{1/2}, \so^{1/2}(\widehat{F}-\tA\tA^\top)\so^{1/2}\rangle\\
& \leq -\lambda_{r+1}\langle\so^{1/2}\tA\tA^\top\so^{1/2},\so^{1/2}(\widehat{F}-\tA\tA^\top)\so^{1/2}\rangle
+\lambda_{r+1}\|A^\top\so A-I\|_\mathrm{F}\|\so^{1/2}\Delta\so^{1/2}\|_\mathrm{F}.
\end{align*}
Therefore, Term II has the following bound\begin{align*}
\langle \Sigma_0B\LambdaRest B^\top\Sigma_0,\Delta\rangle \leq &\lambda_{r+1}\langle \Sigma_0-\widehat{\Sigma}_0,\Delta\rangle -\lambda_{r+1}\langle\Sigma_0AA^\top\Sigma_0-\widehat{\Sigma}_0AA^\top\so,\Delta \rangle \\
&-\lambda_{r+1}\langle\so^{1/2}\tA\tA^\top\so^{1/2},\so^{1/2}(\widehat{F}-\tA\tA^\top)\so^{1/2}\rangle+\delta_2\|\so^{1/2}\Delta\so^{1/2}\|_\mathrm{F},
\end{align*}
for $\delta_2 = \lambda_{r+1}\|A^\top\so A-I\|_\mathrm{F}$. 
Now we combine the results for Term I and Term II to obtain
\begin{align*}
& \hskip -2em \langle \Sigma_0A\Lambda_rA^\top\Sigma_0,\Delta\rangle+\langle \Sigma_0B\LambdaRest B^\top\Sigma_0,\Delta\rangle \\
\leq & \langle\Sigma_0A\Lambda_rA^\top\Sigma_0-\so A\Lambda_r A^\top\so,\Delta\rangle\\&+\langle \so^{1/2}\tA\Lambda_r \tA^\top\so^{1/2},\so^{1/2}(\widehat{F}-\tF)\so^{1/2} \rangle
+\delta_1\|\so^{1/2}\Delta\so^{1/2} \|_\mathrm{F}\\
&+\lambda_{r+1}\langle \Sigma_0-\widehat{\Sigma}_0,\Delta\rangle -\lambda_{r+1}\langle\Sigma_0AA^\top\Sigma_0-\widehat{\Sigma}_0AA^\top\so,\Delta \rangle \\
&-\lambda_{r+1}\langle\so^{1/2}\tA\tA^\top\so^{1/2},\so^{1/2}(\widehat{F}-\tA\tA^\top)\so^{1/2}\rangle+\delta_2\|\so^{1/2}\Delta\so^{1/2}\|_\mathrm{F}\\
=& \langle\Sigma_0A\Lambda_rA^\top\Sigma_0-\so A\Lambda_r A^\top\so,\Delta\rangle+\lambda_{r+1}\langle \Sigma_0-\widehat{\Sigma}_0,\Delta\rangle \\
&-\lambda_{r+1}\langle\Sigma_0AA^\top\Sigma_0-\widehat{\Sigma}_0AA^\top\so,\Delta \rangle+\delta\|\so^{1/2}\Delta\so^{1/2}\|_\mathrm{F}\\
&+ \langle \so^{1/2}\tA(\Lambda_r-\lambda_{r+1}I) \tA^\top\so^{1/2},\so^{1/2}(\widehat{F}-\tF)\so^{1/2} \rangle,
\end{align*}
where $\delta = \delta_1+\delta_2$.
Plugging back to \eqref{eq:basic-ineq-1}, together \eqref{eq:basic-ineq-2} with \eqref{eq:basic-ineq-3}, we have 
\begin{equation}
\label{eq:basic-ineq-4}
\begin{aligned}
0& \leq \rho(\|\tF\|_1-\|\tF+\Delta\|_1)+\langle \widehat{\Sigma},\Delta\rangle\\
& \leq  \rho(\|\Delta_{SS}\|_1-\|\Delta_{(SS)^c}\|_1) 
+\langle\widehat{\Sigma}-\Sigma,\Delta\rangle \\
& \quad + \langle\Sigma_0A\Lambda_rA^\top\Sigma_0-\so A\Lambda_r A^\top\so,\Delta\rangle+\lambda_{r+1}\langle \Sigma_0-\widehat{\Sigma}_0,\Delta\rangle \\
&\quad -\lambda_{r+1}\langle\Sigma_0AA^\top\Sigma_0-\widehat{\Sigma}_0AA^\top\so,\Delta \rangle \\
& \quad +\delta\|\so^{1/2}\Delta\so^{1/2}\|_\mathrm{F}+ \langle \so^{1/2}\tA(\Lambda_r-\lambda_{r+1}I) \tA^\top\so^{1/2},\so^{1/2}(\widehat{F}-\tF)\so^{1/2} \rangle.
\end{aligned}
\end{equation}
By Holder's inequality, we have 
\begin{equation}
      \label{eq:Delta-terms}
\begin{aligned}
& \langle\widehat{\Sigma}-\Sigma,\Delta\rangle+\langle\Sigma_0A\Lambda_rA^\top\Sigma_0-\so A\Lambda_r A^\top\so,\Delta\rangle \\
& \quad +\lambda_{r+1}\langle \Sigma_0-\widehat{\Sigma}_0,\Delta\rangle-\lambda_{r+1}\langle\Sigma_0AA^\top\Sigma_0-\widehat{\Sigma}_0AA^\top\so,\Delta \rangle\\
& \leq  \|\widehat{\Sigma}-\Sigma\|_\infty\|\Delta\|_1+\|\Sigma_0A\Lambda_rA^\top\Sigma_0-\so A\Lambda_r A^\top\so\|_\infty\|\Delta\|_1\\
& \quad 
+\lambda_{r+1}\|\Sigma_0-\widehat{\Sigma}_0\|_\infty\|\Delta\|_1
+\lambda_{r+1}\|\Sigma_0AA^\top\Sigma_0-\widehat{\Sigma}_0AA^\top\so\|_\infty\|\Delta\|_1.
\end{aligned}
\end{equation}
On event $B_3$, we can pick $\rho=\gamma\sqrt{\frac{\log p}{n}}$ for some large constant $\gamma > 0$ such that  
\begin{equation}
      \label{canonical pair condition}
\begin{aligned}
& \|\widehat{\Sigma}-\Sigma\|_\infty+\|\Sigma_0A\Lambda_rA^\top\Sigma_0-\so A\Lambda_r A^\top\so\|_\infty \\
& \quad
+\lambda_{r+1}\|\Sigma_0-\widehat{\Sigma}_0\|_\infty+\lambda_{r+1}\|\Sigma_0AA^\top\Sigma_0-\widehat{\Sigma}_0AA^\top\so\|_\infty\leq\frac{\rho}{2},
\end{aligned}
\end{equation}
thus the lefthand side of \eqref{eq:Delta-terms} is bounded by $\dfrac{\rho}{2}\|\Delta\|_1$. 
Furthermore, by Lemma \ref{cur fantope}, we have 
\begin{align*}
& \langle \so^{1/2}\tA(\Lambda_r-\lambda_{r+1}I) \tA^\top\so^{1/2},\so^{1/2}(\widehat{F}-\tF)\so^{1/2} \rangle
\\ 
& = -\langle \so^{1/2}\tA(\Lambda_r-\lambda_{r+1}I) \tA^\top\so^{1/2},\so^{1/2}(\tA\tA^\top-\widehat{F})\so^{1/2} \rangle\\
& \leq -\frac{\lambda_r-\lambda_{r+1}}{2}\|\so^{1/2}(\tA\tA^\top-\widehat{F})\so^{1/2} \|_\mathrm{F}^2.
\end{align*}
Together with \eqref{eq:basic-ineq-4} and \eqref{eq:Delta-terms}, the last display implies that when \eqref{canonical pair condition} holds,
\begin{equation*}
    0\leq \rho(\|\Delta_{SS}\|_1-\|\Delta_{(SS)^c}\|_1)+\frac{\rho}{2}\|\Delta\|_1-\frac{\lambda_r-\lambda_{r+1}}{2}\|\so^{1/2}(\tA\tA^\top-\widehat{F})\so^{1/2} \|_\mathrm{F}^2+\delta\|\so^{1/2}\Delta\so^{1/2}\|_\mathrm{F}.
\end{equation*}
Rearranging terms and multiplying both side by $2$, we obtain
\begin{align*}    (\lambda_r-\lambda_{r+1})\|\widehat{\Sigma}^{1/2}_0\Delta\widehat{\Sigma}^{1/2}_0\|_\mathrm{F}^2&\leq 3\rho\|\Delta_{SS}\|_1-\rho\|\Delta_{(SS)^c}\|_1+2\delta\|\widehat{\Sigma}^{1/2}_0\Delta\widehat{\Sigma}^{1/2}_0\|_\mathrm{F}\\
    &\leq 3\rho\|\Delta_{SS}\|_1+2\delta\|\widehat{\Sigma}^{1/2}_0\Delta\widehat{\Sigma}^{1/2}_0\|_\mathrm{F}.
\end{align*}
This can be view as a quadratic equation, which, by Lemma 2 in \cite{cai2013sparse}, yields
\begin{equation}
      \label{eq:quad-ineq}
    \|\widehat{\Sigma}^{1/2}_0\Delta\widehat{\Sigma}^{1/2}_0\|_\mathrm{F}^2\leq \frac{4\delta^2}{(\lambda_r-\lambda_{r+1})^2}+\frac{6\rho}{\lambda_r-\lambda_{r+1}}\|\Delta_{SS}\|_1.
\end{equation}
Combining the last two displays, we have 
\begin{align*}
    0&\leq 3\rho\|\Delta_{SS}\|_1-\rho\|\Delta_{(SS)^c}\|_1+\frac{\delta^2}{\lambda_r-\lambda_{r+1}}+(\lambda_r-\lambda_{r+1})\|\so^{1/2}\Delta\so^{1/2}\|_\mathrm{F}^2\\
    &\leq  9\rho\|\Delta_{SS}\|_1-\rho\|\Delta_{(SS)^c}\|_1+\frac{5\delta^2}{\lambda_r-\lambda_{r+1}}.
\end{align*}
This can be viewed as a version of generalized cone condition. Finally, using Cauchy--Schwarz inequality, $\|\Delta_{SS}\|_1\leq s \|\Delta_{SS}\|_\mathrm{F}$, and so \eqref{eq:quad-ineq} leads to
\begin{equation}
      \label{eq:Delta-bd-1}
    \|\widehat{\Sigma}^{1/2}_0\Delta\widehat{\Sigma}^{1/2}_0\|_\mathrm{F}^2\leq \frac{4\delta^2}{(\lambda_r-\lambda_{r+1})^2}+\frac{6\rho s}{\lambda_r-\lambda_{r+1}}\|\Delta_{SS}\|_\mathrm{F}.
\end{equation}
and this is the end of the first step.

\paragraph{Step (2)}
Recall that we have established the {generalized cone condition} 
\begin{align}
      \label{eq:generalized-cone}
   \|\Delta_{(SS)^c}\|_1\leq 9\|\Delta_{SS}\|_1+\frac{5\delta^2}{(\lambda_r-\lambda_{r+1})\rho}.
\end{align}
In this step we lower bound $\|\widehat{\Sigma}^{1/2}_0\Delta\widehat{\Sigma}^{1/2}_0\|_\mathrm{F}^2$ by a function of $\|\Delta\|_\mathrm{F}$ on this cone. 

\par Adapting the ``peeling'' argument in \cite{bickel2009simultaneous}, we define the index set $J_1=\{(i_l,j_l)\}_{l=1}^{t}$ in $(S\times S)^c$ correspond to the entries with the $t$ largest absolute values in $\Delta$, and also define $\widetilde{J}=(S\times S)\bigcup J_1$. Next we partition $\widetilde{J}^c$ into disjoint subsets $J_2,...,J_M$ of size $t$ and possibly $|J_M|<t$ such that each $J_m$ is the set of indices corresponding the entries of the $t$ largest absolute values in $\Delta$ outside $\tJ\bigcup \cup_{j=2}^{m-1}J_j$. Then by triangle inequality
\begin{align*}
  \|\widehat{\Sigma}^{1/2}_0\Delta\widehat{\Sigma}^{1/2}_0\|_\mathrm{F} &\geq \|\widehat{\Sigma}^{1/2}_0\Delta_{\tJ}\widehat{\Sigma}^{1/2}_0\|_\mathrm{F}-\sum_{m=2}^M\|\widehat{\Sigma}^{1/2}_0\Delta_{J_m}\widehat{\Sigma}^{1/2}_0\|_\mathrm{F} \\
  &\geq \phi_{\min}^{\widehat{\Sigma}_0}(s+t)\|\Delta_{\tJ}\|_\mathrm{F}-
  \phi_{\max}^{\widehat{\Sigma}_0}(t)\sum_{m=2}^M \|\Delta_{J_m}\|_\mathrm{F}.
\end{align*}
In addition, by our construction of the index sets,
\begin{align*}
    \sum_{m=2}^M \|\Delta_{J_m}\|_\mathrm{F}&\leq \sqrt{t} \sum_{m=2}^M\|\Delta_{J_m}\|_\infty\leq \frac{\sqrt{t}}{t}\sum_{m=2}^M\|\Delta_{J_{m-1}}\|_1
      \leq t^{-1/2}\|\Delta_{(SS)^c}\|_1\\
      & \leq t^{-1/2}(9\|\Delta_{SS}\|_1+\frac{5\delta^2}{\lambda_r-\lambda_{r+1}})\\
    &\leq \frac{9s}{\sqrt{t}}\|\Delta_{\tJ}\|_\mathrm{F}+\frac{5\delta^2}{(\lambda_r-\lambda_{r+1})\rho\sqrt{t}}.
\end{align*}
The second last inequality follows from the generalized cone condition \eqref{eq:generalized-cone}. Hence combining the results above, we have \begin{equation}
\label{eq:Delta-bd-2}
\|\widehat{\Sigma}^{1/2}_0\Delta\widehat{\Sigma}^{1/2}_0\|_\mathrm{F}\geq \kappa_1\|\Delta_{\tJ}\|_\mathrm{F}-\kappa_2\frac{\delta^2}{(\lambda_r-\lambda_{r+1})\rho\sqrt{t}},
\end{equation}
where\begin{equation*}
    \kappa_1= \phi_{\min}^{\widehat{\Sigma}_0}(s+t)-\frac{9s}{\sqrt{t}} \phi_{\max}^{\widehat{\Sigma}_0}(t),
\end{equation*}
\begin{equation*}
    \kappa_2=5\phi_{\max}^{\widehat{\Sigma}_0}(t).
\end{equation*}
Taking $t=c_1s^2$ for $c_1$ sufficiently large, using the same argument as in \cite{gao2017sparse}, under condition \eqref{init: sample size} we can lower bound $\kappa_1$ by some constant $C_1$ and upper bound $\kappa_2$ by some constant $C_2$. 
Combining \eqref{eq:Delta-bd-1} and \eqref{eq:Delta-bd-2}, we obtain
\begin{equation*}
\|\Delta_{\tJ}\|_\mathrm{F}^2\leq C_1\frac{s\rho}{\lambda_r-\lambda_{r+1}}\|\Delta_{\tJ}\|_\mathrm{F}
+C_2\left(\frac{\delta^2}{(\lambda_r-\lambda_{r+1})^2}
+\left[\frac{\delta^2}{\rho(\lambda_r-\lambda_{r+1})\sqrt{t}}\right]^2\right).
\end{equation*}
Solving this equation gives \begin{equation*}
    \|\Delta_{\tJ}\|_\mathrm{F}^2\leq C_3\left(\frac{s^2\rho^2}{(\lambda_r-\lambda_{r+1})^2}+\frac{\delta^2}{(\lambda_r-\lambda_{r+1})^2}+\left(\frac{\delta^2}{\rho(\lambda_r-\lambda_{r+1})\sqrt{t}}\right)^2\right)
\end{equation*}
for some positive constant $C_3$ that is sufficiently large. 
Also we have\begin{equation*}
    \|\Delta_{{\tJ}^c}\|_\mathrm{F}\leq \sum_{m=2}^M\|\Delta_{J_m}\|_\mathrm{F}\leq  \frac{9s}{\sqrt{t}}\|\Delta_{\tJ}\|_\mathrm{F}+\frac{5\delta^2}{(\lambda_r-\lambda_{r+1})\rho\sqrt{t} }.
\end{equation*}
Combining the last two displays and using $t = c_1 s^2$, we obtain that
\begin{equation*}
\|\Delta\|_{\mathrm{F}} 
\leq C_4 \left( \frac{s\rho + \delta}{\lambda_r - \lambda_{r+1}} + 
\frac{\delta^2}{ s \rho  (\lambda_r - \lambda_{r+1}) } \right).
\end{equation*}
Recall that on event $B_3$ we pick $\rho=\gamma\sqrt{\frac{\log p}{n}}$. 
Then on event $B_4$,  we have \begin{equation*}
    \delta\leq C_0\rho\sqrt{t}.
\end{equation*}
 Hence \begin{equation*}
    \|\Delta\|_\mathrm{F}\leq C_5\frac{s\rho}{\lambda_r-\lambda_{r+1}}
      =C_5\gamma\frac{s\sqrt{\log p}}{\sqrt{n}(\lambda_r-\lambda_{r+1})}
\end{equation*}
on event $B_3\cap B_4$.  

\paragraph{Step (3)} The last step follows from $\|AA^\top  -\widehat{F}\|_\mathrm{F}\leq\|\Delta\|_\mathrm{F}+\|\tF-AA^\top  \|_\mathrm{F}$. 
We combine the last display and \eqref{eq:oracle-error} to obtain the desired bounds for each term on the right side. 
This finishes our proof.
\end{proof}

\subsection{Proof of Corollary \ref{cor:init}}
\begin{proof}
Define 
\begin{equation*}
\tau_n^2= \frac{s^2\log p}{n(\lambda_r-\lambda_{r+1})^2}.
\end{equation*} 
We will bound $\dist(\gd_1, V)$ by a constant multiple of $\tau_n$ on event $B_2\cap B_3\cap B_4$. 

Notice that since $A_0A_0^\top$ is the best rank $r$ approximation, and $AA^\top$ also has rank $r$.
Thus, we obtain that $\|A_0A_0^\top -\widehat{F} \|_\mathrm{F}\leq \|AA^\top -\widehat{F} \|_\mathrm{F}$. 
Triangle inequality further leads to
\begin{equation*}
\|A_0A_0^\top -AA^\top \|_\mathrm{F}\leq \|A_0A_0^\top -\widehat{F} \|_\mathrm{F}+ \|AA^\top -\widehat{F} \|_\mathrm{F}\leq 2 \|AA^\top -\widehat{F} \|_\mathrm{F}.
\end{equation*} 
Together with Lemma \ref{sub dis lemma}, the last display implies $\text{dist}(A,A_0)\leq C_0\tau_n$ for some positive constant $C_0$.  


By definition $\widehat{A}_0 =HT(A_0,\hs)$, $\gd_1$ and $\init $ are both $s'$ sparse. 
Moreover, following the lines of the proof of Proposition \ref{prop: ht} and Corollary \ref{cor: normalized estimate}, we have 
\begin{equation*}
\mathrm{dist}(\init ,A)\leq C_1\tau_n
\end{equation*}
for some constant $C_1 > 0$. 
Let $P$ be the orthogonal matrix such that 
$\|\init P-A\|_\mathrm{F} = \mathrm{dist}(\init , A)$.
Then we have
\begin{align*}
\gd_1P
&=\init (I+{\init ^\top  \ssecond\init }/{\lambda})^{1/2}P
=\init  P P^\top (I+{\init ^\top  \ssecond\init }/{\lambda})^{1/2}P\\
& =\init P(I+ {P^\top  \init ^\top \ssecond\init P}/{\lambda})^{1/2}.
\end{align*}
Since \begin{equation*}
\mathrm{dist}(\gd_1,\scale)\leq \|\gd_1P-\scale\|_\mathrm{F}=
\left\|\init P(I+{P^\top  \init ^\top\ssecond\init P}/{\lambda})^{1/2}-A(I+{\Lambda_r}/{\lambda})^{1/2}
\right\|_\mathrm{F},
\end{equation*}
we turn to bound the rightmost side. 
To this end, triangle inequality gives
\begin{align*}
& \|\init P(I+{P^\top  \init^\top \ssecond\init P}/{\lambda})^{1/2}-A(I+{\Lambda_r}/{\lambda})^{1/2}\|_\mathrm{F}\\
& \leq \underbrace{\|\init P(I+{P^\top  \init^\top \ssecond\init P}/{\lambda})^{1/2}-\init P(I+{\Lambda_r}/{\lambda})^{1/2}\|_\mathrm{F}}_{\text{Term I}}
+\underbrace{\|(\init P-A)(I+{\Lambda_r}/{\lambda})^{1/2}\|_\mathrm{F}}_{\text{Term II}}.
\end{align*}
We now bound each term separately. 
For Term II, we have
\begin{equation*}
\|(\init P-A)(I+{\Lambda_r}/{\lambda})^{1/2}\|_\mathrm{F}\leq\|\init P-A\|_\mathrm{F}\|(I+{\Lambda_r}/{\lambda})^{1/2}\|_\op\leq C_2\tau_n.
\end{equation*}
The last inequality holds since $\lambda = \lambda_1/c$.
To deal with Term I, we define $\Delta=\init P-A$ and notice that 
\begin{align*}
& \|\init P(I+{P^\top  \init ^\top\ssecond\init P}/{\lambda})^{1/2}
-\init P(I+{\Lambda_r}/{\lambda})^{1/2}\|_\mathrm{F}\\
& \qquad \qquad
\leq \|\init P\|_\op\|(I+{P^\top  \init ^\top\ssecond\init P}/{\lambda})^{1/2}-(I+{\Lambda_r}/{\lambda})^{1/2}\|_\mathrm{F}.
\end{align*}
By matrix root perturbation bound (e.g., Lemma 2 in \cite{gao2017sparse}), we have 
\begin{align*}
& \|(I+{P^\top  \init ^\top\ssecond\init P}/{\lambda})^{1/2}
-(I+{\Lambda_r}/{\lambda})^{1/2}\|_\mathrm{F} \\
& \qquad \leq C_3\|P^\top  \init ^\top\ssecond\init P-\Lambda_r\|_\mathrm{F}\\
& \qquad \leq C_3\|(A+\Delta)^\top  \ssecond(A+\Delta)-A^\top\Sigma A\|_\mathrm{F}\\
& \qquad =C_3(\|A^\top(\ssecond-\Sigma) A\|_\mathrm{F}   +2\|\Delta^\top  \ssecond A\|_\mathrm{F}+\|\Delta^\top  \ssecond\Delta\|_\mathrm{F})
\end{align*}
for some positive constant $C_3$. 
Since $\init $ is $s'$ sparse and $A$ is $s$ sparse, 
$\Delta = \init  P - A$ is $s+s'$ sparse. 
By definition of $P$, $\|\Delta\|_\mathrm{F} = \mathrm{dist}(\init , A)\leq C_1\tau_n$. 
By a similar argument to the proof of Corollary \ref{cor: normalized estimate}, on event $B_2$, we have
\begin{equation*}
\|(I+{P^\top  \init ^\top\ssecond\init P}/{\lambda})^{1/2}-(I+{\Lambda_r}/{\lambda})^{1/2}\|_\mathrm{F}\leq C_4\sqrt{\frac{rs\log p}{n}}\leq C_4\tau_n.
\end{equation*} 
Consequently, Term I is also dominated by a constant multiple of $\tau_n$. 
Thus, we conclude that on the intersection of 
$B_2\cap B_3\cap B_4$, 
\begin{equation*}
\mathrm{dist}(\gd_1, \scale)\leq C_5\tau_n
\end{equation*}
for some positive constant $C_5$.  
By union bound, the intersection of the three events holds with probability at least $1-\exp(-C'(s+\log (ep/s)))$, uniformly over $\mathcal{P}_n$.
This completes the proof.
\end{proof}

\subsection{Proof of Lemma \ref{lemma: oracle}}
\label{sec:proof-lemma-oracle}

\begin{proof}
In this proof, we work on event $B_2$ defined in \eqref{eq:event-B2}, which occurs with probability at least $1-\exp(-C's'\log(ep/s'))$ for some constant $C'>0$, uniformly over $\mathcal{P}_n$.
The proof relies on Theorem 3.1 of \cite{sun1983perturbation}, 
a matrix perturbation bound for generalized eigenspaces. 
{Fix any $\cI$ such that $S\subset \cI$ and $|\cI|\leq 2s'+s$}. 
We shall apply the theorem on matrix pair $(\Sigma_{\cI\cI},\Sigma_{0,\cI\cI})$ and its perturbation $(\widehat{\Sigma}_{\cI\cI},\widehat{\Sigma}_{0,\cI\cI})$. 

To this end,  for the fixed set $\cI$, define
\begin{equation*}
\delta = \min_{\lambda,\lambda'}{\frac{\lambda-\lambda'}{\sqrt{(1+\lambda^2)(1+\lambda'^2)}}}\quad\text{for $\lambda\in\{\lambda_1,...\lambda_r\}$ and $\lambda'\in\{\widehat{\lambda}_{r+1},...\widehat{\lambda}_{|\cI|}\}$},
\end{equation*}
where $\lambda_i$ denotes the $i$th generalized eigenvalue of matrix pair $(\Sigma_{\cI\cI}, \Sigma_{0,\cI\cI})$ and $\widehat{\lambda}_i$ denotes the $i$th sample generalized eigenvalue of matrix pair $(\widehat{\Sigma}_{\cI\cI},\widehat{\Sigma}_{0,\cI\cI})$. 
Define event 
\begin{equation*}
 \widetilde{B}_2 = 
\left\{\widehat{\lambda}_{r+1}\leq 2\lambda_{r+1}\wedge\frac{\lambda_r+\lambda_{r+1}}{2}  
\right\}.
\end{equation*}
On event $\widetilde{B}_2$,
since generalized eigenvalues are in decreasing order, 
we have 
\begin{equation*}
      \delta=\frac{\lambda_r-\widehat{\lambda}_{r+1}}{\sqrt{(1+\lambda_1^2)(1+\widehat{\lambda}_{r+1}^2)}}\geq \frac{1}{2}\frac{\lambda_r-\lambda_{r+1}}{\sqrt{(1+\lambda_1^2)(1+\widehat{\lambda}_{r+1}^2)}}.
\end{equation*}
Moreover, since  $\widehat{\lambda}_{r+1}^2\leq {4\lambda_{r+1}^2}$ on $\widetilde{B}_2$,
we further obtain that 
\begin{equation*}
\delta\geq \frac{1}{4}\frac{\lambda_r-\lambda_{r+1}}{\sqrt{1+\lambda_1^2}\sqrt{{1+\lambda_{r+1}^2}}}.
\end{equation*}
On the other hand,  $\widetilde{B}_2$ can be equivalently defined as
\begin{equation*}
\widetilde{B}_2
= \left\{\widehat{\lambda}_{r+1}-\lambda_{r+1}\leq \lambda_{r+1}\wedge\frac{\lambda_r-\lambda_{r+1}}{2}  \right\}.
\end{equation*} 
Since $\lambda_i$ is also the $i$th eigenvalue of $\Sigma_{0,\cI\cI}^{-1/2}\Sigma_{\cI\cI}\Sigma_{0,\cI\cI}^{-1/2}$ and $\widehat\lambda_i$ the $i$th eigenvalue of $\widehat\Sigma_{0,\cI\cI}^{-1/2}\widehat\Sigma_{\cI\cI} \widehat\Sigma_{0,\cI\cI}^{-1/2}$,
Weyl's inequality then implies that on event $B_2$,
\begin{align*}
|\widehat{\lambda}_{r+1}-\lambda_{r+1}|
& \leq  \|\Sigma_{0,\cI\cI}^{-1/2}\Sigma_{\cI\cI}\Sigma_{0,\cI\cI}^{-1/2}-\widehat{\Sigma}_{0,\cI\cI}^{-1/2}\widehat{\Sigma}_{\cI\cI}\widehat{\Sigma}_{0,\cI\cI}^{-1/2}\|_\op\\
& \leq \|(\Sigma_{0,\cI\cI}^{-1/2}-\widehat{\Sigma}_{0,\cI\cI}^{-1/2})\Sigma_{\cI\cI}\Sigma_{0,\cI\cI}^{-1/2}\|_\op+\|\widehat{\Sigma}_{0,\cI\cI}^{-1/2}(\Sigma_{\cI\cI}-\widehat{\Sigma}_{\cI\cI})\Sigma_{0,\cI\cI}^{-1/2}\|_\op\\
&~~~ +\|\widehat{\Sigma}_{0,\cI\cI}^{-1/2}\widehat{\Sigma}_{\cI\cI}(\Sigma_{0,\cI\cI}^{-1/2}-\widehat{\Sigma}_{0,\cI\cI}^{-1/2})\|_\op\\
& \leq \|\Sigma_{0,\cI\cI}^{-1/2}-\widehat{\Sigma}_{0,\cI\cI}^{-1/2}\|_\op\|\Sigma_{\cI\cI}\Sigma_{0,\cI\cI}^{-1/2}\|_\op+\|\widehat{\Sigma}_{0,\cI\cI}^{-1/2}\|_\op\|\Sigma_{\cI\cI}-\widehat{\Sigma}_{\cI\cI}\|_\op\|\Sigma_{0,\cI\cI}^{-1/2}\|_\op\\
&~~~ +\|\widehat{\Sigma}_{0,\cI\cI}^{-1/2}\widehat{\Sigma}_{\cI\cI}\|_\op\|\Sigma_{0,\cI\cI}^{-1/2}-\widehat{\Sigma}_{0,\cI\cI}^{-1/2}\|_\op \\
& \leq  C_0\sqrt{\frac{|\cI|\log p}{n}}
\end{align*}
for some constant $C_0$ depending on $\rev{\nu}$.
Here the last inequality is due to the matrix root bound from Lemma 2 in \cite{gao2015minimax}.
Under condition \eqref{sample size condition}, by the choice of $s'$ in \eqref{eq:s-prime}, the last display implies $|\widehat{\lambda}_{r+1}-\lambda_{r+1}|\leq (\frac{\lambda_r-\lambda_{r+1}}{2}\wedge\lambda_{r+1})$.
Hence $\widetilde{B}_2$ holds on event $B_2$ under condition \eqref{sample size condition}. 

\par Furthermore, we have\begin{equation*}
\gamma(\Sigma_{\cI\cI},\Sigma_{0,\cI\cI})=\min_{\|x\|=1}\sqrt{(x^\top\Sigma_{\cI\cI} x)^2+(x^\top\Sigma_{0,\cI\cI} x)^2}\geq \frac{1}{\rev{\nu}},
\end{equation*}
and \begin{equation*}
\gamma(\widehat{\Sigma}_{\cI\cI},\widehat{\Sigma}_{0,\cI\cI})=\min_{\|x\|=1}\sqrt{(x^\top\widehat{\Sigma}_{\cI\cI} x)^2+(x^\top\widehat{\Sigma}_{0,\cI\cI} x)^2}\geq \frac{1}{2\rev{\nu}},
\end{equation*}
on event $B_2$. 
Now we apply Theorem 3.1 of \cite{sun1983perturbation} for perturbation on restricted covariance matrices and obtain that 
\begin{align*}
\dist(A,\ha(\cI))&\leq C_1\frac{\sqrt{\|\Sigma_{\cI\cI}^2+\Sigma_{0,\cI\cI}^2\|_\op}}{\gamma(\Sigma_{\cI\cI},\Sigma_{0,\cI\cI})\gamma(\widehat{\Sigma}_{\cI\cI},\widehat{\Sigma}_{0,\cI\cI})}\frac{\sqrt{\|(\widehat{\Sigma}_{\cI\cI}-\Sigma_{\cI\cI})A_{\cI*}\|_\mathrm{F}^2+\|(\widehat{\Sigma}_{0,\cI\cI}-\Sigma_{0,\cI\cI})A_{\cI*}\|_\mathrm{F}^2}}{\delta}\\
&\leq C_2 \rev{\nu^3}\frac{\sqrt{1+\lambda_1^2}\sqrt{{1+\lambda_{r+1}^2}}}{\lambda_r-\lambda_{r+1}}\sqrt{\|A_{\cI*}\|_\mathrm{F}^2(\|\widehat{\Sigma}_{\cI\cI}-\Sigma_{\cI\cI}\|_\op^2+\|\widehat{\Sigma}_{0,\cI\cI}-\Sigma_{0,\cI\cI}\|_\op^2)}\\
&\leq C_3 \rev{\nu^3}\frac{\sqrt{1+\lambda_1^2}\sqrt{{1+\lambda_{r+1}^2}}}{\lambda_r-\lambda_{r+1}}\sqrt{r\frac{|\cI|\log p}{n}}\\
&\leq C_4 \sqrt{r}\frac{\sqrt{1+\lambda_1^2}\sqrt{{1+\lambda_{r+1}^2}}}{\lambda_r-\lambda_{r+1}}\sqrt{\frac{|\cI|\log p}{n}},
\end{align*}
for some positive constant $C_4$ 
on event $B_2$. 
This finishes the proof of Lemma \ref{lemma: oracle}. 
\end{proof}

\subsection{Proofs of Lemma \ref{prob B3} and Lemma \ref{prob B4}}
\begin{proof}
Note that $\lambda_{r+1} < 1$. Hence, it suffices to show that each of the four infinity norm terms on the left side in the definition of event $B_3$ in \eqref{eq:event-B3} is upper bounded by a constant multiple of $\sqrt{\log p/n}$ with probability at least $1 - p^{-C'}$ for some positive constant $C'$, {
  uniformly over $\mathcal{P}_n$}.
This can be achieved by following the lines of the proof of Lemma 6.4 in \cite{gao2017sparse} and so we omit the details.
\end{proof}

\begin{proof}
We first bound the operator norm of each term on the left side of \eqref{eq:event-B4} by a constant multiple of ${\sqrt{\frac{1}{n}(s + \log(ep/s))}}$, and the result on Frobenius follows directly since each term is of rank $r$. 

As in the proof of Lemma 6.1 of \cite{gao2017sparse}, we have\begin{equation*}
     \|\Sigma_0^{1/2}(\tA-A)\|_{\text{op}}\leq \|\Sigma_0^{1/2} A\|_{\op}\|(A^\top  \widehat{\Sigma}_0A)^{1/2}-I\|_{\op}\|(A^\top  \widehat{\Sigma}_0A)^{-1/2}\|_{\op},
\end{equation*}
\begin{equation*}
    \|\tlambda-\Lambda_r\|_{\op}\leq \|(A^\top  \widehat{\Sigma}_0A)^{1/2}-I\|_{\op}\|\Lambda_r(A^\top  \widehat{\Sigma}_0A)^{1/2}\|_{\op}+\|\Lambda_r\|_{\op}\|(A^\top  \widehat{\Sigma}_0A)^{1/2}-I\|_{\op}.
\end{equation*}
So it remains to bound $\|(A^\top  \so A)^{1/2}-I\|_{\op}$. 
By Lemma 2 in \cite{gao2015minimax}, it suffices to bound $\|A^\top  \widehat{\Sigma}_0A-I\|_{\op}$. 
Note that 
\begin{align*}
    \|A^\top  \widehat{\Sigma}_0A-I\|_{\op} &= \|A^\top  (\widehat{\Sigma}_0-\Sigma_0)A\|_\op\\
    &=\|A_{S*}^\top  (\widehat{\Sigma}_{0,SS}-\Sigma_{0,SS})A_{S*}\|_{\op}\\
    &= \sup_{\|v\|=1}(A_{S*}v)^\top  (\widehat{\Sigma}_{0,SS}-\Sigma_{0,SS})(A_{S*}v)\\
    &\leq \|\Sigma_{0,SS}^{1/2}A_{S*}\|_{\op}^2\| \Sigma_{0,SS}^{-1/2} \widehat{\Sigma}_{0,SS} \Sigma_{0,SS}^{-1/2}-I\|_{\op}\\
    &\leq \| \Sigma_{0,SS}^{-1/2} \widehat{\Sigma}_{0,SS} \Sigma_{0,SS}^{-1/2}-I\|_{\op}\leq C\| \widehat{\Sigma}_{0,SS}-\Sigma_{0,SS}\|_{\op}.
\end{align*}
By Lemma 3 of \cite{gao2015minimax}, we then have \begin{equation*}
    \|A^\top  \widehat{\Sigma}_0A-I\|_{\op}\leq C_1\sqrt{\frac{1}{n}\left(s+\log\frac{ep}{s}\right)},
\end{equation*}
with probability at least $1-\exp(-C'(s+\log (ep/s))$ { for some constant $C'>0$, uniformly over $\mathcal{P}_n$}.
This completes the proof.
\end{proof}

\section{Proof of Lower Bound}
\label{sec:proof_lower_bound}

We first present a lemma on Kullback-Leibler divergence between data distributions from a special kind of covariance matrix. The Lemma can be viewed as a general case of Lemma 1 in \cite{gao2015supplement}.
\begin{lemma}
\label{lemma:KL-latent-variable-model}
For $t=1, 2$, define $\Sigma^{(t)}$ to be a block matrix whose $(i, j)$th block is given by $\lambda U_{\dc{i}}^{(t)}U_{\dc{j}}^{{(t)}^\top}$ for $i\neq j$, where $U_{\dc{i}}^{(t)}\in O(p_i, r)$ and whose $i$th diagonal block is given by $I_{p_i}.$ Further let $P^{(t)}$ denote the distribution of a random i.i.d sample of size $n$ from the $N_p(0, \Sigma^{(t)})$ distribution where $p=\sum_{i=1}^k p_i$. Then \begin{equation*}
    D(P^{(1)}||P^{(2)}) = \frac{\lambda^2 n}{4\left(-(k-1)\lambda^2+(k-2)\lambda+1\right)}\sum_{i\neq j}\left\|U_{\dc{i}}^{(1)}U_{\dc{j}}^{{(1)}^\top}-U_{\dc{i}}^{(2)}U_{\dc{j}}^{{(2)}^\top}\right\|_\mathrm{F}^2.
\end{equation*}
\end{lemma}
\begin{proof}
The first step is to find the eigenvalues of $\Sigma^{(t)}$. To this end, define \begin{align*}
U^{(t)}&=\left[U_{\dc{1}}^{{(t)}^\top}, U_{\dc{2}}^{{(t)}^\top},\dots, U_{\dc{k}}^{{(t)}^\top}\right]^\top,\\ K_i^{(t)} &= \left[-U_{\dc{1}}^{{(t)}^\top}, -U_{\dc{2}}^{{(t)}^\top}\dots, -U_{\dc{i-1}}^{{(t)}^\top}, (k-1)U_{\dc{i}}^{{(t)}^\top}, -U_{\dc{i+1}}^{{(t)}^\top},\dots, -U_{\dc{k}}^{{(t)}^\top}\right]^\top.
\end{align*}
That is, $K_i^{(t)}\in\mathbb{R}^{p\times r}$ is a block matrix with $j$th blocks being $-U_{\dc{j}}^{(t)}$ for $j\neq i$ and the $i$th block being $(k-1)U_{\dc{i}}^{(t)}$.
It is straightforward to verify that $\Sigma^{(t)}$ yields the following decomposition\begin{equation*}
    \Sigma^{(t)} = I +\frac{(k-1)\lambda}{k}U^{(t)}U^{{(t)}^\top}-\frac{\lambda}{k^2}\sum_{i=1}^k K_i^{(t)}K_i^{{(t)}^\top}.
\end{equation*}
We then deduce that the eigen-structure of $\Sigma^{(t)}$ is given by $r$ eigenvalues equal to $1+\lambda(k-1)$, $(k-1)r$ eigenvalues equal to $1-\lambda$ and $p-kr$ eigenvalues equal to 1. Furthermore, we notice that\begin{equation*}
    \Sigma^{(t)}U^{(t)} = (1+(k-1)\lambda) U^{(t)},
\end{equation*}
hence the leading $r$ dimensional eigenspace is given by the matrix $U^{(t)}$. This in particular implies that\begin{equation*}
    \det \Sigma^{(1)} =  \det \Sigma^{(2)}.
\end{equation*}
Now the KL divergence is given by\begin{align*}
    D(P^{(1)}||P^{(2)}) &= \frac{n}{2}\left(\Tr\left( (\Sigma^{(2)})^{-1} \Sigma^{(1)}\right)-\sum_{i=1}^k p_i-\log\det\left((\Sigma^{(2)})^{-1} \Sigma^{(1)}\right)\right)\\
    &= \frac{n}{2} \left(\Tr\left( (\Sigma^{(2)})^{-1} \Sigma^{(1)}\right) - p\right)\\
    &= \frac{n}{2} \left(\Tr\left( (\Sigma^{(2)})^{-1} \left(\Sigma^{(1)} - \Sigma^{(2)}\right)\right)\right).
\end{align*}
We now calculate the inverse of the block matrix $\Sigma^{(2)}$. To this end, we first guess the form of the inverse and then determine the coefficients. Specifically, we try to solve for the inverse with the following form
\begin{equation*}
    (\Sigma^{(2)})^{-1} = \begin{bmatrix}
I + aU_{\dc{1}}^{(2)}U_{\dc{1}}^{{(2)}^\top} & bU_{\dc{1}}^{(2)}U_{\dc{2}}^{{(2)}^\top} & \dots &\dots & bU_{\dc{1}}^{(2)}U_{\dc{k}}^{{(2)}^\top}\\
bU_{\dc{2}}^{(2)}U_{\dc{1}}^{{(2)}^\top} & I + aU_{\dc{2}}^{(2)}U_{\dc{2}}^{{(2)}^\top}&bU_{\dc{2}}^{(2)}U_{\dc{3}}^{{(2)}^\top} &\dots &bU_{\dc{2}}^{(2)}U_{\dc{k}}^{{(2)}^\top} \\
\dots&\dots&\dots&\dots&\dots\\
bU_{\dc{k}}^{(2)}U_{\dc{1}}^{{(2)}^\top}&\dots&\dots&\dots& I + aU_{\dc{k}}^{(2)}U_{\dc{k}}^{{(2)}^\top}

\end{bmatrix}
\end{equation*}
with the $(i,j)$th block being $bU_{\dc{i}}^{(2)}U_{\dc{j}}^{{(2)}^\top}$ for $i\neq j$ and $i$th diagonal block being $I + aU_{\dc{i}}^{(2)}U_{\dc{i}}^{{(2)}^\top}$. To determine the values of $(a,b)$, we solve for the equation $(\Sigma^{(2)})^{-1}\Sigma^{(2)}=I_p$. This requires two conditions on the matrix product: the off-diagonal block to be 0 and the diagonal block to be $I$. For the $i$th diagonal block, we require\begin{equation*}
    I+(k-1)b\lambda U_{\dc{i}}^{(2)}(U_{\dc{i}}^{(2)})^{\top}+aU_{\dc{i}}^{(2)}(U_{\dc{i}}^{(2)})^{\top}=I,
\end{equation*}
and for the $(i,j)$th off-diagonal block, we require\begin{equation*}
    \lambda b (k-2) U_{\dc{i}}^{(2)}(U_{\dc{j}}^{(2)})^{\top} + bU_{\dc{i}}^{(2)}(U_{\dc{j}}^{(2)})^{\top} +\lambda U_{\dc{i}}^{(2)}(U_{\dc{j}}^{(2)})^{\top} +a\lambda U_{\dc{i}}^{(2)}(U_{\dc{j}}^{(2)})^{\top}=0.
\end{equation*}
Solving the above equations, we have \begin{align*}
    a &= -(k-1)b\lambda,\\
    b &= \frac{\lambda}{(k-1)\lambda^2-\lambda(k-2)-1}.
\end{align*}
As a result, we can simplify the expression for KL divergence as follows
\begin{align*}
    D(P^{(1)}||P^{(2)}) &=\frac{n}{2} \left(\Tr\left( (\Sigma^{(2)})^{-1} \left(\Sigma^{(1)} - \Sigma^{(2)}\right)\right)\right)\\
    &= \frac{\lambda n}{2}\left(\Tr\left(\sum_{i=1}^k \sum_{j\neq i}bU_{\dc{i}}^{(2)}(U_{\dc{j}}^{(2)})^{\top}\left(U_{\dc{j}}^{(1)}(U_{\dc{i}}^{(1)})^{\top}-U_{\dc{j}}^{(2)}(U_{\dc{i}}^{(2)})^{\top}\right) \right)\right)\\
    &= -\frac{b\lambda n}{2}\left(\sum_{i, j:i\neq j}\Tr\left(I-U_{\dc{i}}^{(2)}(U_{\dc{j}}^{(2)})^{\top}U_{\dc{j}}^{(1)}(U_{\dc{i}}^{(1)})^{\top}\right)\right)\\
    &= -\frac{b\lambda n}{4} \sum_{i, j:i\neq j}\left\|U_{\dc{i}}^{(1)}(U_{\dc{j}}^{(1)})^{\top}-U_{\dc{i}}^{(2)}(U_{\dc{j}}^{(2)})^{\top} \right\|_\mathrm{F}^2\\
    &= \frac{\lambda^2 n}{4\left(-(k-1)\lambda^2+(k-2)\lambda+1\right)}\sum_{i, j:i\neq j}\left\|U_{\dc{i}}^{(1)}(U_{\dc{j}}^{(1)})^{\top}-U_{\dc{i}}^{(2)}(U_{\dc{j}}^{(2)})^{\top}\right\|_\mathrm{F}^2.\\
\end{align*}
This finishes our proof for the Lemma. Note that when $k=2$, we recover Lemma 1 in \cite{gao2015supplement} as a special case.

\end{proof}

\subsection{Proof of Theorem \ref{thm:lower_bound_gca}}
\begin{proof}
The main body of the proof is adapted from \cite{gao2015minimax}. The main tool for our proof is Fano's Lemma. For the sake of completeness, we provide the following version of Fano's Lemma  from \cite{yu1997festschrift}.
\begin{lemma}
\label{lemma:fano}
Let $(\Theta,\rho)$ be a metric space and $\{P_\theta:\theta\in\Theta\}$ a collection of probability measures. For any totally bounded $T\in\Theta$, denoted by $\mathcal{M}(T,\rho,\epsilon)$ the $\epsilon$-packing number of $T$ with respect to $\rho$, i.e., the maximal number of points in $T$ whose pairwise maximum distance in $\rho$ is at least $\epsilon$. Define the Kullback-Leibler diameter $T$ by\begin{equation*}
    d_{\kl}(T) \triangleq \sup_{\theta, \theta'\in T} D(P_\theta||D_{\theta'}).
\end{equation*}
Then \begin{equation*}
    \inf_\theta \sup_{\theta\in\Theta} \E_\theta[\rho^2(\hat{\theta}(X), \theta)] \geq \sup_{T\in\Theta} \sup_{\epsilon>0} \frac{\epsilon^2}{4}\left(1-\frac{d_\kl(T)+\log 2}{\log \mathcal{M}(T,\rho,\epsilon)}\right)
\end{equation*}
\end{lemma}
The proof is compose of two steps corresponding to different terms in the lower bound.  Throughout the proof, we define $\lambda = \lambda_r-\lambda_{r+1}$ for simplicity.

\paragraph{Step I.} We first establish the term involving $r\sum_{i=1}^k s_i$. To this end, let $U_{\dc{i}0}=\begin{bmatrix}
I_r\\0
\end{bmatrix}\in O(p_i, r)$ for each $i=1, 2, \dots, k$. For some $\epsilon_0\in(0, \sqrt{r\wedge (s_1-r)}]$ to be specified later, let \begin{equation*}
    B(\epsilon_0)=\{U_{\dc{1}}\in O(p_1,r): \supp(U_{\dc{1}})\subset [s_1], \|U_{\dc{1}}-U_{\dc{1}0}\|_\mathrm{F}\leq \epsilon_0\}
\end{equation*}
and \begin{equation*}
    T_0=\left\{\Sigma = \begin{pmatrix}
    I_{p_1} & \lambda U_{\dc{1}}U_{\dc{2}0}^\top &\lambda U_{\dc{1}}U_{\dc{3}0}^\top \dots & \lambda U_{\dc{1}}U_{\dc{k}0}^\top\\
  \lambda  U_{\dc{2}0}U_{\dc{1}}^\top &I_{p_2} &\dots&\lambda U_{\dc{2}0}U_{\dc{k}0}^\top\\
    \dots & \dots &\dots &\dots\\
  \lambda  U_{\dc{k}0}U_{\dc{1}}^\top &  \lambda U_{\dc{k}0}U_{\dc{2}0}^\top&\dots&
    I_{p_k}
    \end{pmatrix}: U_{\dc{1}}\in B(\epsilon_0)\right\},
\end{equation*}
with $U =\left[U_{\dc{1}}^\top, U_{\dc{2}0}^\top,\dots, U_{\dc{k}0}^\top\right]^\top$. Since our target of estimation is $A$ instead of $U$, we first establish the relationship between $U$ and $A$ under $T_0$. Note that under the construction of $\Sigma$, we have $\Sigma_0 = I_p$ so the generalized eigenspace coincides with eigenspace. From Lemma \ref{lemma:KL-latent-variable-model}, we deduce that the leading $r$ dimensional eigenspace of $\Sigma$ is given by $\mathrm{span}(U)$. From the normalization constraint such that $A^\top A = I$, we conclude that 
    $\dist(A, \frac{1}{\sqrt{k}}U)=0$. 
From here on, we first derive the minimax lower bound for the estimation of $U$ and the lower bound for estimating $A$ under the matrix distance follows immediately by scaling.  \par The above analysis also implies that $T_0\subset\mathcal{F}$, where $\mathcal{F}$ is our original parameter space. By Lemma \ref{lemma:KL-latent-variable-model}, 
\begin{align*}
    d_\kl(T_0) &= \sup_{U^{(1)}_{\dc{1}},U^{(2)}_{\dc{1}}\in B(\epsilon_0)} D(P^{(1)}||P^{(2)})\\
    &= \sup_{U^{(1)}_{\dc{1}},U^{(2)}_{\dc{1}}\in B(\epsilon_0)}\frac{\lambda^2 n}{4\left(-(k-1)\lambda^2+(k-2)\lambda+1\right)}\sum_{i\neq j}\left\|U_{\dc{i}}^{(1)}U_{\dc{j}}^{{(1)}^\top}-U_{\dc{i}}^{(2)}U_{\dc{j}}^{{(2)}^\top}\right\|_\mathrm{F}^2\\
    &= \sup_{U^{(1)}_{\dc{1}},U^{(2)}_{\dc{1}}\in B(\epsilon_0)} \frac{\lambda^2 n}{4\left(-(k-1)\lambda^2+(k-2)\lambda+1\right)} \sum_{j=2}^k 2 \left\|U_{\dc{1}}^{(1)}U_{\dc{j}0}^{{(1)}^\top}-U_{\dc{1}}^{(2)}U_{\dc{j}0}^{{(2)}^\top} \right\|_\mathrm{F}^2\\
    &= \sup_{U^{(1)}_{\dc{1}},U^{(2)}_{\dc{1}}\in B(\epsilon_0)} \frac{\lambda^2 n}{2\left(-(k-1)\lambda^2+(k-2)\lambda+1\right)} \sum_{j=2}^k  \left\|U_{\dc{1}}^{(1)}-U_{\dc{1}}^{(2)} \right\|_\mathrm{F}^2\\
    &= \frac{2\lambda^2 n(k-1)\epsilon_0^2}{-(k-1)\lambda^2+(k-2)\lambda+1}.
\end{align*}
Here, the second to last equality follows from the definition of $B(\epsilon_0)$.
\par We now establish a lower bound for the packing number of $T_0$. For some $\alpha\in(0, 1)$ which shall be specified later, we define $\{\tilU_{\dc{1}}(1), \tilU_{\dc{1}}(2), \dots, \tilU_{\dc{1}}(N)\} \subset O(p_1,r)$ to be a maximal set such that $\supp(\tilU_{\dc{1}}(i))\subset[s_1]$ and for $\forall i\neq j$,\begin{equation*}
    \|\tilU_{\dc{1}}(i)\tilU_{\dc{1}}^\top(i) - U_{\dc{1}0}U_{\dc{1}0}^\top\|_\mathrm{F}\leq \epsilon_0, \quad  \|\tilU_{\dc{1}}(i)\tilU_{\dc{1}}^\top(i) - \tilU_{\dc{1}}(j)\tilU_{\dc{1}}^\top(j)\|_\mathrm{F}\geq \sqrt{2}\alpha\epsilon_0.
\end{equation*}
Then by Lemma 1 in \cite{cai2013sparse}, for some absolute constant $C>1$,\begin{equation*}
    N\geq \left(\frac{1}{C\alpha}\right)^{r(s_1-r)}.
\end{equation*} 
For each $\tilU_{\dc{1}}(i)$, we define $\barU_{\dc{1}}(i)$ to be the matrix such that \begin{equation*}
    \|U_{\dc{1}0} - \barU_{\dc{1}}(i)\|_\mathrm{F}^2 =  \dist^2\left(\tilU_{\dc{1}}(i), U_{\dc{1}0}\right).
\end{equation*}
Then for any $i$, we deduce that $\barU_{\dc{1}}(i)\in O(p_1,r)$, $\supp(\barU_{\dc{1}}(i))\subset [s_1]$ and $\barU_{\dc{1}}(i)\barU_{\dc{1}}(i)^\top =\tilU_{\dc{1}}(i)\tilU_{\dc{1}}(i)^\top $.
In addition, Lemma 6.6 in \cite{gao2017sparse} implies that \begin{equation*}
   \|\barU_{\dc{1}}(i)- U_{\dc{1}0}\|_\mathrm{F} = \dist\left(\tilU_{\dc{1}}(i), U_{\dc{1}0}\right) = \frac{1}{\sqrt{2}}\|\tilU_{\dc{1}}(i)\tilU_{\dc{1}}^\top(i) - U_{\dc{1}0}U_{\dc{1}0}^\top\|_\mathrm{F}\leq \epsilon_0
\end{equation*}
hence $\barU_{\dc{1}}(i)\in B(\epsilon_0)$. On the other hand, note that from Lemma 6.6 in \cite{gao2017sparse},
\begin{align*}
    \dist\left(\barU_{\dc{1}}(i), \barU_{\dc{1}}(j)\right)&= \frac{1}{\sqrt{2}}\|\barU_{\dc{1}}(i)\barU_{\dc{1}}^\top(i) - \barU_{\dc{1}}(j)\barU_{\dc{1}}^\top(j)\|_\mathrm{F}\\&= \frac{1}{\sqrt{2}}\|\tilU_{\dc{1}}(i)\tilU_{\dc{1}}^\top(i) - \tilU_{\dc{1}}(j)\tilU_{\dc{1}}^\top(j)\|_\mathrm{F}\geq \alpha\epsilon_0.
\end{align*}
Define the metric to be $\rho(\Sigma^{(1)}, \Sigma^{(2)}) = \dist\left(U_{\dc{1}}^{(1)}, U_{\dc{1}}^{(2)}\right)$. The above argument implies that for $\epsilon = \alpha \epsilon_0$,\begin{equation*}
    \log \mathcal{M}(T_0, \rho, \epsilon)\geq r(s_1-r)\log \frac{1}{C\alpha}.
\end{equation*}
Setting \begin{equation*}
    \epsilon = c_0\left[\sqrt{r\wedge (s_1-r)}\wedge \sqrt{\frac{-(k-1)\lambda^2+(k-2)\lambda+1}{n(k-1)\lambda^2}r(s_1-r)}\right]
\end{equation*}
for sufficiently small constants $c_0, \alpha$, we obtain a lower bound of the order \begin{equation*}
    r \wedge (s_1-r) \wedge \frac{-(k-1)\lambda^2+(k-2)\lambda+1}{n(k-1)\lambda^2}r(s_1-r)
\end{equation*}
by applying Lemma \ref{lemma:fano}. By symmetry, we also have the lower bound with $r(s_1-r)$ replaced by $r(s_i-r)$ for $i=2, 3, \dots, k$. Furthermore, recall that we have $r\leq \frac{1}{2}\min_i s_i$ hence $r\leq (s_i-r)$, we obtain the lower bound of the order \begin{equation*}
    r\wedge \frac{r\sum_{i=1}^k s_i}{n\lambda^2}
\end{equation*}
for the estimation of $U$ with metric $\dist^2(U, \widehat{U})$, when $k$ is finite. 

\paragraph{Step II}
We now turn to establish the lower bound term involving $s_i\log\frac{ep_i}{s_i}$. The step follows from the rank-one argument from \cite{chen2013sparse}. Without loss of generality, we may assume that $s_i\leq \frac{p_i}{2}$ for any $i$, we then consider the following subset of the parameter space:
\begin{align*}
    T_1=\left\{\Sigma = \begin{pmatrix}
    I_{p_1} &\lambda U_{\dc{1}}U_{\dc{2}0}^\top &\lambda U_{\dc{1}}U_{\dc{3}0}^\top \dots & \lambda U_{\dc{1}}U_{\dc{k}0}^\top\\
  \lambda  U_{\dc{2}0}U_{\dc{1}}^\top &I_{p_2} &\dots&\lambda U_{\dc{2}0}U_{\dc{k}0}^\top\\
    \dots & \dots &\dots &\dots\\
   \lambda U_{\dc{k}0}U_{\dc{1}}^\top &  \lambda U_{\dc{k}0}U_{\dc{2}0}^\top&\dots&
    I_{p_k}
    \end{pmatrix}: U_{\dc{1}}=\begin{bmatrix}
    I_{r-1}&0\\
    0&u_r
    \end{bmatrix}
   \right\},
\end{align*}
for $u_r\in\mathbb{S}^{p_1-r+1}$ and $|\supp(u_r)|\leq s_1-r+1$. Restricting on the set $T_1$, the minimax risk for estimating $U$ is the same as the minimax risk for estimating $u_r$ under the squared error loss $\|u_r-\widehat{u}_r\|_\mathrm{F}^2$. Let $X_{\dc{i}} = [X_{\dc{i}1}, X_{\dc{i}2}]$ for $X_{\dc{i}1}\in\mathbb{R}^{n\times{(r-1)}}$ and $X_{\dc{i}2}\in\mathbb{R}^{n\times{(p_i-r+1)}}$. By the same argument in \cite{gao2015minimax}, it is further equivalent to estimating $u_r$ under the squared loss based on the observations from $X_{\dc{i}2}$ since $X_{\dc{i}2}$ for $i=1, 2\dots, k$ is a sufficient statistic for $u_r$. Applying the argument in \cite{chen2013sparse}, we obtain the lower bound for $\dist^2(U, \widehat{U})$ with the following term\begin{equation*}
    \frac{1}{n\lambda^2}\left(s_1\log\frac{ep_1}{s_1}\right)\wedge 1.
\end{equation*}
By symmetry the same lower bound holds if we replace $s_1, p_1$ by $s_i, p_i$ for $i=2, 3,\dots k$.
\par Combining all the above steps and noticing the $\frac{1}{\sqrt{k}}$ scaling between the estimation of $A$ and $U$, we finally deduce that the lower bound for $\dist^2(A, \widehat{A})$ is given by \eqref{eq:lower_bound}. This finishes our proof. 
\end{proof}

\section{Simulation Details for Sparse PCA of Correlation Matrices}
\label{sec:simu-pca-detail}


Here we describe the procedure of generating a $p\times p$ correlation matrix $R$ with eigenvalues $\lambda_1\geq \dots\geq \lambda_r > \lambda_{r+1} = 1 \geq \dots \geq \lambda_p > 0$ and $s$-sparse leading eigenvectors. 
Necessarily, $\sum_{i=1}^r\lambda_r \leq s < p-1$ and $s\geq r$. 
Without loss of generality, we assume that $s = m \times 2^l$ where $m$ and $l$ are positive integers and $2^l\geq r$.

Since $2^l\geq r$, there is a $2^l \times r$ matrix $T_0$ such that $T_{0, ij}\in \{\pm 1\}$ and that the columns of $T_0$ are orthogonal.
For example, we may take the first $r$ columns of a $2^l\times 2^l$ Hadamard matrix.
Fix such a $T_0$, we generate an $s\times r$ matrix $T_s$ as
\begin{equation*}
	T_s = \frac{1}{\sqrt{s}}
	\begin{bmatrix}
		T_0^\top & \dots & T_0^\top
	\end{bmatrix}^\top.
\end{equation*}
By our construction, each row of $T_s$ has the same $l_2$ norm and the columns are orthonormal.
Next, we define
\begin{equation*}
	R_s = T_s\, \mathrm{diag}(\theta_1,\dots,\theta_r) T_s^\top + 
	\left(1 - \frac{\sum_{i=1}^r\theta_i}{s}\right) I_s.
\end{equation*}
This decomposition ensures that the diagonal of $R_s$ is equal to 1. We then solve for $\theta_i$ such that the eigenvalues of $R_s$ are given by $\lambda_1, \dots, \lambda_r.$
Then we augment this $s\times s$ positive definite matrix to a $p\times p$ correlation matrix $R = \mathrm{diag}(R_s, I_{p-s})$.
It is straightforward to verify that $R$ has leading eigenvalues $\lambda_1 \geq \dots \geq \lambda_r > 1$ and the columns of $T = [T_s^\top\, O_{(p-s)\times r}^\top]^\top$ are the corresponding eigenvectors.

\vskip 0.2in

\bibliographystyle{abbrvnat}

\bibliography{sgca.bib,zm.bib}

\begin{thebibliography}{42}
\providecommand{\natexlab}[1]{#1}
\providecommand{\url}[1]{\texttt{#1}}
\expandafter\ifx\csname urlstyle\endcsname\relax
  \providecommand{\doi}[1]{doi: #1}\else
  \providecommand{\doi}{doi: \begingroup \urlstyle{rm}\Url}\fi

\bibitem[Amini and Wainwright(2009)]{Amini09}
A.~Amini and M.~Wainwright.
\newblock High-dimensional analysis of semidefinite relaxations for sparse
  principal components.
\newblock \emph{The Annals of Statistics}, 37\penalty0 (5B):\penalty0
  2877--2921, 2009.

\bibitem[Berthet and Rigollet(2013)]{Berthet13}
Q.~Berthet and P.~Rigollet.
\newblock Computational lower bounds for sparse pca.
\newblock \emph{arXiv preprint arXiv:1304.0828}, 2013.

\bibitem[Bickel and Levina(2008{\natexlab{a}})]{bickel2008covariance}
P.~J. Bickel and E.~Levina.
\newblock Covariance regularization by thresholding.
\newblock \emph{The Annals of Statistics}, 36\penalty0 (6):\penalty0
  2577--2604, 2008{\natexlab{a}}.

\bibitem[Bickel and Levina(2008{\natexlab{b}})]{bickel2008regularized}
P.~J. Bickel and E.~Levina.
\newblock Regularized estimation of large covariance matrices.
\newblock \emph{The Annals of Statistics}, 36\penalty0 (1):\penalty0 199--227,
  2008{\natexlab{b}}.

\bibitem[Bickel et~al.(2009)Bickel, Ritov, and
  Tsybakov]{bickel2009simultaneous}
P.~J. Bickel, Y.~Ritov, and A.~B. Tsybakov.
\newblock Simultaneous analysis of lasso and dantzig selector.
\newblock \emph{The Annals of Statistics}, 37\penalty0 (4):\penalty0
  1705--1732, 2009.

\bibitem[Birnbaum et~al.(2013)Birnbaum, Johnstone, Nadler, and
  Paul]{Birnbaum12}
A.~Birnbaum, I.~Johnstone, B.~Nadler, and D.~Paul.
\newblock {Minimax bounds for sparse PCA with noisy high-dimensional data}.
\newblock \emph{The Annals of Statistics}, 41\penalty0 (3):\penalty0
  1055--1084, 2013.

\bibitem[Cai et~al.(2013)Cai, Ma, and Wu]{cai2013sparse}
T.~T. Cai, Z.~Ma, and Y.~Wu.
\newblock Sparse pca: Optimal rates and adaptive estimation.
\newblock \emph{The Annals of Statistics}, 41\penalty0 (6):\penalty0
  3074--3110, 2013.

\bibitem[Chen et~al.(2013)Chen, Gao, Ren, and Zhou]{chen2013sparse}
M.~Chen, C.~Gao, Z.~Ren, and H.~H. Zhou.
\newblock Sparse cca via precision adjusted iterative thresholding.
\newblock \emph{arXiv preprint arXiv:1311.6186}, 2013.

\bibitem[Chen and Candes(2015)]{chen2015solving}
Y.~Chen and E.~Candes.
\newblock Solving random quadratic systems of equations is nearly as easy as
  solving linear systems.
\newblock In \emph{Advances in Neural Information Processing Systems}, pages
  739--747, 2015.

\bibitem[Chi et~al.(2019)Chi, Lu, and Chen]{chi2019nonconvex}
Y.~Chi, Y.~M. Lu, and Y.~Chen.
\newblock Nonconvex optimization meets low-rank matrix factorization: An
  overview.
\newblock \emph{IEEE Transactions on Signal Processing}, 67\penalty0
  (20):\penalty0 5239--5269, 2019.

\bibitem[Dattorro(2003)]{dattorroconvex}
J.~Dattorro.
\newblock Convex optimization and euclidean distance geometry, 2005.
\newblock \emph{Meboo, Palo Alto}, 2003.

\bibitem[Donnell et~al.(1994)Donnell, Buja, and Stuetzle]{donnell1994analysis}
D.~J. Donnell, A.~Buja, and W.~Stuetzle.
\newblock Analysis of additive dependencies and concurvities using smallest
  additive principal components.
\newblock \emph{The Annals of Statistics}, pages 1635--1668, 1994.

\bibitem[Donoho(2000)]{donoho2000high}
D.~L. Donoho.
\newblock {High-dimensional data analysis: The curses and blessings of
  dimensionality}.
\newblock \emph{AMS math challenges lecture}, 1\penalty0 (2000):\penalty0 32,
  2000.

\bibitem[Fan et~al.(2019)Fan, Guo, and Zheng]{fan2019estimating}
J.~Fan, J.~Guo, and S.~Zheng.
\newblock Estimating number of factors by adjusted eigenvalues thresholding.
\newblock \emph{arXiv preprint arXiv:1909.10710}, 2019.

\bibitem[Gao et~al.(2015{\natexlab{a}})Gao, Ma, Ren, and Zhou]{gao2015minimax}
C.~Gao, Z.~Ma, Z.~Ren, and H.~H. Zhou.
\newblock Minimax estimation in sparse canonical correlation analysis.
\newblock \emph{The Annals of Statistics}, 43\penalty0 (5):\penalty0
  2168--2197, 2015{\natexlab{a}}.

\bibitem[Gao et~al.(2015{\natexlab{b}})Gao, Ma, Ren, and
  Zhou]{gao2015supplement}
C.~Gao, Z.~Ma, Z.~Ren, and H.~H. Zhou.
\newblock Supplement to" minimax estimation in sparse canonical correlation
  analysis".
\newblock \emph{The Annals of Statistics}, 2015{\natexlab{b}}.

\bibitem[Gao et~al.(2017)Gao, Ma, and Zhou]{gao2017sparse}
C.~Gao, Z.~Ma, and H.~H. Zhou.
\newblock Sparse cca: Adaptive estimation and computational barriers.
\newblock \emph{The Annals of Statistics}, 45\penalty0 (5):\penalty0
  2074--2101, 2017.

\bibitem[Ge et~al.(2017)Ge, Jin, and Zheng]{ge2017no}
R.~Ge, C.~Jin, and Y.~Zheng.
\newblock No spurious local minima in nonconvex low rank problems: A unified
  geometric analysis.
\newblock In \emph{Proceedings of the 34th International Conference on Machine
  Learning-Volume 70}, pages 1233--1242. JMLR. org, 2017.

\bibitem[Golub and Van~Loan(2012)]{golub2012matrix}
G.~H. Golub and C.~F. Van~Loan.
\newblock \emph{Matrix computations}, volume~3.
\newblock JHU press, 2012.

\bibitem[Hastie et~al.(2015)Hastie, Tibshirani, and
  Wainwright]{hastie2015statistical}
T.~Hastie, R.~Tibshirani, and M.~Wainwright.
\newblock \emph{Statistical learning with sparsity: the lasso and
  generalizations}.
\newblock CRC press, 2015.

\bibitem[Hotelling(1992)]{hotelling1992relations}
H.~Hotelling.
\newblock Relations between two sets of variates.
\newblock In \emph{Breakthroughs in statistics}, pages 162--190. Springer,
  1992.

\bibitem[Johnstone and Lu(2009)]{JohnstoneLu09}
I.~Johnstone and A.~Lu.
\newblock On consistency and sparsity for principal components analysis in high
  dimensions.
\newblock \emph{Journal of the American Statistical Association}, 104\penalty0
  (486):\penalty0 682--693, 2009.

\bibitem[Johnstone(2001)]{johnstone2001distribution}
I.~M. Johnstone.
\newblock On the distribution of the largest eigenvalue in principal components
  analysis.
\newblock \emph{Annals of statistics}, pages 295--327, 2001.

\bibitem[Kettenring(1971)]{kettenring1971canonical}
J.~R. Kettenring.
\newblock Canonical analysis of several sets of variables.
\newblock \emph{Biometrika}, 58\penalty0 (3):\penalty0 433--451, 1971.

\bibitem[Ma(2013)]{ma2013sparse}
Z.~Ma.
\newblock Sparse principal component analysis and iterative thresholding.
\newblock \emph{The Annals of Statistics}, 41\penalty0 (2):\penalty0 772--801,
  2013.

\bibitem[Overton and Womersley(1991)]{maxeigenvalue}
M.~L. Overton and R.~S. Womersley.
\newblock On the sum of the largest eigenvalues of a symmetric matrix.
\newblock \emph{SIAM J. Matrix Anal. Appl.}, 1991.

\bibitem[Shen et~al.(2014)Shen, Sun, Tang, and Priebe]{shen2014generalized}
C.~Shen, M.~Sun, M.~Tang, and C.~E. Priebe.
\newblock Generalized canonical correlation analysis for classification.
\newblock \emph{Journal of Multivariate Analysis}, 130:\penalty0 310--322,
  2014.

\bibitem[Sun(1983)]{sun1983perturbation}
J.-g. Sun.
\newblock The perturbation bounds for eigenspaces of a definite matrix-pair.
\newblock \emph{Numerische Mathematik}, 41\penalty0 (3):\penalty0 321--343,
  1983.

\bibitem[Tan et~al.(2018)Tan, Wang, Liu, and Zhang]{tan2018sparse}
K.~M. Tan, Z.~Wang, H.~Liu, and T.~Zhang.
\newblock Sparse generalized eigenvalue problem: Optimal statistical rates via
  truncated rayleigh flow.
\newblock \emph{Journal of the Royal Statistical Society: Series B (Statistical
  Methodology)}, 80\penalty0 (5):\penalty0 1057--1086, 2018.

\bibitem[Ten~Berge(1977)]{ten1977orthogonal}
J.~M. Ten~Berge.
\newblock Orthogonal procrustes rotation for two or more matrices.
\newblock \emph{Psychometrika}, 42\penalty0 (2):\penalty0 267--276, 1977.

\bibitem[Tenenhaus and Tenenhaus(2011)]{tenenhaus2011regularized}
A.~Tenenhaus and M.~Tenenhaus.
\newblock Regularized generalized canonical correlation analysis.
\newblock \emph{Psychometrika}, 76\penalty0 (2):\penalty0 257, 2011.

\bibitem[Tian(2004)]{tian2004rank}
Y.~Tian.
\newblock Rank equalities for block matrices and their moore-penrose inverses.
\newblock \emph{Houston J. Math}, 30\penalty0 (4):\penalty0 483--510, 2004.

\bibitem[Tu et~al.(2015)Tu, Boczar, Simchowitz, Soltanolkotabi, and
  Recht]{tu2015low}
S.~Tu, R.~Boczar, M.~Simchowitz, M.~Soltanolkotabi, and B.~Recht.
\newblock Low-rank solutions of linear matrix equations via procrustes flow.
\newblock \emph{arXiv preprint arXiv:1507.03566}, 2015.

\bibitem[Vu and Lei(2012)]{Vu12sp}
V.~Vu and J.~Lei.
\newblock Minimax sparse principal subspace estimation in high dimensions.
\newblock \emph{arXiv preprint arXiv:1211.0373}, 2012.

\bibitem[Vu et~al.(2013)Vu, Cho, Lei, and Rohe]{vu2013fantope}
V.~Q. Vu, J.~Cho, J.~Lei, and K.~Rohe.
\newblock Fantope projection and selection: A near-optimal convex relaxation of
  sparse pca.
\newblock In \emph{Advances in neural information processing systems}, pages
  2670--2678, 2013.

\bibitem[Wainwright(2019)]{wainwright2019high}
M.~J. Wainwright.
\newblock \emph{High-dimensional statistics: A non-asymptotic viewpoint},
  volume~48.
\newblock Cambridge University Press, 2019.

\bibitem[Wang et~al.(2014)Wang, Lu, and Liu]{wang2014tighten}
Z.~Wang, H.~Lu, and H.~Liu.
\newblock Tighten after relax: Minimax-optimal sparse pca in polynomial time.
\newblock In \emph{Advances in neural information processing systems}, pages
  3383--3391, 2014.

\bibitem[Weyl(1912)]{weyl1912asymptotische}
H.~Weyl.
\newblock Das asymptotische verteilungsgesetz der eigenwerte linearer
  partieller differentialgleichungen (mit einer anwendung auf die theorie der
  hohlraumstrahlung).
\newblock \emph{Mathematische Annalen}, 71\penalty0 (4):\penalty0 441--479,
  1912.

\bibitem[Witten et~al.(2009)Witten, Tibshirani, and Hastie]{Witten09}
D.~Witten, R.~Tibshirani, and T.~Hastie.
\newblock {{A} penalized matrix decomposition, with applications to sparse
  principal components and canonical correlation analysis}.
\newblock \emph{Biostatistics}, 10:\penalty0 515--534, 2009.

\bibitem[Yu et~al.(1997)Yu, Assouad, and Le~Cam]{yu1997festschrift}
B.~Yu, F.~Assouad, and L.~Le~Cam.
\newblock Festschrift for lucien le cam, 1997.

\bibitem[Yuan and Zhang(2013)]{yuan2013truncated}
X.-T. Yuan and T.~Zhang.
\newblock Truncated power method for sparse eigenvalue problems.
\newblock \emph{Journal of Machine Learning Research}, 14\penalty0
  (Apr):\penalty0 899--925, 2013.

\bibitem[Zou et~al.(2006)Zou, Hastie, and Tibshirani]{Zou06}
H.~Zou, T.~Hastie, and R.~Tibshirani.
\newblock Sparse principal component analysis.
\newblock \emph{Journal of Computational and Graphical Statistics},
  15:\penalty0 265--286, 2006.

\end{thebibliography}

\end{document}